\documentclass[11pt]{article}
\usepackage[letterpaper]{geometry}
\usepackage[parfill]{parskip}
\usepackage{amsmath,amsthm,amssymb,bbm}
\usepackage{mathtools}
\usepackage{cases}
\usepackage{dsfont}
\usepackage{microtype}

\usepackage{subfigure}
\usepackage{algorithm,algorithmic}
\usepackage{color}
\usepackage{appendix}

\usepackage{bm}
\usepackage{subfigure}
\usepackage{algorithm,algorithmic}
\usepackage{color}
\usepackage{booktabs}       
\usepackage{appendix}
\usepackage{authblk}

\usepackage{url}
\usepackage[authoryear]{natbib}
\usepackage[colorlinks,citecolor=blue,urlcolor=blue,linkcolor=blue,linktocpage=true]{hyperref}
\usepackage{cleveref}
\crefformat{equation}{(#2#1#3)}
\crefrangeformat{equation}{(#3#1#4) to~(#5#2#6)}
\crefname{equation}{}{}
\Crefname{equation}{}{}

\crefname{definition}{\textbf{definition}}{definitions}
\Crefname{definition}{Definition}{Definitions}
\crefname{assumption}{\textbf{assumption}}{assumptions}
\Crefname{assumption}{Assumption}{Assumptions}

\definecolor{maroon}{RGB}{192,80,77}
\definecolor{mypink3}{cmyk}{0, 0.7808, 0.4429, 0.1412}

\newtheorem{theorem}{Theorem}[section]
\newtheorem{lemma}[theorem]{Lemma}

\newtheorem{remark}[theorem]{Remark}
\newtheorem{assumption}[theorem]{Assumption}

\usepackage{amsmath}

\newcommand\norm[1]{\left\lVert#1\right\rVert}
\newcommand{\mypink}{\textcolor{mypink3}}

\newcommand{\argmax}{\mathop{\mathrm{argmax}}}

\def\hpi{\widehat{\pi}}
\def\E{\mathbb{E}}
\def\P{\mathbb{P}}

\def\Var{\mathrm{Var}}

\def\R{\mathbb{R}}

\begin{document}

\title{Near-Optimal Provable Uniform Convergence in Offline Policy Evaluation for Reinforcement Learning}

\author[1,3]{Ming Yin }
\author[2]{Yu Bai}
\author[3]{Yu-Xiang Wang}
\affil[1]{Department of Statistics and Applied Probability, UC Santa Barbara}
\affil[2]{Salesforce Research}
\affil[3]{Department of Computer Science, UC Santa Barbara}
\affil[ ]{\texttt{ming\_yin@ucsb.edu}  \quad \texttt{yu.bai@salesforce.com} \quad
	\texttt{yuxiangw@cs.ucsb.edu}}

\maketitle

\begin{abstract}	


The problem of \emph{Offline Policy Evaluation} (OPE) in Reinforcement Learning (RL) is a critical step towards applying RL in real life applications. Existing work on OPE mostly focus on evaluating a \emph{fixed} target policy $\pi$, which does not provide useful bounds for offline policy learning as $\pi$ will then be data-dependent. We address this problem by \emph{simultaneously} evaluating all policies in a policy class $\Pi$ --- uniform convergence in OPE --- and obtain nearly optimal error bounds for a number of global / local policy classes. Our results imply that the model-based planning achieves an optimal episode complexity of $\widetilde{O}(H^3/d_m\epsilon^2)$ in identifying an $\epsilon$-optimal policy under the \emph{time-inhomogeneous episodic} MDP model ($H$ is the planning horizon, $d_m$ is a quantity that reflects the exploration of the logging policy $\mu$). To the best of our knowledge, this is the first time the optimal rate is shown to be possible for the offline RL setting and the paper is the first that systematically investigates the uniform convergence in OPE. 




\end{abstract}

\newpage
\tableofcontents
\newpage

\section{Introduction}\label{sec:introduction}

In offline reinforcement learning (offline RL), there are mainly two fundamental problems: \textit{offline policy evaluation} (OPE) and \textit{offline learning} (also known as \textit{batch RL}) \citep{sutton2018reinforcement}.  OPE addresses to the statistical estimation problem of predicting the performance of a fixed target policy $\pi$ with only data collected by a logging/behavioral policy $\mu$. On the other hand, offline learning is a \emph{statistical learning} problem that aims at learning a near-optimal policy using an offline dataset alone \citep{lange2012batch}.

As offline RL methods do not require interacting with the task environments or having access to a simulator,  they are more suitable for real-world applications of RL such as those in marketing \citep{thomas2017predictive}, targeted advertising \citep{bottou2013counterfactual,tang2013automatic}, finance \citep{bertoluzzo2012testing}, robotics \citep{quillen2018deep,dasari2019robonet}, language \citep{jaques2019way} and health care \citep{ernst2006clinical,raghu2017continuous,raghu2018behaviour,gottesman2019guidelines}. In these tasks, it is usually not feasible to deploy an online RL algorithm to trials-and-error with the environment. Instead, we are given a large offline dataset of historical interaction to come up with a new policy $\pi$ and to demonstrate that this new policy $\pi$ will perform better using the same dataset without actually testing it online. 
 

In this paper, we present our solution via a statistical learning perspective by studying the \emph{uniform convergence} in OPE 
under the \emph{non-stationary transition, finite horizon, episodic Markov decision process (MDP)} model with finite states and actions. Informally, given a policy class $\Pi$ and a logging policy $\mu$, uniform convergence problem in OPE (Uniform OPE for short)  focuses on coming up with OPE estimator $\widehat{v}^\pi$ and characterizing the number of episodes $n$ we need (from $\mu$) in order for $\widehat{v}^\pi$ to satisfies that with high probability
\begin{equation*}
\sup_{\pi\in\Pi} \left| \widehat{v}^\pi - v^\pi \right| \le \epsilon.
\end{equation*}
The focus of research would be to characterizing the \emph{episode complexity}:  the number of episodes $n$ needed as a function of $\epsilon$, failure probability $\delta$, the parameters of the MDP as well as the logging policy $\mu$. 

We highlight that even though uniform convergence is the main workhorse in statistical learning theory \citep[see, e.g.,][]{vapnik1999overview}, few analogous results have been established for the offline reinforcement learning problem. The overarching theme of this work is to understand what a natural complexity measure is for policy classes in reinforcement learning and its dependence in the size of the state-space and planning horizon. 

In addition, uniform OPE has two major consequences (which we elaborate in detail in the following motivation section): (1) allowing any accurate planning algorithm to work as sample efficient offline learning algorithm with our model-based method; (2) providing finite sample guarantee for offline evaluation uniformly for all policies in the policy class.

\noindent\textbf{The Motivation.} Existing research in offline RL usually focuses on designing specific algorithms that learn the optimal policy $\pi^\star:=\argmax_{\pi}v^\pi$ with given static offline data $\mathcal{D}$. In the rich literature of statistical learning theory, however, learning bounds are often obtained via a stronger uniform convergence argument which ensures an arbitrary learner to output a model that generalizes.  Specifically, the \textit{empirical risk minimizer} (ERM) that outputs the \textit{empirical optimal policy} has been shown to be sufficient and necessary for efficiently learning almost all learnable problems \citep{vapnik1999overview,shalev2010learnability}. 

The natural analogy of ERM in the RL setting would be to find the \textit{empirical optimal policy} $\widehat{\pi}^\star:=\argmax_{\pi}\widehat{v}^\pi$ for some OPE estimator $\widehat{v}^\pi$. If we could establish a uniform convergence bound for $\widehat{v}^\pi$, then it implies that $\widehat{\pi}^\star$ is nearly optimal too via
\begin{align*}
0&\leq v^{\pi^\star}-v^{\widehat{\pi}^\star}= v^{\pi^\star}-\widehat{v}^{\widehat{\pi}^\star}+\widehat{v}^{\widehat{\pi}^\star}-v^{\widehat{\pi}^\star}\\
&\leq |v^{\pi^\star}-\widehat{v}^{{\pi}^\star}|+|\widehat{v}^{\widehat{\pi}^\star}-v^{\widehat{\pi}^\star}|\leq 2\sup_\pi|v^\pi-\widehat{v}^\pi|.
\end{align*}
Thus, uniform OPE is a stronger setting than offline learning with the additional benefit of accurately evaluating any other (possibly heuristic) policy optimization algorithms that are used in practice.

From the OPE perspective, there is often a need to evaluate the performance of a \emph{data-dependent} policy, and uniform OPE becomes useful. 
For example, when combined with existing methods, it will allow us to evaluate policies selected by safe-policy improvements, proximal policy optimization, UCB-style exploration-bonus as well as any heuristic exploration criteria such as curiosity, diversity and reward-shaping techniques.


\noindent\textbf{Model-based estimator for OPE.} The OPE estimator we consider in this paper is the standard model-based estimator, i.e., estimating the transition dynamics and immediate rewards, then simply plug in the parameters of empirically estimated MDP $\widehat{M}$ to obtain $\hat{v}^\pi$ for any $\pi$. This model-based approach 
has several benefits. \textbf{1.} It
enables flexible choice of policy search methods since it converts the problem to planning over the estimated MDP $\widehat{M}$. \textbf{2.} Uniform OPE with model-based estimator avoids the use of data-splitting that leads to inefficient data use. 
For example, \citet{sidford2018near} learns the $\epsilon$-optimal policy with the optimal rate in the generative model setting, where in each subroutine new independent data $s_{s,a}^{(1)},...,s^{(m)}_{s,a}$ need to be sampled to estimate $P_{s,a}$ and samples from previous rounds cannot be reused. A uniform convergence result could completely avoid data splitting during the learning procedure. 


\noindent\textbf{Our contribution.} 
Our main contributions are summarized as follows.
\begin{itemize}
	\itemsep1pt
	\item For the global policy class (deterministic or stochastic), we use fully model-based OPEMA estimator to obtain an $\epsilon$-uniform OPE with episode complexity $\widetilde{O}(H^4S/d_m \epsilon^2)$ (Theorem~\ref{thm:rad}) and in some cases this can be reduced to $\widetilde{O}(H^4/d_m \epsilon^2)$, where $d_m$ is minimal marginal state-action occupancy probability depending on logging policy $\mu$. 
	
	\item For the global deterministic policy class, we obtain an $\epsilon$-uniform OPE with episode complexity $\widetilde{O}(H^3S/d_m \epsilon^2)$ with an optimal dependence on $H$ (Theorem~\ref{thm:finer_bound}).  
	
	\item For a (data-dependent) local policy class that cover all policies are in the $O(\sqrt{H}/S)$-neighborhood of the \emph{empirical} optimal policy (see the definition in Section~\ref{sec:problems}), we obtain $\epsilon$-uniform OPE with $\widetilde{O}(H^3/d_m\epsilon^2)$ episodes (Theorem~\ref{thm:local_uni_opt}). 
	
	\item 
	We prove a information-theoretical lower bound of {$\Omega(H^3/d_m\epsilon^2)$} for OPE (Theorem~\ref{thm:uni_lower}) which certifies that results for local policy class is optimal.
	
	
	\item 
	Our uniform OPE over the local policy class implies that ERM (VI or PI with empirically estimated MDP), as well as any sufficiently accurate model-based planning algorithm,  has an optimal episode complexity of $\widetilde{O}(H^3/d_m\epsilon^2)$ (Theorem~\ref{thm:offlinelearning}). To the best of our knowledge, this is the first rate-optimal algorithm in the offline RL setting. 
	
	\item Last but not least, our result can be viewed as an improved analysis of the \emph{simulation lemma}; which demystifies the common misconception that purely model plug-in estimator is inefficient, comparing to their model-free counterpart.
\end{itemize} 
To the best of our knowledge, these results are new and this is the first work that derives uniform convergence analogous to those in the statistical learning theories for offline RL.

\noindent\textbf{Related work.}
Before formally stating our results, we briefly discuss the related literature in three categories.

\textbf{1. OPE:} Most existing work on OPE focuses on the \textit{Importance Sampling} (IS) methods \citep{li2011unbiased,dudik2011doubly,li2015toward,thomas2016data} or their doubly robust variants \citep{jiang2016doubly,farajtabar18a}. These methods are more generally applicable even if the the Markovian assumption is violated or the states are not observable, but has an error (or sample complexity) that depends exponential dependence in horizon $H$. Recently, a family of estimators based on \textit{marginalized importance sampling} (MIS) \citep{liu2018breaking,xie2019towards,kallus2019double,kallus2019efficiently,yin2020asymptotically} have been proposed in order to overcome the ``\textit{curse of horizon}'' under the additional assumption of state observability. In the tabular setting, \citet{yin2020asymptotically} design the Tabular-MIS estimator which matches the  Cramer-Rao lower bound constructed by \citet{jiang2016doubly} up to a low order term for every instance ($\pi,\mu$ and the MDP), which translates into an $O(H^2/d_m\epsilon^2)$ episode complexity in the (pointwise) OPE problem we consider for all $\pi$. Tabular-MIS, however, is identical to the model-based plug-in estimator we use, \emph{off-policy empirical model approximator} (OPEMA), as we discuss further in Section~\ref{subsec:method}. These methods do not address the uniform convergence problem. The only exception is \citep{yin2020asymptotically}, which has a result analogous to Theorem~\ref{thm:local_uni_opt}, but for a data-splitting-type estimator.


\textbf{2. Offline Learning:} For the offline learning, most theoretical work consider the infinite horizon discounted setting with function approximation. \citet{chen2019information,le2019batch} first raises the information-theoretic considerations for offline learning
and uses Fitted Q-Iteration (FQI) to obtain $\epsilon V_{\max}$-optimal policy using sample complexity 
$\widetilde{O}((1-\gamma)^{-4}C_\mu/\epsilon^2)$ where $C_\mu$ is \emph{concentration coefficient} \citep{munos2003error} that is similar to our $1/d_m$. More recently,  \citep{xie2020q} improves the result to $\tilde{O}((1- \gamma)^{-2}C_\mu/\epsilon^2)$. However, these bounds are not tight 
in terms of the dependence on the effective horizon\footnote{The optimal rate should be $(1- \gamma)^{-1}C/\epsilon^2$, analogous to our $H^{3}/d_m\epsilon^2$ bound. The additional $H^2$ is due to scaling ---  we are obtaining $\epsilon$-optimal policy and they obtain $\epsilon V_{\max}$-optimal policy ($V_{\max} = H$ in our case). See Table~\ref{table} for a consistent comparison.} $(1-\gamma)^{-1}$. More recently, \citet{xie2020batch,liu2020provably} explore weaker settings for batch learning but with suboptimal sample complexity dependences.  Our result is the first that achieves the optimal rate (despite focusing on the finite horizon episodic setting).

\textbf{3. Uniform convergence in RL:} There are few existing work that deals with uniform convergence in OPE. However, we notice that the celebrated simulation lemma \citep{kearns2002near} is actually an uniform bound with an episode complexity of $O(H^4S^2/d_m\epsilon^2)$.  Several existing work uses uniform-convergence arguments over value function classes for online RL \citep[see, e.g.,][and the references therein]{jin2020provably}. The closest to our work is perhaps \citep{agarwal2020model}, which studies model-based planning in the generative model setting. We are different in that we are in the offline learning setting. In addition, our local policy class is optimal for a larger region of $\epsilon_{\text{opt}}$ (independent to $n$), while their results (Lemma 10) imply optimal OPE only for empirically optimal policy with $\epsilon_{\text{opt}}\leq \sqrt{(1-\gamma)^{-5}SA/n}$. Lastly, we discovered the thesis of \citet[Ch.3 Theorem 1]{tewari2007reinforcement}, which discusses the pseudo-dimension of policy classes. The setting is not compatible to ours, and does not imply a uniform OPE bound in our setting.







\section{Problem setup and method}\label{sec:formulation}
RL environment is usually modeled as a \textit{Markov Decision Process} (MDP) which is denoted by $M=(\mathcal{S},\mathcal{A}, r,P,d_1,H)$. The MDP consists of a state space $\mathcal{S}$, an action space $\mathcal{A}$ and a transition kernel $P_t : \mathcal{S}\times\mathcal{A}\times\mathcal{S} \mapsto [0, 1]$  with $P_t(s'|s,a)$ representing the probability transition from state $s$, action $a$ to next state $s'$ at time $t$. In particular here we consider non-stationary transition dynamics so $P_t$ varies over time $t$. Besides, $r_t : \mathcal{S} \times{A} \mapsto \mathbb{R}$ is the expected reward function and given $(s_t,a_t)$, $r_t(s_t, a_t)$ specifies the average reward obtained at time $t$. $d_1$ is the initial state distribution and $H$ is the horizon. Moreover, we focus on the case where state space $\mathcal{S}$ and the action space $\mathcal{A}$ are finite, \textit{i.e.} $S:=|\mathcal{S}|<\infty,A:=|\mathcal{A}|<\infty$. A (non-stationary) policy is formulated by $\pi:=(\pi_1,\pi_2,...,\pi_H)$, where $\pi_t$ assigns each state $s_t \in \mathcal{S}$ a probability distribution over actions at each time $t$. Any fixed policy $\pi$ together with MDP $M$ induce a distribution over trajectories of the form $(s_1,a_1,r_1,s_2,...,s_H,a_H,r_H,s_{H+1})$ where $s_1\sim d_1$, $a_t\sim \pi_t(\cdot|s_t)$, $s_{t+1}\sim P_t(\cdot|s_t,a_t)$ and $r_t$ has mean $r_t(s_t, a_t)$ for $t=1,...,H$.\footnote{Here $r_t$ without any argument is random reward and $\E[r_t|s_t,a_t]=r_t(s_t,a_t)$.} 

In addition, we denote $d^\pi_t(s_t,a_t)$ the induced marginal state-action distribution and $d^\pi_t(s_t)$ the marginal state distribution, satisfying $d^{\pi}_t(s_t,a_t)=d^{\pi}_t(s_t)\cdot \pi(a_t|s_t)$. Moreover, $d^\pi_1=d_1$ $\forall \pi$. We use the notation $P^\pi_{t}\in\R^{S\cdot A \times S\cdot A}$ to represent the state-action transition $(P^\pi_{t})_{(s,a),(s',a')}:=P_{t}(s'|s,a)\pi_{t}(a'|s')$, then the marginal state-action vector $d^\pi_t(\cdot,\cdot)\in\R^{S\times A}$ satisfies the expression $d^\pi_{t+1}=P^\pi_{t+1}d^\pi_{t}$. We define the quantity
$
V^\pi_t(s)=\E_\pi[\sum_{t'=t}^H r_{t'}|s_t=s] 
$
and the Q-function
$Q^\pi_t(s,a)=\E_\pi[\sum_{t'=t}^H r_{t'}|s_t=s,a_t=a]$ for all $t=1,...,H$. The ultimate measure of the performance of policy $\pi$ is the value function:
{
\[
v^\pi=\E_\pi\left[\sum_{t=1}^H r_t\right]. 
\]}
\normalsize{}
Lastly, for the standard OPE problem, the goal is to estimate $v^\pi$ for a given $\pi$ while assuming that $n$ episodic data $\mathcal{D}=\left\lbrace (s_t^{(i)},a_t^{(i)},r_t^{(i)},s_{t+1}^{(i)})\right\rbrace_{ i\in[n]}^{t\in[H]} $ are rolling from a \textit{different} behavior policy $\mu$.

\subsection{Uniform convergence problems}\label{sec:problems}

Uniform OPE extends the pointwise OPE to a family of policies. Specifically, for an policy class $\Pi$ of interest, we aim at showing that
$\sup_{\pi\in\Pi}|\widehat{v}^\pi-v^\pi|<\epsilon$
with high probability with optimal dependence in all parameters. In this paper, we consider three policy classes. 

\noindent\textbf{The global policy class.} The policy class $\Pi$ we considered here consists of all the non-stationary policies, deterministic or stochastic. This is the largest possible class we can consider and hence the hardest one. 

\noindent\textbf{The global deterministic policy class.} Here class consists of all the non-stationary deterministic policies. By the standard results in reinforcement learning, there exists at least one deterministic policy that is optimal \citep{sutton2018reinforcement}. Therefore, the deterministic policy class is rich enough for evaluating any learning algorithm (\textit{e.g.} Q-value iteration in \citet{sidford2018near}) that wants to learn to the optimal policy.

\noindent\textbf{The local policy class: in the neighborhood of empirical optimal policy.} Given empirical MDP $\widehat{M}$ (\textit{i.e.} the transition kernel is replaced by $\widehat{P}_t(s_{t+1}|s_{t},a_{t}):=n_{s_{t+1},s_{t},a_{t}}/n_{s_{t},a_{t}}$ if $n_{s_t,a_t}>0$ and $0$ otherwise, where $ n_{s_{t},a_{t}}$ is the number of visitations to $(s_t,a_t)$ among all $n$ episodes\footnote{Similar definition holds for $n_{s_{t+1},s_{t},a_{t}}$.}), it is convenient to learn the empirical optimal policy $\widehat{\pi}^\star:=\argmax_{\pi}\widehat{v}^\pi$ since the full empirical transition $\widehat{P}$ is known. Standard methods like Policy Iteration (PI) and Value Iteration (VI) can be leveraged for finding $\widehat{\pi}^\star$. This observation allows us to consider the following interesting policy class: 
$
\Pi_1:=\lbrace \pi:s.t.\; ||\widehat{V}^\pi_t-\widehat{V}^{\widehat{\pi}^\star}_t ||_\infty\leq\epsilon_\mathrm{opt},\;\forall t=1,...,H\rbrace
$ with $\epsilon_\mathrm{opt}\geq0$ a parameter. Here we consider $\widehat{\pi}^\star$ (instead of $\pi^\star$) since by defining with empirical optimal policy, we can use data $\mathcal{D}$ to really check class $\Pi_1$, therefore this definition is more practical. 

\subsection{Assumptions}
Next we present some mild necessary regularity assumptions for uniform convergence OPE problem.
\begin{assumption}[Bounded rewards]\label{assume1}
	$\forall \;t=1,...,H$ and $i=1,...,n$, $0\leq r^{(i)}_t\leq 1$.
\end{assumption}

\begin{assumption}[Exploration requirement]\label{assume2}
Logging policy $\mu$ obeys that $\min_{t,s_t}d^\mu_t(s_t)>0$, for any state $s_t$ that is ``accessible''. Moreover, we define quantity $d_m:=\min\{d_t^\mu(s_t ,a_t): d^\mu_t (s_t , a_t )>0\}$.
\end{assumption}
\vspace{-5pt}

State $s_t$ is ``accessible'' means there exists a policy $\pi$ so that $d^\pi_t(s_t)>0$. If for any policy $\pi$ we always have $d^\pi_t(s_t)=0$, then state $s_t$ can never be visited in the given MDP. Assumption~\ref{assume2} simply says $\mu$ have the right to explore all ``accessible'' states. This assumption is required for the consistency of uniform convergence estimator since we have ``$\sup_{\pi\in\Pi}$'' and is similar to the standard \textit{concentration coefficient} assumption made by \citet{munos2003error,le2019batch}.  As a short comparison, offline learning problems (\textit{e.g.} offline policy optimization in \citet{liu2019off})  only require $d_t^\mu(s_t)>0$ for any state $s_t$ satisfies $d^{\pi^\star}_t(s_t)>0$. Last but not least, even though our target policy class is deterministic, by above assumptions $\mu$ is always stochastic.  

\subsection{Method: Offline Policy Empirical Model Approximator}\label{subsec:method}
The method we use for doing OPE in uniform convergence is the \emph{offline policy empirical model approximator} (OPEMA). OPEMA uses off-policy data to build the empirical estimators for both the transition dynamic and the expected reward and then substitute the related components in real value function by its empirical counterparts. First recall for any target policy $\pi$, by definition:
{ 
$
v^\pi
=\sum_{t=1}^H\sum_{s_t,a_t}{d}^\pi_t(s_t,a_t){r}_t(s_t,a_t),
$
} where the marginal state-action transitions satisfy $d^\pi_{t+1}=P^\pi_{t+1}d^\pi_{t}$. OPEMA then directly construct empirical estimates for $\widehat{P}_{t+1}(s_{t+1}|s_{t},a_{t})$ and $\widehat{r}_t(s_t,a_t)$ as: 
{\small
	\begin{align*}
	\widehat{P}_{t+1}(s_{t+1}|s_{t},a_{t})&=\frac{\sum_{i=1}^n\mathbf{1}[(s^{(i)}_{t+1},a^{(i)}_t,s^{(i)}_t)=(s_{t+1},s_{t},a_{t})]}{n_{s_{t},a_{t}}},\;\;\widehat{r}_t(s_t,a_t)&=\frac{\sum_{i=1}^n r_t^{(i)}\mathbf{1}[(s^{(i)}_t,a^{(i)}_t)=(s_t,a_t)]}{n_{s_t,a_t}}.
	\end{align*}
}and $\widehat{P}_{t+1}(s_{t+1}|s_{t},a_{t})=0$ and $\widehat{r}_t(s_t,a_t)=0$ if $n_{s_t,a_t}=0$ (recall $n_{s_t,a_t}$ is the visitation frequency to $(s_t,a_t)$ at time $t$) and then the estimates for state-action transition $\widehat{P}^\pi_t$ is defined as:
{$
\widehat{P}^\pi_t(s_{t+1},a_{t+1}|s_{t},a_t)=\widehat{P}_t(s_{t+1}|s_{t},a_{t})\pi(a_{t+1}|s_{t+1}).
$}The initial distribution is also constructed using empirical estimator $\widehat{d}_1^\pi(s_1)=n_{s_1}/n$. Based on the construction, the empirical marginal state-action transition follows $\widehat{d}^\pi_{t+1}=\widehat{P}^\pi_{t+1}\widehat{d}^\pi_{t}$ and the final estimator for $v^\pi$ is:
{
\begin{equation}\label{eqn:estimated_value}
\widehat{v}_\mathrm{OPEMA}^\pi=\sum_{t=1}^H\sum_{s_t,a_t}\widehat{d}^\pi_t(s_t,a_t)\widehat{r}_t(s_t,a_t).
\end{equation}
}
OPEMA is model-based method as it uses plug-in estimators ($\widehat{d}^\pi_t$ and $\widehat{r}_t$) for each model components (${d}^\pi_t$ and ${r}_t$). Traditionally, the error of OPEMA is obtained via the simulation lemma~\citep{kearns2002near}, with $O(H^4S^2/d_m\epsilon^2)$-episode complexity. Recent work \citep{xie2019towards,yin2020asymptotically,duan2020minimax} reveals that there is an importance sampling interpretation of OPEMA {
 	\begin{equation}
 	\widehat{v}_\mathrm{OPEMA}^\pi=\frac{1}{n}\sum_{i=1}^n\sum_{t=1}^H\frac{\hat{d}^\pi(s^{(i)}_t)}{\hat{d}^\mu_t(s^{(i)}_t)}\hat{r}^\pi_t(s^{(i)}),
 	\end{equation}
 }and the effectiveness of MIS of recent work partially explains why OPEMA could work, even for
the Uniform OPE problem.

\section{Main results for Uniform OPE}\label{sec:results}

In this section, we present our results for uniform OPE problems from Section~\ref{sec:problems}. For brevity, we use $\widehat{v}^\pi$ to denote $\widehat{v}^\pi_\mathrm{OPEMA}$ in the rest of paper. Proofs of all technical results are deferred to the appendix. We start with the following Lemma:
\begin{lemma}[martingale decomposition]\label{lem:martingale_decompose} For fixed $\pi$:
	\[
	\sum_{t=1}^H\langle \widehat{d}^\pi_t-{d}^\pi_t,r_t\rangle=\sum_{h=2}^H\langle{V^\pi_h,(\widehat{T}_h-T_h)\widehat{d}^\pi_{h-1}}\rangle+\langle{V^\pi_1,\widehat{d}^\pi_1-d^\pi_1}\rangle
	\]
where $T_{h+1}\in\R^{S\times(SA)}$ be the one step transition matrix, \textit{i.e.} $T_{s_{h+1},(s_h,a_h)}=P_{h+1}(s_{h+1}|s_h,a_h)$. the inner product on the left hand side is taken w.r.t state-action and the inner product on the left hand side is taken w.r.t state only. Proof can be found in appendix (Theorem~\ref{thm:martingale_decompose}).
	
\end{lemma}

\begin{remark}
	Note when the reward is deterministic, the left hand side is simply $\widehat{v}^\pi-v^\pi$, and the right hand side has a martingale structure which enables the applicability of concentration analysis that gives rise to the following theorems. Moreover, this decomposition is essentially ``primal-dual'' formulation since the LHS can be viewed as the primal form through marginal distribution representation and RHS is the dual form with value function representation.   
\end{remark}

\subsection{Uniform OPE for global policy class} \label{sec:mart_con}

We present the following result Theorem~\ref{thm:rad} for global policy class.
\begin{theorem}\label{thm:rad}
		Let $\Pi$ consists of all policies, then there exists an absolute constant $c$ such that if $n>c\cdot 1/d_m\cdot \log(HSA/\delta)$, then with probability $1-\delta$, we have:
	\begin{equation*}
	\sup_{\pi\in\Pi} \left|\widehat{v}^\pi-v^\pi\right|\leq c \left(\sqrt{\frac{H^4\log(\frac{HSA}{\delta})}{d_m \cdot n}} +\sqrt{\frac{H^4S\log(nHSA)}{d_m \cdot n}}\right).
	\end{equation*}
	
Moreover, if failure probability $\delta<e^{-S}$, then above can be further bounded by $2c \sqrt{\frac{H^4}{d_m\cdot n}\log(\frac{nHSA}{\delta})}$.
\end{theorem}
The first term in the bound reflects the concentration of $\sup_{\pi\in\Pi} \left|\widehat{v}^\pi-v^\pi\right|$ around its mean, via McDiarmid inequality. The second term is a bound of $\E[\sup_{\pi\in\Pi} \left|\widehat{v}^\pi-v^\pi\right|]$. The analysis of both terms rely on the Martingale decomposition from Lemma~\ref{lem:martingale_decompose}. 
 
%
Our result improves over the simulation lemma by a factor of $HS$ but is suboptimal by another factor $HS$ comparing to the lower bound (Theorem~\ref{thm:uni_lower}). In the small failure probability regime ( $\delta<e^{-S}$) we can get rid of the dependence on $S$ except for the implicit dependence through $d_m$. This is meaningful since we usually consider deriving results with high confidence. 

 
 \subsection{Uniform OPE for deterministic policies} \label{sec:deter_policy}
 The Martingale decomposition also allows us to derive a high-probability OPE bound via a concentration argument, which complements the optimal bounds on mean square error from \citep{yin2020asymptotically}.
\begin{lemma}[Convergence for fixed policy]
  \label{thm:single_finer_bound}
	Fix any policy $\pi$. Then there exists absolute constants $c,c_1,c_2 $ such that if $n>c\cdot 1/d_m\cdot \log(HSA/\delta)$, then with probability $1-\delta$, we have:{
		\begin{align*}
	&\left|\widehat{v}^\pi-v^\pi\right|\leq c_1\sqrt{\frac{H^2\log(\frac{c_2HSA}{\delta})}{n\cdot d_m}} + \tilde{O}\left(\frac{H^2\sqrt{SA}}{n \cdot d_m}\right).
	\end{align*}}
\end{lemma}
Note if we absorb the higher order term, our result implies sample complexity of $\widetilde{O}({H^2}/{d_m\epsilon^2})$ for evaluating any fixed target policy $\pi$. Notice that the total number of deterministic policies is $A^{HS}$ in our problem, a standard union bound over all deterministic policies yields the following result.
\begin{theorem}\label{thm:finer_bound}
	Let $\Pi$ consists of all deterministic policies, then there exists absolute constants $c,c_1,c_2$ such that if $n>c\cdot 1/ d_m\cdot \log(HSA/\delta)$, then with probability $1-\delta$, we have:
		\begin{align*}
	&\sup_{\pi\in\Pi}\left|\widehat{v}^\pi-v^\pi\right|
	\leq c_1\sqrt{\frac{H^3S\log(\frac{c_2HSA}{\delta})}{n\cdot d_m}}
	+ \tilde{O}\left(\frac{H^3S^{1.5}A^{0.5}}{n \cdot d_m}\right).
	\end{align*}
\end{theorem}
Theorem~\ref{thm:finer_bound} implies an episode complexity of $\widetilde{O}(H^3S/d_m\epsilon^2)$, which is optimal in $H$ but suboptimal by a factor of $S$.  While the deterministic policy class seems restrictive, it could be useful in many cases because the optimal policy is deterministic, and many exploration-bonus based exploration methods use deterministic policy throughout.

\begin{remark}\label{remark:diff_dp}
	The similar high-probability OPE bound in Lemma~\ref{thm:single_finer_bound} was proven before by \citep{yin2020asymptotically} through the data-splitting type estimator. However their theory does not imply efficient offline learning, see Section~\ref{sec:alg} in appendix for discussion.
\end{remark}

\subsection{Uniform OPE for the local (near \emph{empirically optimal}) policy class}
For the local (near \emph{empirically optimal}) policy class  we described in Section~\ref{sec:problems}, the following theorem obtains the optimal episode complexity.
\begin{theorem}
	\label{thm:local_uni_opt}
	Suppose $\epsilon_{\text{opt}}\leq \sqrt{H}/S$ and $
	\Pi_1:=\lbrace \pi:s.t.\; ||\widehat{V}^\pi_t-\widehat{V}^{\widehat{\pi}^\star}_t ||_\infty\leq\epsilon_\text{opt},\;\forall t=1,...,H\rbrace
	$. Then there exists constant $c_1,c_2$ such that for any $0<\delta<1$, when $n>c_1H^2\log(HSA/\delta)/d_m$, we have with probability $1-\delta$,
	\[
	\sup_{\pi\in\Pi_1}\norm{\widehat{Q}^{{\pi}}_1-Q^{{\pi}}_1}_\infty\leq
	c_2\sqrt{\frac{H^3\log(HSA/\delta)}{n\cdot d_m}}.
	\]
      \end{theorem}
This uniform convergence result is presented with $l_\infty$ norm over $(s,a)$. A direct corollary is $\sup_{\pi\in\Pi_1}\norm{\widehat{V}^{{\pi}}_1-V^{{\pi}}_1}_\infty$ achieves the same rate. Theorem~\ref{thm:local_uni_opt} provides the sample complexity of $O(H^3\log(HSA/\delta)/d_m\epsilon^2)$ and the dependence of all parameters are optimal up to the logarithmic term. Note that our bound does not explicitly depend on $\epsilon_{\text{opt}}$, which is an improvement over \citep{agarwal2020model} as they have an additional $O(\epsilon_{\text{opt}}/(1-\gamma))$ error in the infinite horizon setting. Besides, our assumption on $\epsilon_{\text{opt}}$ is mild since the required upper bound is proportional to $\sqrt{H}$. Lastly, this result implies a $O(\epsilon+\epsilon_{\text{opt}})$-optimal policy for offline/batch learning of the optimal order $O(H^3\log(HSA/\delta)/d_m\epsilon^2)$ (Theorem~\ref{thm:offlinelearning}), which means statistical learning result enables offline learning.  

\subsection{Information-theoretical lower bound}\label{subsec:lower}
Finally, we present a fine-grained sample complexity lower bound of the uniform OPE problem that captures the dependence of all parameters including $d_m$.
\begin{theorem}[Minimax lower bound for uniform OPE]\label{thm:uni_lower}
	For all $0<d_m\leq\frac{1}{SA}$. Let the class of problems be
		$$\mathcal{M}_{d_m} :=\big\{(\mu,M) \; \big| \;\min_{t,s_t,a_t} d_t^\mu(s_t,a_t)  \geq d_m\big\}.$$
	There exists universal constants $c_1,c_2,c_3,p$	(with $H,S,A\geq c_1$ and $0<\epsilon<c_2$)  such that 
	\[
	\inf_{\widehat{v}}\sup_{(\mu,M)\in\mathcal{M}_{d_m}}\P_{\mu,M}\left(\sup_{\pi\in\Pi}|\widehat{v}^{{\pi}}-v^{{\pi}}|\geq \epsilon\right)\geq p
	\]
	if $n\leq c_3H^3/d_m\epsilon^2$. Here $\Pi$ consists of all deterministic policies. 
\end{theorem}
The proof uses a reduction argument that shows if a stronger uniform OPE bound exists, then it implies an algorithm that breaks an offline learning lower bound (Theorem~\ref{thm:learn_low_bound}), which itself is proven by embedding many stochastic multi-armed bandits problems in a family of hard MDPs. 
Our construction is inspired by the MDPs in \citep{jiang2017contextual} and a personal communication with Christopher Dann but involve substantial modifications to account for the differences in the assumption about rewards. The part in which we obtain explicit dependence on $d_m$ is new and it certifies that the offline learning (and thus uniform OPE) problem strictly more difficult than their online counterpart. 

\noindent\textbf{On optimality.} The above result provides the minimax lower bound of complexity $\Omega(H^3/d_m\epsilon^2)$. As a comparison, Theorem~\ref{thm:finer_bound} gives $\widetilde{O}(H^3 S /d_m\epsilon^2)$ is a factor of $S$ away from the lower bound and Theorem~\ref{thm:local_uni_opt} has the same rate of the lower bound up to logarithmic factor.


\section{Main results for offline learning}\label{sec:offline}

In this section we discuss the implication of our results on offline learning. As we discussed earlier in the introduction, a uniform OPE bound of $\epsilon$ implies that the corresponding ERM algorithm finds a $2\epsilon$-suboptimal policy.  But it also implies that all other offline policy-learning algorithms that are not ERM, we could gracefully decompose their error into optimization error and statistical  (generalization) error. 
\begin{theorem}\label{thm:offlinelearning}
	Let $\hat{\pi}^* = \argmax_{\pi} \hat{v}^{\pi}$ --- the empirically optimal policy. Let $\hat{\pi}$ be any data-dependent choice of policy such that $\hat{v}^{\hat{\pi}^*} - \hat{v}^{\hat{\pi}} \leq \epsilon_\text{opt}$, then. There is a universal constant $c$ such that w.p. $\geq 1-\delta$ 
	\begin{enumerate}
		\itemsep1pt
		\item 	$
		v^{\pi^*}  - v^{\hat{\pi}} \leq c\sqrt{\frac{H^4S\log(HSA/\delta)}{d_m \cdot n}} + \epsilon_\text{opt}.
		$
		\item If  $\delta< e^{-S}$, the bound improves to $c\sqrt{\frac{H^4S\log(HSA/\delta)}{d_m \cdot n}} +  \epsilon_\text{opt}$. And if in addition $\hat{\pi}$ is deterministic, the bound further improves to 
		$c\sqrt{\frac{H^3\min\{H,S\}\log(HSA/\delta)}{d_m \cdot n}} +  \epsilon_\text{opt}$.
		\item If $\epsilon_\text{opt}\leq \sqrt{H}/S$ and that 
		$||\widehat{V}^{\hat{\pi}}_t-\widehat{V}^{\widehat{\pi}^\star}_t ||_\infty\leq\epsilon_\text{opt},\;\forall t=1,...,H$
		, then 
		$
		v^{\pi^*}  - v^{\hat{\pi}} \leq c\sqrt{\frac{H^3\log(HSA/\delta)}{d_m \cdot n}} + \epsilon_\text{opt}.
		$
	\end{enumerate}
\end{theorem}

\begin{table*}
	\caption{A comparison of related offline policy learning results. }
	\label{table}
	\centering
	\resizebox{\linewidth}{!}{%
		\begin{tabular}{llll}
			\toprule
			\cmidrule(r){1-3}
			Method/Analysis     &Setting& Guarantee     & Sample complexity$^b$ \\
			\midrule
			\cite{agarwal2020model} &Generative model &$\epsilon+O(\epsilon_{\text{opt}}/(1-\gamma))$-optimal  & {$\widetilde{O}(SA/(1-\gamma)^3 \epsilon^2)$}    \\
			\cite{le2019batch,chen2019information} &$\infty$-horizon offline & $\epsilon$-optimal policy  & $\widetilde{O}((1-\gamma)^{-6}C_\mu/ \epsilon^2)$   \\
			\cite{xie2020q} & $\infty$-horizon offline & $\epsilon$-optimal policy  & $\widetilde{O}((1-\gamma)^{-4}C_\mu/ \epsilon^2)$   \\
			{SIMPLEX} for exact empirical optimal$^a$ & $H$-horizon offline  &$\epsilon$-optimal policy & \mypink{$\widetilde{O}(H^3/d_m \epsilon^2)$}    \\
			{PI/VI} for $\epsilon_{\text{opt}}$-empirical optimal  &$H$-horizon offline   & $(\epsilon+\epsilon_{\text{opt}})$-optimal policy & \mypink{$\widetilde{O}(H^3/d_m \epsilon^2)$}    \\
			Minimax lower bound (Theorem~\ref{thm:learn_low_bound})  &$H$-horizon offline    & over class $\mathcal{M}_{d_m}$     & $\Omega(H^3/d_m\epsilon^2)$  \\
			\bottomrule
		\end{tabular}
	}
	\hspace{-4cm}\footnotesize{ $^a$ PI/VI or SIMPLEX is not essential and can be replaced by any efficient empirical MDP solver.}\\
	
	\footnotesize{ $^b$ \emph{Episode} complexity in $H$-horizon setting is comparable to \emph{step} complexity in $\infty$-horizon setting because our finite-horizon MDP is \emph{time-inhomogeneous}.  Informally, we can just take $(1-\gamma)^{-1} \asymp H$ and $C_\mu\asymp 1/d_m$. }
\end{table*}

The third statement implies that all sufficiently accurate planning algorithms based on the empirically estimated MDP are optimal. For example, we can run value iteration or policy iteration to the point that $\epsilon_{\text{opt}} \leq O(H^3/{n d_m})$.  

\noindent\textbf{Comparing to existing work.} Previously no algorithm is known to achieve the optimal sample complexity in the offline setting. Our result also applies to the related generative model setting by replacing $1/d_m$ with $SA$, which avoids the data-splitting procedure usually encountered by specific algorithm design \citep[e.g.,][]{sidford2018near}. 
The analogous policy-learning results In the generative model setting \citep[Theorem 1]{agarwal2020model} , 
 achieves a suboptimality of $\tilde{O}(  (1-\gamma)^{-3}SA/n + (1-\gamma)^{-1} \epsilon_{\text{opt}}  )$ with no additional assumption on $ \epsilon_{\text{opt}}$. Informally, if we replace $(1-\gamma)^{-1}$ with $H$, then our result improves the bound from $H \epsilon_{\text{opt}}$ to just $\epsilon_{\text{opt}}$ for $\epsilon_{\text{opt}}\leq \sqrt{H}/S$. 
  These results are summarized in Table~\ref{table}. 
  
  \noindent\textbf{Sparse MDP estimate.}
  We highlight that the result does not require the estimated MDP to be an accurate approximation in any sense. Recall that the true MDP has $O(S^2)$ parameters (ignoring the dependence on $H,A$ and logarithmic terms), but our result is valid provided that  $n = \tilde{\Omega}(1/d_m)$ which is $\Omega(S)$. This suggests that we may not even exhaustively visit all pairs to state-transitions and that the estimator of $\hat{P}_t$ is allowed to be zero in many coordinates.

  \noindent\textbf{Optimal computational complexity.}
  Lastly, from the computational perspective, we can leverage the best existing solutions for solving optimization $\widehat{\pi}^\star:=\text{argmax}_{\pi\in\Pi} \widehat{v}^\pi$. For example, with $\epsilon_{\text{opt}}>0$, as explained by \citet{agarwal2020model}, value iteration ends in $O(H\log \epsilon_{\text{opt}}^{-1})$ iteration 
  and takes at most $O(HSA)$ time after the model has been estimated with one pass of the data ($O(nH)$ time). We have a total computational complexity of $O(H^4/(d_m \epsilon^2) + H^2SA\log(1/\epsilon))$ time algorithm for obtaining the $\epsilon$-suboptimal policy using $n=O(H^4/(d_m \epsilon^2) $ episodes. This is essentially optimal because the leading term $H^4SA/\epsilon^2$ is required even to just process the data needed for the result to be information-theoretically possible. In comparison, the algorithm that obtains an exact empirical optimal policy $\widehat{\pi}^\star$, the SIMPLEX policy iteration runs in time $O(\text{poly}(H,S,A,n))$ \citep{ye2011simplex}.

\section{Proof overview }\label{sec:discussion} 



Our uniform convergence analysis in Section~\ref{sec:mart_con}, relies on creating an unbiased version of $\widehat{v}_\text{OPEMA}$ (which we call it $\widetilde{v}_\text{OPEMA}$) artificially and use concentration (Lemma~\ref{lem:sufficient_sample}) to guarantee $\widehat{v}_\text{OPEMA}$ is identical to $\widetilde{v}_\text{OPEMA}$ in most situations. By doing so we can reduce our analysis from $\sup_{\pi\in\Pi} \left|\widehat{v}^\pi-v^\pi\right|$ to $\sup_{\pi\in\Pi} \left|\widetilde{v}^\pi-v^\pi\right|$. Specifically, $\widetilde{v}^\pi$ replaces  $\widehat{P}_{t}$, $\widehat{r}_t$ in $\widehat{v}^\pi$ by its fictitious counterparts $\widetilde{P}_{t}$, $\widetilde{r}_t$, defined as:
{\small
	\begin{align*}
	\widetilde{r}_t(s_t,a_t)&=\widehat{r}_t(s_t,a_t)\mathbf{1}(E_t)+{r}_t(s_t,a_t)\mathbf{1}(E_t^c),\\
	\widetilde{P}_{t+1}(\cdot|s_t,a_t)&=\widehat{P}_{t+1}(\cdot|s_t,a_t)\mathbf{1}(E_t)+{P}_{t+1}(\cdot|s_t,a_t)\mathbf{1}(E_t^c).
	\end{align*}
}where $E_t$ denotes the event $\{n_{s_t,a_t}\geq nd^\mu_t(s_t,a_t)/2\}$. This is saying, if observation $n_{s_t,a_t}$ is large enough ($E_t$ is true), we use $\widehat{P}$; otherwise we directly use $P$ instead. This track helps dealing with out-of-sample state-action pairs. The next key is the martingale decomposition (Lemma~\ref{lem:martingale_decompose}). On one hand, by using the structure of {\small $\sup_{\pi\in\Pi}\langle{V^\pi_h,(\widetilde{T}_h-T_h)\widetilde{d}^\pi_{h-1}}\rangle$} we can relax it into a ``Rademacher-type complexity'' which corresponds to $\tilde{O}(\sqrt{H^4S/d_m n})$ term in Theorem~\ref{thm:rad}. On the other hand, this decomposition has a natural martingale structure so martingale concentration inequalities can be appropriately applied, \emph{i.e.} Theorem~\ref{thm:single_finer_bound}. In addition, each term {\small $\langle{V^\pi_h,(\widetilde{T}_h-T_h)\widetilde{d}^\pi_{h-1}}\rangle$} separates the non-stationary policy into two parts with empirical distribution only depends on $\pi_{1:h-1}$ that governs how the data ``roll in'' and the long term value function $V^\pi_h$ only depends on $\pi_{h:H}$ that governs how the reward ``roll out''. 

For local uniform convergence, by Bellman equations we can obtain a similar decomposition on $Q$-function:
{
\begin{align*}
\widehat{Q}^\pi_t-Q^\pi_t&=\sum_{h=t+1}^H \Gamma_{t+1:h-1}^\pi(\widehat{P}_{h}-{P}_{h}) \widehat{V}^\pi_{h},
\end{align*}}
where $\Gamma_{t:h}^\pi=\prod_{i=t}^hP^\pi_i$ is the multi-step state-action transition and $\Gamma_{t+1:t}^\pi:=I$. Since $\pi$ is any policy in ${\Pi}_1$ which may dependent on $\mathcal{D}'$ so we cannot directly apply concentration inequalities on $(\widehat{P}_{h}-{P}_{h}) \widehat{V}^\pi_{h}$. Instead, we overcome this hurdle by doing concentration on $(\widehat{P}_{h}-{P}_{h}) \widehat{V}^{\widehat{\pi}^*}_{h}$ since $\widehat{V}^{\widehat{\pi}^*}_{h}$ and $\widehat{P}_{h}$ are independent, and we connect $\widehat{V}^{\widehat{\pi}^*}_{h}$ back to $\widehat{V}^\pi_{h} $ by using they are $\epsilon_\mathrm{opt}$ close (Theorem~\ref{thm:local_uni_opt}). This idea helps avoiding the technicality of absorbing MDP used in \cite{agarwal2020model} for infinite horizon case because of our non-stationary transition setting. For the uniform convergence lower bound, our analysis relies on reducing the problem to identifying $\epsilon$-optimal policy and proving any algorithm that learns a $\epsilon$-optimal policy requires at least $\Omega(H^3/d_m\epsilon^2)$ episodes in the non-stationary episodic setting. Previously, \cite{jiang2017contextual} proves the $\Omega(HSA/\epsilon^2)$ lower bound with assumption $\sum_{i=1}^Hr_i\leq 1$. Our proof uses a modified version of their hard-to-learn MDP instance to achieve the desired result. To produce extra $H^2$ dependence, we leverage the Assumption~\ref{assume1} that $\sum_{i=1}^Hr_i$ may be of order $O(H)$. We only present the high-level ideas here due the space constraint, detailed proofs are explicated in order in Appendix~\ref{sec:the_uni_crude}, \ref{app_sec:mar}, \ref{sec:E}, \ref{sec:lower}.


\section{Numerical simulation}\label{sec:simulation} 


In this section we use a simple simulated environment to empirically demonstrate the correct scaling in $H$. Direct evaluating $\sup_{\pi\in\Pi} |\widehat{v}^\pi-v^\pi|$ empirically is computationally infeasible since the policy classes we considered here contains either $A^{HS}$ or $\infty$ many policies. 
Instead, in the experiment we will plot the sub-optimality gap $ |v^\star-v^{\widehat{\pi}^\star}|$ with $\hat{\pi}^\star$ being the outputs of policy planning algorithms. The sub-optimality gap is considered as a surrogate for the lower bound of $\sup_{\pi\in\Pi} |\widehat{v}^\pi-v^\pi|$.
Concretely, the non-stationary MDP has $2$ states $s_0,s_1$ and $2$ actions $a_1,a_2$ where action $a_1$ has probability $1$ going back the current state and for action $a_2$, there is one state s.t. after choosing $a_2$ the dynamic transitions to both states with equal probability $\frac{1}{2}$ and the other one has asymmetric probability assignment ($\frac{1}{4}$ and $\frac{3}{4}$). The transition after choosing $a_2$ is changing over different time steps therefore the MDP is non-stationary and the change is decided by a sequence of pseudo-random numbers (Figure~\ref{fig:mdp} shows the transition kernel at a particular time step). Moreover, to make the learning problem non-trivial we use non-stationary rewards with $4$ categories, \emph{i.e.} $r_t(s,a)\in\{\frac{1}{4},\frac{2}{4},\frac{3}{4},1\}$ and assignment of $r_t(s,a)$ for each value is changing over time (see Section~\ref{sec:simluation_detail} in appendix for more details). Lastly, the logging policy in Figure~\ref{fig:different_H} is uniform with $\mu_t(a_1|s)=\mu_t(a_2|s)=\frac{1}{2}$ for both states.

Figure~\ref{fig:different_H} use a fixed number of episodes $n=2048$ while varying $H$ to examine the horizon dependence for uniform OPE. We can see for fixed pointwise OPE with OPEMA (blue line), $|v^\pi-\widehat{v}^\pi|$ scales as $O(\sqrt{H^2})$ which reflects the bound of Lemma~\ref{thm:single_finer_bound}; for the model-based planning, we ran both VI and PI until they converge to the empirical optimal policy $\widehat{\pi}^\star$. 
The figure shows that for this MDP example $|v^\star-v^{\widehat{\pi}^\star}|$ scales as $O(\sqrt{H^3/d_m})$ for fixed $n$ since it is parallel to the reference magenta line. This fact empirically shows $O(\sqrt{H^3/d_m})$ bound is required 
confirms the scaling of our theoretical results. 

\begin{figure}
	\centering     
	\subfigure[A non-stationary MDP]{\label{fig:mdp}\includegraphics[width=0.45\linewidth]{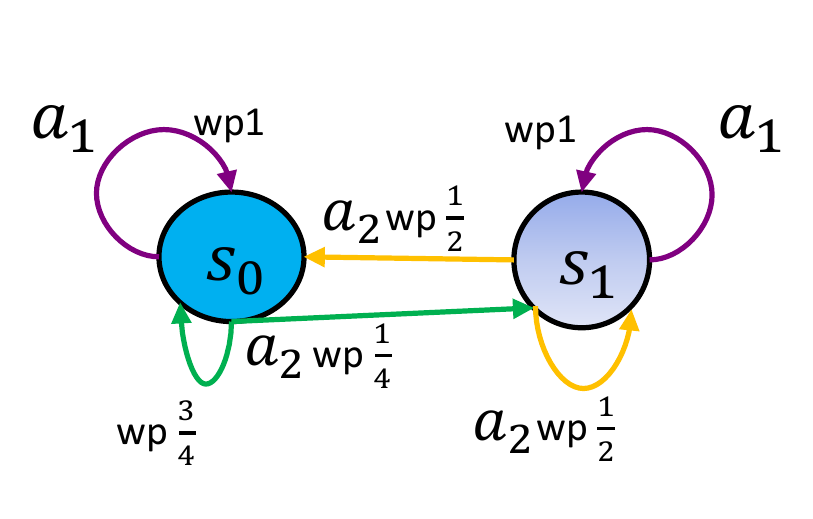}}
	\subfigure[RMSE vs. Horizon $H$]{\label{fig:different_H}\includegraphics[width=0.5\linewidth]{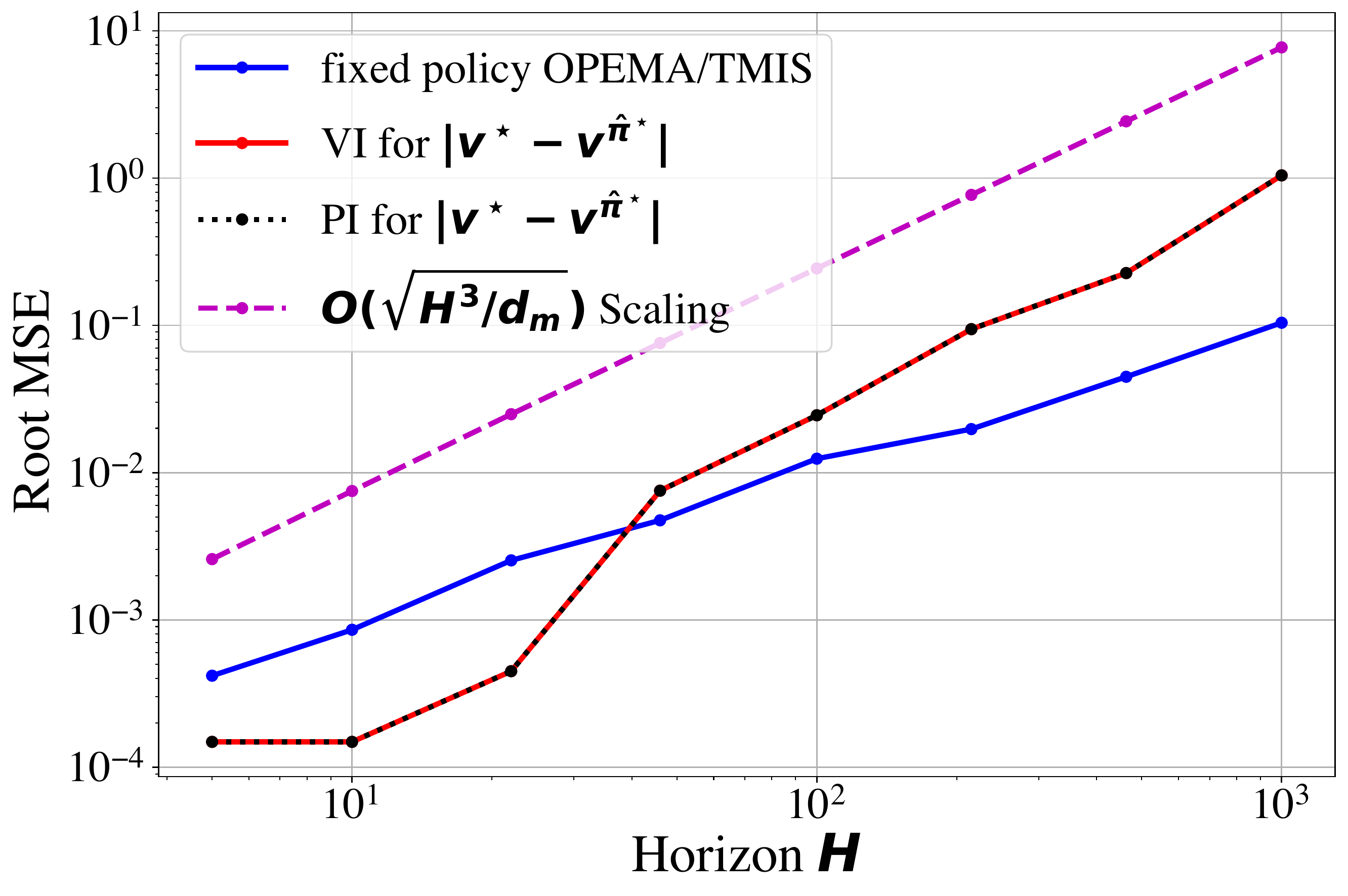}}
	\caption{Log-log plot  showing the dependence on horizon of uniform OPE and pointwise OPE via learning ($|{v}^{\star}-v^{\widehat{\pi}^\star}|$) over a non-stationary MDP example.}
	\label{fig:main}
\end{figure}


\section{Discussion }\label{sec:dis}


\noindent\textbf{The efficiency of model-based methods.}
There had been a long-lasting debate about model-based vs model-free methods in RL. The model-based methods were considered inefficient in both space and sample complexity, due to the need to represents the transition kernel in $O(HS^2A)$. Most sample-efficient methods with the right dependence in $S$ are model-free methods that directly represents and updates the $Q$-function. Our analysis reveals that direct model-based plug-in estimator is optimal in both pointwise and uniform prediction problems, which helps to correct the commonly held misunderstanding that purely model plug-in estimator is loose due to simulation lemma.


\noindent\textbf{Uniform OPE that depends on $\pi$.}
In this paper, we primarily consider obtaining uniform bound independent to $\pi$, however, given a logging policy $\mu$, it is often easier  to evaluate some policies than others, as is revealed in the pointwise OPE bound of \citep{yin2020asymptotically}. Specifically, obtaining a high probability bound of the form $\sup_\pi  \frac{\sqrt{n}|\hat{v}^\pi -  v^\pi|}{ \gamma(\pi,\mu,M,\delta)} \leq C $ for some function $\gamma$ and constant $C$ would  be of great interest. We could already get such a bound by applying union bound to the data-dependent high probability pointwise convergence of either \citep{yin2020asymptotically} or \citep{duan2020minimax} but it comes with an additional $O(S)$ factor. Characterizing the optimal per-instance OPE bound is an interesting future direction.

\noindent\textbf{Simulation Lemma.} Our result can be viewed as a strengthened version of the \emph{simulation lemma} \citep{kearns2002near} (see also the exposition in \citep{jiang2018notes}, which uses similar notations to us).  The OPE bound that can be obtained by applying the simulation lemma is
{\small\begin{align*}
|\widehat{v}^\pi-v^\pi|&\leq H^2\sup_{t,s_t,a_t}\norm{\widehat{P}(\cdot|s_t,a_t)-P(\cdot|s_t,a_t)}_1\leq \widetilde{O}\left(\sqrt{\frac{H^4S^2}{nd_m}}\right)
\end{align*}} which implies an episode complexity\footnote{See Section~\ref{sec:app_sl} for more calculation details.} of $\widetilde{O}(H^4S^2/d_m\epsilon^2)$. The main limitation of the simulation lemma is that it does not distinguish between pointwise / uniform convergence (and their bound is in fact a uniform OPE bound), thus will suffer from a loose bound when applied to fixed policies or data-dependent policies that qualify for the smaller policy classes that we considered. For example, our Lemma~\ref{thm:single_finer_bound} shows that for the same plug-in estimator, the bound improves to $\widetilde{O}(H^2/d_m\epsilon^2)$ for pointwise OPE and Theorem~\ref{thm:local_uni_opt} shows that we can knock out a factor of $HS^2$ in the uniform convergence of \emph{near empirically optimal} policies.  Finally,  there is a factor of $S$improvement in the global policy class unconditionally. These savings can be used as drop-in replacements to many instances where the simulation lemma is applied to improve the parameters of the analysis therein.

\section{Conclusion}

This work represents the first systematic study of uniform convergence in offline policy evaluation. We derive near optimal results for three representative policy classes. By viewing offline policy evaluation from the uniform convergence perspective, we are able to unify two central topics in offline RL, OPE and offline learning while establishing optimal rates in a subset of these settings including the first rate-optimal offline reinforcement learning method. 
The work focuses on the episodic tabular MDP with nonstationary transitions. Carrying out the same analysis for the stationary transition case, infinite horizon case, as well as the linear MDP setting is highly tractable with the techniques presented. Formalizing these is left as a future work. More generally, a natural complexity measure for the policy class of RL remains elusive. We hope the work could inspire a more general statistical learning theory for RL in the near future. 

\subsection*{Acknowledgment}
The research is partially supported by NSF Awards \#2007117 and \#1934641.  The authors thank Christopher Dann for a discussion related to sample complexity lower bounds in the pointwise bounded reward case; and Tengyang Xie and Nan Jiang for clarifying the scaling in the sample complexity of their results in \citep{xie2020q} with us.

\bibliographystyle{apa-good.bst}
\bibliography{sections/stat_uniform_rl}

\appendix

\clearpage
\begin{center}
	{\LARGE Appendix}
\end{center}


\section{Technical lemmas}

\begin{lemma}[Multiplicative Chernoff bound \cite{chernoff1952measure}]\label{lem:chernoff_multiplicative}
	Let $X$ be a Binomial random variable with parameter $p,n$. For any $\delta>0$, we have that 
	$$
	\P[X <(1-\delta)pn] < \left( \frac{e^{-\delta}}{(1-\delta)^{1-\delta}} \right)^{np}.  
	$$
\end{lemma}
A slightly looser bound that suffices for our propose:
$$
\P[X < (1-\delta)pn] <  e^{-\frac{\delta^2 pn}{2}}.
$$

\begin{lemma}[Hoeffding’s Inequality \cite{sridharan2002gentle}]\label{lem:hoeffding_ineq}
	Let $x_1,...,x_n$ be independent bounded random variables such that $\E[x_i]=0$ and $|x_i|\leq \xi_i$ with probability $1$. Then for any $\epsilon >0$ we have 
	$$
	\P\left( \frac{1}{n}\sum_{i=1}^nx_i\geq \epsilon\right) \leq e^{-\frac{2n^2\epsilon^2}{\sum_{i=1}^n\xi_i^2}}.
	$$
\end{lemma}

\begin{lemma}[Bernstein’s Inequality]\label{lem:bernstein_ineq}
	Let $x_1,...,x_n$ be independent bounded random variables such that $\E[x_i]=0$ and $|x_i|\leq \xi$ with probability $1$. Let $\sigma^2 = \frac{1}{n}\sum_{i=1}^n \mathrm{Var}[x_i]$, then with probability $1-\delta$ we have 
	$$
	\frac{1}{n}\sum_{i=1}^n x_i\leq \sqrt{\frac{2\sigma^2\cdot\log(1/\delta)}{n}}+\frac{2\xi}{3n}\log(1/\delta)
	$$
\end{lemma}

\begin{lemma}[Mcdiarmid’s Inequality \citep{sridharan2002gentle}]\label{lem:mcdiarmid_ineq}
		Let $x_1,...,x_n$ be independent random variables and $S : X^n \rightarrow \R$ be a measurable function which is invariant under permutation and let the random variable $Z$ be given by $Z = S(x_1, x_2, ..., x_n)$. Assume $S$ has bounded difference: \textit{i.e.}
	\[
	\sup_{x_1,...,x_n,x'_i}|S(x_1,...,x_i,...,x_n)-S(x_1,...,x_i',...,x_n)|\leq \xi_i,
	\]
	then for any $\epsilon>0$ we have
	\[
	\P(|Z-\E[Z]|\geq\epsilon)\leq 2e^{-\frac{2\epsilon^2}{\sum_{i=1}^n \xi_i^2}}.
	\]
\end{lemma}

\begin{lemma}[Azuma-Hoeffding inequality]\label{lem:azuma_hoeff}
Suppose $X_k$, $k=1,2,3,...$ is a martingale and $|X_k-X_{k-1}|\leq c_k$ almost surely. Then for all positive integers $N$ and any $\epsilon>0$, 
\[
\P(|X_N-X_0|\geq\epsilon)\leq 2 e^{-\frac{\epsilon^2}{2\sum_{i=1}^Nc_i^2}}.
\]
\end{lemma}

\begin{lemma}[Freedman's inequality \cite{tropp2011freedman}]\label{lem:freedman}
	Let $X$ be the martingale associated with a filter $\mathcal{F}$ (\textit{i.e.} $X_i=\E[X|\mathcal{F}_i]$) satisfying $|X_i-X_{i-1}|\leq M$ for $i=1,...,n$. Denote $W:=\sum_{i=1}^n\Var(X_i|\mathcal{F}_{i-1})$  then we have 
	\[
	\P(|X-\E[X]|\geq\epsilon,W\leq \sigma^2)\leq 2 e^{-\frac{\epsilon^2}{2(\sigma^2+M\epsilon/3)}}.
	\]
	Or in other words, with probability $1-\delta$,
	\[
		|X-\E[X]|\leq \sqrt{{8\sigma^2\cdot\log(1/\delta)}}+\frac{2M}{3}\cdot\log(1/\delta), \quad\text{Or} \quad W\geq \sigma^2.
	\]
\end{lemma}

\begin{lemma}[Best arm identification lower bound \cite{krishnamurthy2016pac}]\label{lem:bandit}
	For any $A \geq 2$ and $\tau \leq \sqrt{1/8}$ and any best arm identification algorithm that produces an estimate $\hat{a}$, there exists a multi-arm bandit problem for which the best arm $a^\star$ is $\tau$  better than all others, but $\P[\hat{a}\neq a^\star]\geq 1/3$ unless the number of samples $T$ is at least $\frac{A}{72\tau^2}$ .
\end{lemma}


\section{On error metric for OPE}

In this section, we discuss the metric considered in this work. Traditionally, most works directly use \textit{Mean Square Error} (MSE) $\E[(\widehat{v}^\pi-v^\pi)^2]$ as the criterion for measuring OPE methods \emph{e.g.} \cite{thomas2016data,thomas2015safe,thomas2017predictive,farajtabar18a}, or equivalently, by proposing unbiased estimators and discussing its variance \emph{e.g.} \cite{jiang2016doubly}. Alternately, one can consider bounding the absolute difference between $v^\pi$ and $\widehat{v}^\pi$ with high probability (\emph{e.g.} \citet{duan2020minimax}), \textit{i.e.} $|\widehat{v}^\pi-v^\pi|\leq \epsilon_\mathrm{prob}$ \textit{w.h.p.} Generally speaking, high probability bound can be seen as a stricter criterion compared to MSE  since 
\begin{align*}
\E[(\widehat{v}^\pi-v^\pi)^2]&=\E[(\widehat{v}^\pi-v^\pi)^2\mathds{1}_{E}]+\E[(\widehat{v}^\pi-v^\pi)^2\mathds{1}_{E^c}]\\
&\leq \epsilon_\mathrm{prob}(\delta)^2\cdot (1-\delta)+H^2\cdot \delta,
\end{align*}
where $E$ is the event that $\epsilon_\mathrm{prob}$ error holds and $\delta$ is the failure probability. As a result, if both $\delta$ and $\epsilon_\mathrm{prob}(\delta)$ can be controlled small, then the high probability bound implies a result for MSE bound. This is realistic, since $\delta$ mostly appears inside the logarithmic term of $\epsilon_\mathrm{prob}(\delta)$ so the second term can be scaled to sufficiently small without affecting the polynomial dependence for the first term.

\newpage
\begin{table}[H]
	\caption{Summary of Uniform OPE results for $H$-horizon non-stationary setting}
	\centering
	\resizebox{\linewidth}{!}{%
		\begin{tabular}{llll}
			\toprule
			\cmidrule(r){1-3}
			Method/Analysis     &Policy class& Guarantee     & Sample complexity \\
			\midrule
			
			Simulation Lemma & All policies &$\epsilon$-uniform convergence & $O(H^4S^2/d_m\epsilon^2)$    \\
			Theorem~\ref{thm:rad} &All policies & $\epsilon$-uniform convergence  & $O(H^4S/d_m\epsilon^2)$     \\
			Theorem~\ref{thm:finer_bound} &All deterministic policies  & $\epsilon$-uniform convergence  & $O(H^3S/d_m\epsilon^2)$      \\
			Theorem~\ref{thm:local_uni_opt} &local policies   & $\epsilon$-uniform convergence  & $O(H^3/d_m\epsilon^2)$      \\
			Minimax lower bound (Theorem~\ref{thm:uni_lower})  & ---------------   & over class $\mathcal{M}_{d_m}$     & $\Omega(H^3/d_m\epsilon^2)$  \\
			\bottomrule
		\end{tabular}
	}

\end{table}

\section{Some preparations}

In this section we present some results that are critical for proving the main theorems.

\begin{lemma}\label{lem:sufficient_sample} For any $0<\delta<1$, there exists an absolute constant $c_1$ such that when total episode $n>c_1 \cdot 1/d_m\cdot \log(HSA/\delta)$, then with probability $1-\delta$,
	\[
	n_{s_t,a_t}\geq n\cdot d^\mu_t(s_t,a_t)/2,\quad\forall \; s_t,a_t.
	\]
\end{lemma}

If state $s_t$ is not accessible, then $n_{s_t,a_t}=d^\mu_t(s_t,a_t)=0$ so the lemma holds trivially.\footnote{In general, non-accessible state will not affect our results so to make our presentation succinct we will not mention non-accessible state for the rest of paper unless necessary.}

\begin{proof}[Proof of Lemma~\ref{lem:sufficient_sample}]
	Define $E:=\{\exists t, s_t,a_t \;\text{s.t.}\; n_{s_t,a_t} <  n d_t^\mu(s_t,a_t)/2 \}$. Then combining the multiplicative Chernoff bound (Lemma~\ref{lem:chernoff_multiplicative} in the Appendix) and a union bound over each $t$,$s_t$ and $a_t$, we obtain
	\begin{align*}
	\P[E] &\leq \sum_{t}\sum_{s_t}\sum_{a_t} \P[n_{s_t,a_t} < n d^{\mu}_t(s_t,a_t)/2] \\
	&\leq HSA\cdot e^{-\frac{ n\cdot \min_{t,s_t,a_t} d^{\mu}_t(s_t,a_t)}{8}}=HSA\cdot e^{-\frac{ n\cdot d_m}{8}}:=\delta
	\end{align*}
	solving this for $n$ then provides the stated result.
\end{proof}

Now we define: $N:=\min_{t,s_t,a_t}n_{s_t,a_t}$, then above implies $N\geq nd_m/2$ (recall $d_m$ in Assumption~\ref{assume2}). Now we aggregate only the first $N$ pieces of data in each state-action $(s_t,a_t)$\footnote{Note we can do this since by definition $N\leq n_{s_t,a_t}$ for all $s_t,a_t$.} of off-policy data $\mathcal{D}$ and they consist of a new dataset $\mathcal{D}^\prime=\lbrace(s_t,a_t,s_{t+1}^{(i)},r_t^{(i)}):i=1,...,N;t\in[H];s_t\in\mathcal{S},a_t\in\mathcal{A}\rbrace$, and is a subset of $\mathcal{D}$. For the rest of paper, we will use either $\mathcal{D}^\prime$ or the original $\mathcal{D}$ to create OPEMA $\widehat{v}^\pi$ (only for theoretical analysis purpose). Whether $\mathcal{D}$ or $\mathcal{D}'$ is used will be stated clearly in each context. 

\begin{remark}
	It is worth mentioning that when use $\mathcal{D}'$ to construct $\widehat{v}^\pi$, $n_{s_t,a_t}^{\mathcal{D}'}=N$ for all $s_t,a_t$. Also, $N:=\min n^\mathcal{D}_{s_t,a_t}$ (note $n^\mathcal{D}_{s_t,a_t}$ is the count from $\mathcal{D}$) itself is a random variable and in the extreme case we could have $N=0$ and if that happens $\widehat{v}^\pi=0$ (since in that case $\widehat{P}_t\equiv 0$ and $\widehat{d}_t^\pi$ is degenerated). However, there is only tiny probability $N$ will be small, as guaranteed by Lemma~\ref{lem:sufficient_sample}.
\end{remark}

We wanted to point out that this technique of dropping certain amount of data, is not uncommon for analyzing model-based method in RL: e.g.
Rmax exploration \citep{brafman2002r} for online episodic setting (see [\cite{jiang2018notes}, Notes on Rmax exploration] Section 2 Algorithm for
tabular MDP. The data they use is the “known set” $K$ with parameter $m$, in step3 data pairs observed more than $m$ times are not recorded).

\subsection{Fictitious OPEMA estimator.}

Similar to \cite{xie2019towards,yin2020asymptotically}, we introduce an unbiased version of $\widehat{v}^\pi$ to fill in the gap at $(s_t,a_t)$ where $n_{s_t,a_t}$ is small. Concretely, every component in $\widehat{v}^\pi$ is substituted by the fictitious counterpart, \emph{i.e.} $\widetilde{v}^\pi :=  \sum_{t=1}^H \langle \widetilde{d}_t^\pi,  \widetilde{r}_t^\pi\rangle$, with $\widetilde{d}_t^\pi =  \widetilde{P}_{t}^\pi   \widetilde{d}_{t-1}^\pi$ and $\widetilde{P}^\pi_t(s_t|s_{t-1})=\sum_{a_{t-1}}\widetilde{P}_t(s_t|s_{t-1},a_{t-1})\pi(a_{t-1}|s_{t-1})$.
 In particular, consider the high probability event in Lemma~\ref{lem:sufficient_sample}, \textit{i.e.} let $E_t$ denotes the event $\{n_{s_t,a_t}\geq nd^\mu_t(s_t,a_t)/2\}$\footnote{More rigorously, $E_t$ depends on the specific pair $s_t,a_t$ and should be written as $E_t(s_t,a_t)$. However, for brevity we just use $E_t$ and this notation should be clear in each context.}, then we define
\begin{align*}
\widetilde{r}_t(s_t,a_t)&=\widehat{r}_t(s_t,a_t)\mathbf{1}(E_t)+{r}_t(s_t,a_t)\mathbf{1}(E_t^c)\\
\widetilde{P}_{t+1}(\cdot|s_t,a_t)&=\widehat{P}_{t+1}(\cdot|s_t,a_t)\mathbf{1}(E_t)+{P}_{t+1}(\cdot|s_t,a_t)\mathbf{1}(E_t^c).
\end{align*}

Similarly, for the OPEMA estimator uses data $\mathcal{D}'$, the fictitious estimator is set to be 
\begin{align*}
\widetilde{r}_t(s_t,a_t)&=\widehat{r}_t(s_t,a_t)\mathbf{1}(E)+{r}_t(s_t,a_t)\mathbf{1}(E^c)\\
\widetilde{P}_{t+1}(\cdot|s_t,a_t)&=\widehat{P}_{t+1}(\cdot|s_t,a_t)\mathbf{1}(E)+{P}_{t+1}(\cdot|s_t,a_t)\mathbf{1}(E^c)
\end{align*}
where $E$ denote the event $\{N\geq nd_m/2\}$. 

$\widetilde{v}^\pi$ creates a bridge between $\widehat{v}^\pi$ and $v^\pi$ because of its unbiasedness and it is also bounded by $H$ (see Lemma B.3 and Lemma~B.5 in \cite{yin2020asymptotically} for those preliminary results). Also, $\widetilde{v}^\pi$ is identical to $\widehat{v}^\pi$ with high probability, as stated by the following lemma.  

\begin{lemma}\label{lem:zero_diff}
For any $0<\delta<1$, there exists an absolute constant $c_1$ such that when total episode $n>c_1 d_m\cdot \log(HSA/\delta)$, then with probability $1-\delta$,
\[
\sup_{\pi\in\Pi}|\widehat{v}^\pi-\widetilde{v}^\pi|=0.
\]
\end{lemma}

\begin{proof}
	This Lemma is a direct corollary of Lemma~\ref{lem:sufficient_sample} by considering the event $E_1:=\{\exists t, s_t,a_t \;\text{s.t.}\; n_{s_t,a_t} <  n d_t^\mu(s_t,a_t)/2 \}$ or $\{N <  n d_m/2 \}$ since $\widehat{v}^\pi$ and $\widetilde{v}^\pi$ are identical on $E_1^c$. 
\end{proof}

Note $\widehat{v}^\pi$ and $\widetilde{v}^\pi$ even equal to each other uniformly over all $\pi$ in $\Pi$. This is not surprising since only logging policy $\mu$ will decide if they are equal or not. This lemma shows how close $\widehat{v}^\pi$ and $\widetilde{v}^\pi$ are. Therefore in the following it suffices to consider the uniform convergence of $\sup_{\pi\in\Pi}|\widetilde{v}^\pi-{v}^\pi|$.

Next by using a fictitious analogy of state-action expression as in equation \eqref{eqn:estimated_value}, we have:

\begin{equation}\label{eqn:naive_decompose}
\begin{aligned}
\sup_{\pi\in\Pi}|\widetilde{v}^\pi-{v}^\pi|&=\sup_{\pi\in\Pi}|\sum_{t=1}^H\langle \widetilde{d}^\pi_t,\widetilde{r}_t\rangle-\sum_{t=1}^H\langle {d}^\pi_t,{r}_t\rangle|\\
&=\sup_{\pi\in\Pi}|\sum_{t=1}^H\langle \widetilde{d}^\pi_t,\widetilde{r}_t\rangle-\sum_{t=1}^H\langle \widetilde{d}^\pi_t,{r}_t\rangle+\sum_{t=1}^H\langle \widetilde{d}^\pi_t,{r}_t\rangle-\sum_{t=1}^H\langle {d}^\pi_t,{r}_t\rangle|\\
&\leq \underbrace{\sup_{\pi\in\Pi}|\sum_{t=1}^H\langle \widetilde{d}^\pi_t-d^\pi_t,{r}_t\rangle|}_{(*)}+\underbrace{\sup_{\pi\in\Pi}|\sum_{t=1}^H\langle \widetilde{d}^\pi_t, \widetilde{r}_t-r_t\rangle|}_{(**)}
\end{aligned}
\end{equation}

We first deal with $(**)$ by the following lemma.

\begin{lemma}\label{lem:app_1}
	We have with probability $1-\delta$:
	\[
	\sup_{\pi\in\Pi}|\sum_{t=1}^H\langle \widetilde{d}^\pi_t, \widetilde{r}_t-r_t\rangle|\leq O(\sqrt{\frac{H^2\log(HSA/\delta)}{n\cdot d_m}})
	\]
\end{lemma}

\begin{proof}[Proof of Lemma~\ref{lem:app_1}] 

	Since $|\langle \widetilde{d}^\pi_t, \widetilde{r}_t-r_t\rangle|\leq ||\widetilde{d}^\pi_t||_1\cdot||\widetilde{r}_t-r_t||_\infty $, we obtain 
	\[
	|\sum_{t=1}^H\langle \widetilde{d}^\pi_t, \widetilde{r}_t-r_t\rangle|\leq \sum_{t=1}^H ||\widetilde{d}^\pi_t||_1\cdot||\widetilde{r}_t-r_t||_\infty=\sum_{t=1}^H ||\widetilde{r}_t-r_t||_\infty,
	\]
	where we used $\widetilde{d}^\pi_t(\cdot)$ is a probability distribution. Therefore above expression further indicates $\sup_{\pi\in\Pi}|\sum_{t=1}^H\langle \widetilde{d}^\pi_t, \widetilde{r}_t-r_t\rangle|\leq \sum_{t=1}^H ||\widetilde{r}_t-r_t||_\infty$. Now by a union bound and Hoeffding inequality (Lemma~\ref{lem:hoeffding_ineq}),  
	\begin{align*}
	\P(\sup_t||\widetilde{r}_t-r_t||_\infty>\epsilon)&=\P(\sup_{t,s_t,a_t}|\widetilde{r}_t(s_t,a_t)-r_t(s_t,a_t)|>\epsilon)\\
	&\leq HSA\cdot\P(|\widetilde{r}_t(s_t,a_t)-r_t(s_t,a_t)|>\epsilon)\\
	&= HSA\cdot\P(|\widehat{r}_t(s_t,a_t)-r_t(s_t,a_t)|\mathds{1}(E_t)>\epsilon)\\
	&\leq 2HSA \cdot \E[\E[e^{-{2n_{s_t,a_t}\epsilon^2}}|E_t]]\\
	&\leq 2HSA\cdot \E[\E[e^{-{nd_m\epsilon^2}}|E_t]]= 2HSA\cdot e^{-{nd_m\epsilon^2}}:=\frac{\delta}{2}.
	\end{align*}
	where we use $\P(A)=\E[\mathds{1}_A]=\E[\E[\mathds{1}_A|X]]$. Solving for $\epsilon$, then it follows:
	
	\[
	\sup_{\pi\in\Pi}|\sum_{t=1}^H\langle \widetilde{d}^\pi_t, \widetilde{r}_t-r_t\rangle|\leq \sum_{t=1}^H ||\widetilde{r}_t-r_t||_\infty\leq O(\sqrt{\frac{H^2\log(HSA/\delta)}{n\cdot d_m}})
	\]
	with probability $1-\delta$. The case for $E=\{N\geq nd_m/2\}$ can be proved easily in a similar way.
	
\end{proof}
Note that in order to measure the randomness in reward, sample complexity $n$ only has dependence of order $H^2$, this result implies random reward will only cause error of lower order dependence in $H$. Therefore, in many RL literature deterministic reward is directly assumed. Next we consider $(*)$ in \eqref{eqn:naive_decompose} by decomposing $\sum_{t=1}^H\langle \widetilde{d}^\pi_t-d^\pi_t,{r}_t\rangle$ into a martingale type representation. This is the key for our proof since with it we can use either uniform concentration inequalities or martingale concentration inequalities to prove efficiency.

\subsection{Decomposition of $\sum_{t=1}^H\langle \widetilde{d}^\pi_t-d^\pi_t,{r}_t\rangle$}

Let $\widetilde{d}^\pi_t\in\R^{S\cdot A}$ denote the marginal state-action probability vector, $\pi_t\in\R^{(S\cdot A)\times S}$ is the policy matrix with $(\pi_t)_{(s_t,a_t),s_t}=\pi_t(a_t|s_t)$ and $(\pi_t)_{(s_t,a_t),s}=0$ for $s\neq s_t$. Moreover, let state-action transition matrix $T_t\in\R^{S\times(S\cdot A)}$ to be $(T_t)_{s_{t},(s_{t-1},a_{t-1})}=P_t(s_{t}|s_{t-1},a_{t-1})$, then we have 
\begin{align}
\widetilde{d}^\pi_t&=\pi_t\widetilde{T}_t\widetilde{d}^\pi_{t-1}\\
{d}^\pi_t&=\pi_t{T}_t{d}^\pi_{t-1}.
\end{align}
Therefore we have 
\begin{align}\label{eq:decomp}
\widetilde{d}^\pi_t-{d}^\pi_t&=\pi_t(\widetilde{T}_t-T_t)\widetilde{d}^\pi_{t-1}+\pi_t T_t(\widetilde{d}^\pi_{t-1}-d^\pi_{t-1})
\end{align}
recursively apply this formula, we have 
\begin{equation}\label{difference_density}
\widetilde{d}^\pi_t-{d}^\pi_t=\sum_{h=2}^{t}\Gamma_{h+1:t}\pi_h(\widetilde{T}_h-T_h)\widetilde{d}^\pi_{h-1}+\Gamma_{1:t}(\widetilde{d}^\pi_1-d^\pi_1)
\end{equation}
where $\Gamma_{h:t}=\prod_{v=h}^t\pi_vT_v$ and $\Gamma_{t+1:t}:=1$. Now let $X=\sum_{t=1}^H\langle r_t,\widetilde{d}^\pi_t-{d}^\pi_t\rangle$, then we have the following:

\begin{theorem}[martingale decomposition of $X$: Restate of the fictitious version of Lemma~\ref{lem:martingale_decompose}]\label{thm:martingale_decompose} We have:
	\[
	X=\sum_{h=2}^H\langle{V^\pi_h(s),((\widetilde{T}_h-T_h)\widetilde{d}^\pi_{h-1}})(s)\rangle+\langle{V^\pi_1(s),(\widetilde{d}^\pi_1-d^\pi_1)(s)}\rangle,
	\]
	where the inner product is taken w.r.t states.
	
\end{theorem}

\begin{proof}[Proof of Theorem~\ref{thm:martingale_decompose}]
	By applying \eqref{difference_density} and the change of summation, we have 
	\[
	\begin{aligned}
	X&=\sum_{t=1}^H\left(\sum_{h=2}^t \langle{r_t,\Gamma_{h+1:t}\pi_h(\widetilde{T}_h-T_h)\widetilde{d}^\pi_{h-1}}\rangle+\langle{r_t, \Gamma_{1:t}(\widetilde{d}^\pi_1-d^\pi_1)}\rangle \right)\\
	&=\sum_{t=1}^H\left(\sum_{h=2}^t \langle{r_t,\Gamma_{h+1:t}\pi_h(\widetilde{T}_h-T_h)\widetilde{d}^\pi_{h-1}}\rangle \right)+\sum_{h=1}^H \langle{r_t, \Gamma_{1:t}(\widetilde{d}^\pi_1-d^\pi_1)}\rangle\\
	&=\sum_{t=2}^H\left(\sum_{h=2}^t \langle{r_t,\Gamma_{h+1:t}\pi_h(\widetilde{T}_h-T_h)\widetilde{d}^\pi_{h-1}}\rangle \right)+\sum_{h=1}^H \langle{r_t, \Gamma_{1:t}(\widetilde{d}^\pi_1-d^\pi_1)}\rangle\\
	&=\sum_{h=2}^H\left(\sum_{t=h}^H \langle{r_t,\Gamma_{h+1:t}\pi_h(\widetilde{T}_h-T_h)\widetilde{d}^\pi_{h-1}}\rangle \right)+\sum_{h=1}^H \langle{(\pi_1^T\Gamma_{1:t}^Tr_t)(s), (\widetilde{d}^\pi_1-d^\pi_1)(s)}\rangle\\
	&=\sum_{h=2}^H\left( \langle{\underbrace{\sum_{t=h}^H \pi_h^T\Gamma_{h+1:t}^Tr_t}_{V^\pi_h(s)},(\widetilde{T}_h-T_h)\widetilde{d}^\pi_{h-1}}\rangle \right)+ \langle{\underbrace{(\sum_{h=1}^H \pi_1^T\Gamma_{1:t}^Tr_t)(s)}_{V^\pi_1(s)}, (\widetilde{d}^\pi_1-d^\pi_1)(s)}\rangle\\
	\end{aligned}
	\]
\end{proof}

\section{Proof of uniform convergence in OPE with full policies using standard uniform concentration tools: Theorem~\ref{thm:rad}}
\label{sec:the_uni_crude}
As a reminder for the reader, the OPEMA estimator used in this section is with data subset $\mathcal{D}'$. Also, by Lemma~\ref{lem:app_1} we only need to consider $\sup_{\pi\in\Pi}|\sum_{t=1}^H\langle \widetilde{d}^\pi_t-d^\pi_t,{r}_t\rangle|$.

\begin{theorem}\label{thm:mcdiarmid} 
	
	There exists an absolute constant $c$ such that if $n>c\cdot \frac{1}{d_m}\cdot \log(HSA/\delta)$, then with probability $1-\delta$, we have:
	\[
	\sup_{\pi\in\Pi}\left|\sum_{t=1}^H\langle \widetilde{d}^\pi_t-d^\pi_t,{r}_t\rangle\right|\leq O(\sqrt{\frac{H^4\log(HSA/\delta)}{nd_m}})+\E\left[\sup_{\pi\in\Pi}\left|\sum_{t=1}^H\langle \widetilde{d}^\pi_t-d^\pi_t,{r}_t\rangle\right|\right]
	\]
\end{theorem}

\begin{proof}[Proof of Theorem~\ref{thm:mcdiarmid}]
	Note in data $\mathcal{D}^\prime=\lbrace(s_t,a_t,s_{t+1}^{(i)}):i=1,...,N;t=1,...,H;s_t\in\mathcal{S},a_t\in\mathcal{A}\rbrace$\footnote{Here we do not include $r^{(i)}_t$ since the quantity $\sup_{\pi\in\Pi}|\sum_{t=1}^H\langle \widetilde{d}^\pi_t-d^\pi_t,{r}_t\rangle|$ only contains the mean reward function $r_t$.}, not only $s_{t+1}^{(i)}$ but also $N$ are random variables.
	
	We first conditional on $N$, then $(s_t,a_t,s_{t+1}^{(i)})$'s are independent samples for all $i,s_t,a_t$ since any sample will not contain information about other samples. Therefore we can regroup $\mathcal{D}'$ into $N$ independent samples with $\mathcal{D}'=\lbrace X^{(i)}:i=1,...,N\rbrace$ where $X^{(i)}=\{(s_t,a_t,s^{(i)}_{t+1}),t=1,...,H;s_t\in\mathcal{S},a_t\in\mathcal{A}\}$. Now for any $i_0$, change $X^{(i_0)}$ to $X'^{(i_0)}=\{(s_t,a_t,s'^{(i_0)}_{t+1}),t=1,...,H;s_t\in\mathcal{S},a_t\in\mathcal{A}\}$ and keep the rest $N-1$ data the same, use this data to create new estimator with state-action transition $\widetilde{d}'^\pi$, then we have 
	\begin{align*}
	&\left|\sup_{\pi\in\Pi}|\sum_{t=1}^H\langle \widetilde{d}^\pi_t-d^\pi_t,{r}_t\rangle|-\sup_{\pi\in\Pi}|\sum_{t=1}^H\langle \widetilde{d}'^\pi_t-d^\pi_t,{r}_t\rangle|\right|\\
	\leq &\sup_{\pi\in\Pi}\left||\sum_{t=1}^H\langle \widetilde{d}^\pi_t-d^\pi_t,{r}_t\rangle|-|\sum_{t=1}^H\langle \widetilde{d}'^\pi_t-d^\pi_t,{r}_t\rangle|\right|\\
	\leq&\sup_{\pi\in\Pi}\left|\sum_{t=1}^H\langle \widetilde{d}^\pi_t-d^\pi_t,{r}_t\rangle-\sum_{t=1}^H\langle \widetilde{d}'^\pi_t-d^\pi_t,{r}_t\rangle\right|\\
	=&\sup_{\pi\in\Pi}\left|\sum_{t=1}^H\langle \widetilde{d}^\pi_t-\widetilde{d}'^\pi_t,{r}_t\rangle\right|\\
	=&\sup_{\pi\in\Pi}\left|\sum_{h=2}^H\langle{\widetilde{V}'^\pi_h,(\widetilde{T}_h-\widetilde{T}'_h)\widetilde{d}^\pi_{h-1}}\rangle+\langle{\widetilde{V}'^\pi_1,\widetilde{d}^\pi_1-\widetilde{d}'^\pi_1}\rangle\right|,
	\end{align*}
	where the last equation comes from the trick that substitutes $d^\pi_t$ by $\widetilde{d}'^\pi_t$ in Theorem~\ref{thm:martingale_decompose}. By definition, the above equals to
	\begin{align*}
	=&\sup_{\pi\in\Pi}\left|\sum_{h=2}^H\langle{\widehat{V}'^\pi_h,(\widehat{T}_h-\widehat{T}'_h)\widehat{d}^\pi_{h-1}}\rangle+\langle{\widehat{V}'^\pi_1,\widehat{d}^\pi_1-\widehat{d}'^\pi_1}\rangle\right|\cdot \mathds{1}(E)\\
	\leq & \sup_{\pi\in\Pi}\left(\sum_{h=2}^H||(\widehat{T}_h-\widehat{T}'_h)^T\widehat{V}'^\pi_h||_\infty||\widehat{d}^\pi_{h-1}||_1+|\langle{\widehat{V}'^\pi_1,\widehat{d}^\pi_1-\widehat{d}'^\pi_1}\rangle|\right)\cdot \mathds{1}(E)\\
	\leq & \sup_{\pi\in\Pi}\left(\sum_{h=2}^H||(\widehat{T}_h-\widehat{T}'_h)^T\widehat{V}'^\pi_h||_\infty+|\langle{\widehat{V}'^\pi_1,\widehat{d}^\pi_1-\widehat{d}'^\pi_1}\rangle|\right)\cdot \mathds{1}(E)
	\end{align*}
	Note the change of a single $X^{(i_0)}$ will only change two entries of each row of $(\widehat{T}_h-\widehat{T}'_h)^T$ by $1/N$ since with data $\mathcal{D}'$, $n_{s_t,a_t}=N$ for all $s_t,a_t$. Or in other words, given $E$,
	\[
	\widehat{T}_h^T-\widehat{T}'^T_h=\begin{bmatrix}
	0&...&0& \frac{1}{N}&0&...&-\frac{1}{N}&...&0\\
	0&\frac{1}{N}&0&...&-\frac{1}{N}&...&...&...&0\\
	........\\
	-\frac{1}{N}&0&...&0&...&...&0&...&\frac{1}{N}
	\end{bmatrix},
	\]
	where the locations of $1/N,-1/N$ in each row are random as it depends on how different is $X'^{(i_0)}$ from $X^{(i_0)}$. However, based on this fact, it is enough for us to guarantee
	\[
	||(\widehat{T}_h-\widehat{T}'_h)^T\widehat{V}'^\pi_h||_\infty\leq \frac{2}{N}||\widehat{V}'^{\pi}_h||_\infty\leq \frac{2}{N}(H-h+1)\leq \frac{2}{N}H
	\]
	and same result holds for $|\langle{\widehat{V}'^\pi_1,\widehat{d}^\pi_1-\widehat{d}'^\pi_1}\rangle|\leq 2H/N$ given $N$.
	
	Combine all the results above, for a single change of $X^{(i_0)}$ we have
	\[
	\left|\sup_{\pi\in\Pi}|\sum_{t=1}^H\langle \widetilde{d}^\pi_t-d^\pi_t,{r}_t\rangle|-\sup_{\pi\in\Pi}|\sum_{t=1}^H\langle \widetilde{d}'^\pi_t-d^\pi_t,{r}_t\rangle|\right|\leq 2\frac{H^2}{N}\mathds{1}(E)\leq 2\frac{H^2}{N}
	\]
	for any fixed $N$. If we let $Z=S(X^{(1)},...,X^{(N)})=\sup_{\pi\in\Pi}|\sum_{t=1}^H\langle \widetilde{d}^\pi_t-d^\pi_t,{r}_t\rangle|$, then for a given $N$ by independence and above bounded difference condition we can apply Mcdiarmid inequality Lemma~\ref{lem:mcdiarmid_ineq} (where $\xi_i=2H^2/N$) to obtain 
	
	\begin{equation}\label{eqn:mcdiarmid_final}
	\P(|Z-\E[Z]|\geq \epsilon|N)\leq 2e^{-\frac{N\epsilon^2}{2H^4}}:=\frac{\delta}{2}
	\end{equation}
	Now note when $n>O(\frac{1}{d_m}\cdot \log(HSA/\delta))$, by Lemma~\ref{lem:sufficient_sample} we can obtain $N>nd_m/2$ with probability $1-\delta/2$, combining this result and solving $\epsilon$ in \eqref{eqn:mcdiarmid_final}, we have 
	\[
	\sup_{\pi\in\Pi}\left|\sum_{t=1}^H\langle \widetilde{d}^\pi_t-d^\pi_t,{r}_t\rangle\right|\leq O(\sqrt{\frac{H^4\log(HSA/\delta)}{n\cdot d_m}})+\E\left[\sup_{\pi\in\Pi}\left|\sum_{t=1}^H\langle \widetilde{d}^\pi_t-d^\pi_t,{r}_t\rangle\right|\right]
	\] with probability $1-\delta$.

\end{proof}

Next before bounding $\E\left[\sup_{\pi\in\Pi}\left|\sum_{t=1}^H\langle \widetilde{d}^\pi_t-d^\pi_t,{r}_t\rangle\right|\right]$, we first give a useful lemma.

Let $\gamma\in(0,1)$ to be any threshold parameter. Then we first have the following lemma:
\begin{lemma}\label{lem:thre}
	Recall  by definition $P_h(s_h,|s_{h-1},a_{h-1})=T_h(s_h,|s_{h-1},a_{h-1})$. It holds that with probability $1-\delta$, for all $t,s_t,a_t\in[H],\mathcal{S},\mathcal{A}$: if $P_h(s_h|s_{h-1},a_{h-1})\leq\gamma$, then 
	\[
	\left|\widetilde{T}_h(s_h|s_{h-1},a_{h-1})-{T}_h(s_h|s_{h-1},a_{h-1})\right|\leq \sqrt{\frac{\gamma \log(HSA/\delta)}{2nd_m}}+\frac{2\log(HSA/\delta)}{3nd_m}; 
	\]
	if $P_h(s_h,|s_{h-1},a_{h-1})>\gamma$, then 
	\[
	\left|\frac{\widetilde{T}_h(s_h|s_{h-1},a_{h-1})-{T}_h(s_h|s_{h-1},a_{h-1})}{{T}_h(s_h|s_{h-1},a_{h-1})}\right|\leq \sqrt{\frac{ \log(HSA/\delta)}{2nd_m\gamma}}+\frac{2\log(HSA/\delta)}{3nd_m\gamma}; 
	\]
\end{lemma}

\begin{proof}
	
	First consider the case where $P_h(s_h|s_{h-1},a_{h-1})\leq\gamma$.
	\[
	\widetilde{T}_h(s_h|s_{h-1},a_{h-1})-{T}_h(s_h|s_{h-1},a_{h-1})=\frac{1}{n_{s_{h-1},a_{h-1}}}\sum_{i=1}^{n_{s_{h-1},a_{h-1}}}\left(\mathds{1}[s_h^{(i)}]-{T}_h(s_h|s_{h-1},a_{h-1})\right)\mathds{1}(E_h),
	\]
	since $\Var[\mathds{1}[s_h^{(i)}]|s_{h-1},a_{h-1}]=P_h(s_h|s_{h-1},a_{h-1})(1-P_h(s_h|s_{h-1},a_{h-1}))\leq P_h(s_h|s_{h-1},a_{h-1})\leq \gamma$, therefore by Lemma~\ref{lem:bernstein_ineq}, 
	\[
	\left|\widetilde{T}_h(s_h|s_{h-1},a_{h-1})-{T}_h(s_h|s_{h-1},a_{h-1})\right|\leq \mathds{1}(E_h)\left(\sqrt{\frac{\gamma \log(1/\delta)}{n_{s_{h-1},a_{h-1}}}}+\frac{2\log(1/\delta)}{n_{s_{h-1},a_{h-1}}}\right)\leq \sqrt{\frac{\gamma \log(1/\delta)}{2nd_m}}+\frac{2\log(1/\delta)}{3nd_m}; 
	\]
	
	Second, when $P_h(s_h|s_{h-1},a_{h-1})>\gamma$.
	\[
	\frac{\widetilde{T}_h(s_h|s_{h-1},a_{h-1})-{T}_h(s_h|s_{h-1},a_{h-1})}{{T}_h(s_h|s_{h-1},a_{h-1})}=\frac{1}{n_{s_{h-1},a_{h-1}}}\sum_{i=1}^{n_{s_{h-1},a_{h-1}}}\left(\frac{\mathds{1}[s_h^{(i)}]}{{T}_h(s_h|s_{h-1},a_{h-1})}-1\right)\mathds{1}(E_h),
	\]
	since 
	\[
	\Var\left[\frac{\mathds{1}[s_h^{(i)}]}{{T}_h(s_h|s_{h-1},a_{h-1})}\middle|s_{h-1},a_{h-1}\right]\leq \frac{1}{{T}_h(s_h|s_{h-1},a_{h-1})^2}\Var\left[{\mathds{1}[s_h^{(i)}]}\middle|s_{h-1},a_{h-1}\right]\leq \frac{1}{{T}_h(s_h|s_{h-1},a_{h-1})}\leq \frac{1}{\gamma},
	\]
	and since $\frac{\mathds{1}[s_h^{(i)}]}{{T}_h(s_h|s_{h-1},a_{h-1})}\leq 1/\gamma$, again be Bernstein inequality we have 
	\[
	\left|\frac{\widetilde{T}_h(s_h|s_{h-1},a_{h-1})-{T}_h(s_h|s_{h-1},a_{h-1})}{{T}_h(s_h|s_{h-1},a_{h-1})}\right|\leq \sqrt{\frac{ \log(1/\delta)}{2nd_m\gamma}}+\frac{2\log(1/\delta)}{3nd_m\gamma}; 
	\]
	apply the union bound over $t,s_t,a_t$ we obtain the stated result.
\end{proof}

\paragraph{Bounding $\E\left[\sup_{\pi\in\Pi}\left|\sum_{t=1}^H\langle \widetilde{d}^\pi_t-d^\pi_t,{r}_t\rangle\right|\right]$.}

First note by Theorem~\ref{thm:martingale_decompose}:
 \[
 \E\left[\sup_{\pi\in\Pi}\left|\sum_{t=1}^H\langle \widetilde{d}^\pi_t-d^\pi_t,{r}_t\rangle\right|\right]\leq
\sum_{h=2}^H \E\left[\sup_{\pi\in\Pi}\left|\langle{v^\pi_h(s),((\widetilde{T}_h-T_h)\widetilde{d}^\pi_{h-1}})(s)\rangle\right|\right]+\E\left[\sup_{\pi\in\Pi}\left|\langle{V^\pi_1(s),(\widetilde{d}^\pi_1-d^\pi_1)(s)}\rangle\right|\right],
 \]
 so it suffices to bound each $\E\left[\sup_{\pi\in\Pi}\left|\langle{V^\pi_h(s),((\widetilde{T}_h-T_h)\widetilde{d}^\pi_{h-1}})(s)\rangle\right|\right]$. First of all,
{\small
\begin{align*}
&\E\left[\sup_{\pi\in\Pi}\left|\langle{V^\pi_h(s),((\widetilde{T}_h-T_h)\widetilde{d}^\pi_{h-1}})(s)\rangle\right|\right]\\
=&\E\left[\sup_{\pi\in\Pi}\left|\sum_{s_h,s_{h-1},a_{h-1}}V_h^\pi(s_h)(\widetilde{T}_h-T_h)(s_h|s_{h-1},a_{h-1})\widetilde{d}^\pi_{h-1}(s_{h-1},a_{h-1})\right|\right]\\
\leq&\E\left[\sup_{\pi\in\Pi}\left|\sum_{s_h,s_{h-1},a_{h-1}}V_h^\pi(s_h)(\widetilde{T}_h-T_h)(s_h|s_{h-1},a_{h-1})\widetilde{d}^\pi_{h-1}(s_{h-1},a_{h-1})\right|\cdot\mathds{1}[T_h(s_h|s_{h-1},a_{h-1})>\gamma]\right]\\
+&\E\left[\sup_{\pi\in\Pi}\left|\sum_{s_h,s_{h-1},a_{h-1}}V_h^\pi(s_h)(\widetilde{T}_h-T_h)(s_h|s_{h-1},a_{h-1})\widetilde{d}^\pi_{h-1}(s_{h-1},a_{h-1})\right|\cdot\mathds{1}[T_h(s_h|s_{h-1},a_{h-1})\leq\gamma]\right]\\
=&\underbrace{\E\left[\sup_{\pi\in\Pi}\left|\sum_{s_h,s_{h-1},a_{h-1}}V_h^\pi(s_h)T_h(s_h|s_{h-1},a_{h-1})\widetilde{d}^\pi_{h-1}(s_{h-1},a_{h-1})\frac{\widetilde{T}_h-T_h}{T_h}(s_h|s_{h-1},a_{h-1})\right|\cdot\mathds{1}[T_h>\gamma]\right]}_{(a)}\\
+&\underbrace{\E\left[\sup_{\pi\in\Pi}\left|\sum_{s_h,s_{h-1},a_{h-1}}V_h^\pi(s_h)\widetilde{d}^\pi_{h-1}(s_{h-1},a_{h-1})(\widetilde{T}_h-T_h)(s_h|s_{h-1},a_{h-1})\right|\cdot\mathds{1}[T_h(s_h|s_{h-1},a_{h-1})\leq\gamma]\right]}_{(b)},\\
\end{align*}
}
Apply Lemma~\ref{lem:thre} with $\delta^\prime/2$ where $\delta^\prime=\delta/H$, then
{\small
\begin{align*}
(a)\leq& \sup_{\pi\in\Pi}\left|\sum_{s_h,s_{h-1},a_{h-1}}V_h^\pi(s_h)T_h(s_h|s_{h-1},a_{h-1})\widetilde{d}^\pi_{h-1}(s_{h-1},a_{h-1})\left(\sqrt{\frac{ \log(2HSA/\delta^\prime)}{2nd_m\gamma}}+\frac{2\log(2HSA/\delta^\prime)}{3nd_m\gamma}\right)\right|(1-\frac{\delta^\prime}{2})\\
+&H\delta^\prime/2\\
\leq& \sup_{\pi\in\Pi}\left|\sum_{s_h,s_{h-1},a_{h-1}}V_h^\pi(s_h)T_h(s_h|s_{h-1},a_{h-1})\widetilde{d}^\pi_{h-1}(s_{h-1},a_{h-1})\left(\sqrt{\frac{ \log(2H^2SA/\delta)}{2nd_m\gamma}}+\frac{2\log(2H^2SA/\delta)}{3nd_m\gamma}\right)\right|\\
+&\delta/2\\
\leq& \sup_{\pi\in\Pi}\left|H\left(\sqrt{\frac{2 \log(H^2SA/\delta)}{2nd_m\gamma}}+\frac{2\log(2H^2SA/\delta)}{3nd_m\gamma}\right)\right|+\delta/2=H\left(\sqrt{\frac{ \log(2H^2SA/\delta)}{2nd_m\gamma}}+\frac{2\log(2H^2SA/\delta)}{3nd_m\gamma}\right)+\delta/2,\\
\end{align*}
\begin{align*}
(b)\leq& \sup_{\pi\in\Pi}\left|\sum_{s_h,s_{h-1},a_{h-1}}V_h^\pi(s_h)\widetilde{d}^\pi_{h-1}(s_{h-1},a_{h-1})\left(\sqrt{\frac{\gamma \log(2HSA/\delta)}{2nd_m}}+\frac{2\log(2HSA/\delta)}{3nd_m} \right)\right|(1-\frac{\delta^\prime}{2})+H\frac{\delta^\prime}{2}\\
\leq& \sup_{\pi\in\Pi}\left|HS\left(\sqrt{\frac{ \gamma\log(2H^2SA/\delta)}{2nd_m}}+\frac{2\log(2H^2SA/\delta)}{3nd_m}\right)\right|+\frac{\delta}{2}=HS\left(\sqrt{\frac{\gamma \log(2H^2SA/\delta)}{2nd_m}}+\frac{2\log(2H^2SA/\delta)}{3nd_m}\right)+\frac{\delta}{2},\\
\end{align*}
}

Hence we have for any $\gamma$, 
\begin{align*}
&\E\left[\sup_{\pi\in\Pi}\left|\langle{V^\pi_h(s),((\widetilde{T}_h-T_h)\widetilde{d}^\pi_{h-1}})(s)\rangle\right|\right]\\
\leq&H\left(\sqrt{\frac{ \log(2H^2SA/\delta)}{2nd_m\gamma}}+\frac{2\log(2H^2SA/\delta)}{3nd_m\gamma}\right)+HS\left(\sqrt{\frac{\gamma \log(2H^2SA/\delta)}{2nd_m}}+\frac{2\log(2H^2SA/\delta)}{3nd_m}\right)+\delta\\
\end{align*}
In particular, choose $\gamma=1/S<1$, then above becomes
\[
\E\left[\sup_{\pi\in\Pi}\left|\langle{V^\pi_h(s),((\widetilde{T}_h-T_h)\widetilde{d}^\pi_{h-1}})(s)\rangle\right|\right]\leq \sqrt{\frac{2H^2S \log(2H^2SA/\delta)}{nd_m}}+\frac{4HS\log(2H^2SA/\delta)}{3nd_m}+{\delta}
\]
Critically, above holds for any $\forall 1>\delta>0$. Based on theorem condition $n > c\cdot 1/d_m \log (HSA/\theta)>c\cdot 1/d_m$\footnote{Note the $\theta$ in $\log (HSA/\theta)$ is identical to the failure probability in Theorem~\ref{thm:mcdiarmid} }, choose $\delta=\frac{c}{nd_m}$, then above is further less equal to 
\[
\sqrt{\frac{2H^2S \log(2nH^2SA)}{nd_m}}+\frac{4HS\log(2nH^2SA)}{3nd_m}+\frac{c}{nd_m}\leq \sqrt{\frac{2H^2S \log(2nH^2SA)}{nd_m}}+C\cdot\frac{HS\log(2nH^2SA)}{3nd_m}
\]
where $C$ is a new constant absorbs $1/nd_m$. If we further reducing it to 

Finally, summing over all $H$, and again using new constant $C'$ to absorb higher order term, we obtain
\[
\E\left[\sup_{\pi\in\Pi}\left|\sum_{t=1}^H\langle \widetilde{d}^\pi_t-d^\pi_t,{r}_t\rangle\right|\right]\leq C^\prime\sqrt{\frac{H^4S \log(nHSA)}{nd_m}}
\]
Combing this with Theorem~\ref{thm:mcdiarmid}  and Lemma~\ref{lem:app_1}, we have proved {Theorem~\ref{thm:rad}}.

\begin{remark}
	The key for proving this uniform convergence bound is that applying concentration inequality only to terms that are independent of the policies, i.e. $\widetilde{T}_h(s_h|s_{h-1},a_{h-1})-T_h(s_h|s_{h-1},a_{h-1})$. Therefore when taking supremum over policies, high probability event holds with same probability without decay.
\end{remark}

\section{Proof of uniform convergence in OPE with deterministic policies using martingale concentration inequalities: Theorem~\ref{thm:finer_bound}}\label{app_sec:mar}

A reminder that all results in this section use data $\mathcal{D}$ for OPEMA estimator $\widehat{v}^\pi$.

\subsection{Martingale concentration result on $\sum_{t=1}^H\langle \widetilde{d}^\pi_t-d^\pi_t,{r}_t\rangle$.}

Let $X=\sum_{t=1}^H\langle \widetilde{d}^\pi_t-d^\pi_t,{r}_t\rangle$ and $\mathcal{D}_h:=\{s_t^{(i)},a_t^{(i)}:t=1,...,h\}_{i=1}^n$. Since $\mathcal{D}_h$ forms a filtration, then by law of total expectation we have $X_t=\E[X|\mathcal{D}_t]$ is martingale. Moreover, we have 
\begin{lemma}\label{lem:X_t}
\[
X_t:=\E[X|\mathcal{D}_t]=\sum_{h=2}^t\langle{V^\pi_h,(\widetilde{T}_h-T_h)\widetilde{d}^\pi_{h-1}}\rangle+\langle{V^\pi_1,\widetilde{d}^\pi_1-d^\pi_1}\rangle.
\]
\end{lemma}

\begin{proof}[Proof of Lemma~\ref{lem:X_t}]
	By martingale decomposition Theorem~\ref{thm:martingale_decompose} and note $\widetilde{T}_i,\widetilde{d}^\pi_i$ are measurable \textit{w.r.t.} $\mathcal{D}_t$ for $i=1,...,t$, so we have 
	\[
	\E[X|\mathcal{D}_t]=\sum_{h=t+1}^H\E\left[\langle{V^\pi_h,(\widetilde{T}_h-T_h)\widetilde{d}^\pi_{h-1}}\rangle\middle|\mathcal{D}_t\right]+\sum_{h=2}^t\langle{V^\pi_h,(\widetilde{T}_h-T_h)\widetilde{d}^\pi_{h-1}}\rangle+\langle{V^\pi_1,(\widetilde{d}^\pi_1-d^\pi_1)}\rangle.
	\]
	Note for $h\geq t+1$, $\mathcal{D}_t\subset\mathcal{D}_{h-1}$, so by total law of expectation (tower property) we have
	\begin{align*}
	&\E\left[\langle{V^\pi_h,(\widetilde{T}_h-T_h)\widetilde{d}^\pi_{h-1}}\rangle\middle|\mathcal{D}_t\right]\\
	=&\E\left[\E\left[\langle{V^\pi_h,(\widetilde{T}_h-T_h)\widetilde{d}^\pi_{h-1}}\rangle\middle|\mathcal{D}_{h-1}\right]\middle|\mathcal{D}_t\right]\\
	=&\E\left[\langle{V^\pi_h,\E\left[(\widetilde{T}_h-T_h)\middle|\mathcal{D}_{h-1}\right]\widetilde{d}^\pi_{h-1}}\rangle\middle|\mathcal{D}_t\right]=0
	\end{align*}
	where the last equality uses $\widetilde{T}_h$ is unbiased of $T_h$ given $\mathcal{D}_{h-1}$. This gives the desired result.
\end{proof}
Next we show martingale difference $|X_t-X_{t-1}|$ is bounded with high probability.

\begin{lemma}\label{lem:mart_diff_bound}
	With probability $1-\delta$, 
	\[
	\sup_t |X_t-X_{t-1}|\leq O(\sqrt{\frac{H^2\log(HSA/\delta)}{n\cdot d_m}}).
	\]
\end{lemma}

\begin{proof}
	\[
	|X_t-X_{t-1}|=\langle{V^\pi_t,(\widetilde{T}_t-T_t)\widetilde{d}^\pi_{t-1}}\rangle\leq ||(\widetilde{T}_t-T_t)^TV^\pi_t||_\infty||\widetilde{d}^\pi_{t-1}||_1=||(\widetilde{T}_t-T_t)^TV^\pi_t||_\infty.
	\]
	For any fixed pair $(s_t,a_t)$, we have 
	\begin{align*}
	&((\widetilde{T}_t-T_t)^TV^\pi_t)(s_{t-1},a_{t-1})\\
	=&\mathds{1}(E_{t-1})\cdot ((\widehat{T}_t-T_t)^TV^\pi_t)(s_{t-1},a_{t-1})\\
	=&\mathds{1}(E_{t-1})\cdot\sum_{s_t}V^\pi_t(s_t)(\widehat{T}_t-T_t)(s_t|s_{t-1},a_{t-1})\\
	=&\mathds{1}(E_{t-1})\cdot\left(\sum_{s_t}V^\pi_t(s_t)\widehat{T}_t(s_t|s_{t-1},a_{t-1})-\E[V^\pi_t|s_{t-1},a_{t-1}]\right)\\
	=&\mathds{1}(E_{t-1})\cdot\left(\sum_{s_t}V^\pi_t(s_t)\frac{1}{n_{s_{t-1},a_{t-1}}}\sum_{i=1}^n\mathds{1}(s_t^{(i)}=s_t,s_{t-1}^{(i)}=s_{t-1},a_{t-1}^{(i)}=a_{t-1})-\E[V^\pi_t|s_{t-1},a_{t-1}]\right)\\
	=&\mathds{1}(E_{t-1})\left(\frac{1}{n_{s_{t-1},a_{t-1}}}\sum_{i=1}^nV^\pi_t(s_t^{(i)})\mathds{1}(s_t^{(i)}=s_t,s_{t-1}^{(i)}=s_{t-1},a_{t-1}^{(i)}=a_{t-1})-\E[V^\pi_t|s_{t-1},a_{t-1}]\right)\\
	=&\mathds{1}(E_{t-1})\left(\frac{1}{n_{s_{t-1},a_{t-1}}}\sum_{i:s_{t-1}^{(i)}=s_{t-1},a_{t-1}^{(i)}=a_{t-1}}V^\pi_t(s_t^{(i)})-\E[V^\pi_t|s_{t-1},a_{t-1}]\right),\\
	\end{align*}
	where the fourth line uses the definition of $\widehat{T}_t$ and the fifth line uses the fact $\sum_{s_t}V^\pi_t(s_t)\mathds{1}(s_t^{(i)}=s_t,s_{t-1}^{(i)}=s_{t-1},a_{t-1}^{(i)}=a_{t-1})=V^\pi_t(s_t^{(i)})\mathds{1}(s_t^{(i)}=s_t,s_{t-1}^{(i)}=s_{t-1},a_{t-1}^{(i)}=a_{t-1})$.
	
	Note $||V^\pi_t(\cdot)||_\infty\leq H$ and also conditional on $E_t$, $n_{s_t,a_t}\geq nd^\mu_t(s_t,a_t)/2$, therefore by Hoeffding's inequality and a Union bound we obtain with probability $1-\delta$
	\[
	\sup_t |X_t-X_{t-1}|\leq O(\sqrt{\frac{H^2\log(HSA/\delta)}{n\cdot \min_{t,s_t,a_t} d_t^\mu(s_t,a_t)}})=O(\sqrt{\frac{H^2\log(HSA/\delta)}{n\cdot d_m}}).
	\]
\end{proof}

Next we calculate the conditional variance of $\Var[X_{t+1}|\mathcal{D}_t]$.

\begin{lemma}\label{lem:cond_var}
	We have the following decomposition of conditional variance:
	\[
	\mathrm{Var}[X_{t+1}|\mathcal{D}_{t}]=\sum_{s_t,a_t}\frac{\widetilde{d}^\pi_t(s_t,a_t)^2\cdot\mathds{1}(E_t)}{n_{s_t,a_t}}\cdot\mathrm{Var}[V^\pi_{t+1}(s^{(1)}_{t+1})|s^{(1)}_t=s_t,a^{(1)}_t=a_t]
	\]
\end{lemma}

\begin{proof}
	Indeed, 
	\begin{equation}\label{eq:Var_decomp}
	\begin{aligned}
	\mathrm{Var}[X_{t+1}|\mathcal{D}_t]&=\mathrm{Var}\left[\sum_{s_t,a_t}\sum_{s_{t+1}}V^\pi_{t+1}(s_{t+1})(\widetilde{T}-T)(s_{t+1}|s_t,a_t)\widetilde{d}^\pi_t(s_t,a_t)\middle|\mathcal{D}_t\right]\\
	&=\sum_{s_t,a_t}\mathrm{Var}\left[\sum_{s_{t+1}}V^\pi_{t+1}(s_{t+1})(\widetilde{T}-T)(s_{t+1}|s_t,a_t)\middle|\mathcal{D}_t\right]\widetilde{d}^\pi_t(s_t,a_t)^2\\
	&=\sum_{s_t,a_t}\mathds{1}(E_t)\cdot\mathrm{Var}\left[\sum_{s_{t+1}}V^\pi_{t+1}(s_{t+1})\widehat{T}(s_{t+1}|s_t,a_t)\middle|\mathcal{D}_t\right]\widetilde{d}^\pi_t(s_t,a_t)^2\\
	&=\sum_{s_t,a_t}\mathds{1}(E_t)\cdot\mathrm{Var}\left[\sum_{s_{t+1}}V^\pi_{t+1}(s_{t+1})\frac{1}{n_{s_t,a_t}}\sum_{i=1}^n\mathds{1}(s^{(i)}_{t+1}=s_{t+1},s^{(i)}_t=s_t,a^{(i)}_t=a_t)\middle|\mathcal{D}_t\right]\widetilde{d}^\pi_t(s_t,a_t)^2\\
	&=\sum_{s_t,a_t}\frac{\mathds{1}(E_t)}{n_{s_t,a_t}^2}\cdot\mathrm{Var}\left[\sum_{i:s^{(i)}_t=s_t,a^{(i)}_t=a_t}V^\pi_{t+1}(s_{t+1}^{(i)})\middle|\mathcal{D}_t\right]\widetilde{d}^\pi_t(s_t,a_t)^2\\
	&=\sum_{s_t,a_t}\frac{\widetilde{d}^\pi_t(s_t,a_t)^2\cdot\mathds{1}(E_t)}{n_{s_t,a_t}}\cdot\mathrm{Var}[V^\pi_{t+1}(s^{(1)}_{t+1})|s^{(1)}_t=s_t,a^{(1)}_t=a_t]
	\end{aligned}
	\end{equation}
	where the second equal sign comes from the fact that when conditional on $\mathcal{D}_t$, we can separate $n$ episodes into $SA$ groups and episodes from different groups are independent of each other. The third uses $\mathds{1}(E_t)$ is measurable w.r.t $\mathcal{D}_t$. Similarly, the last equal sign again uses $n_{s_t,a_t}$ episodes are independent given $\mathcal{D}_t$. 
\end{proof}

\begin{lemma}[\cite{yin2020asymptotically} Lemma~3.4]\label{lem:H3_to_H2}
	For any policy $\pi$ and any MDP.
	{\small
		\begin{align*}
		&\mathrm{Var}_\pi\left[\sum_{t=1}^H r^{(1)}_t\right] = \sum_{t=1}^H \Big(\E_\pi\left[ \mathrm{Var}\left[r^{(1)}_t+V^\pi_{t+1}(s_{t+1}^{(1)}) \middle|s^{(1)}_t,a^{(1)}_t\right] \right]\\
		&\quad +  \E_\pi\left[ \mathrm{Var}\left[  \E[r^{(1)}_t+V^\pi_{t+1}(s_{t+1}^{(1)}) | s^{(1)}_t, a^{(1)}_t]  \middle|s^{(1)}_t\right] \right]\Big).
		\end{align*}
	}
\end{lemma}

This Lemma suggests if we can bound $\widetilde{d}^\pi_t $ by $O({d}^\pi_t)$ with high probability, then by Lemma~\ref{lem:cond_var} we have w.h.p
\[
\sum_{t=1}^H \mathrm{Var}[X_{t+1}|\mathcal{D}_t]\leq O(\frac{1}{nd_m}\cdot\sum_{t=1}^H \E[\mathrm{Var}[V^\pi_{t+1}(s^{(1)}_{t+1})|s^{(1)}_t,a^{(1)}_t]])\leq O(\frac{H^2}{nd_m})
\]
Note this gives only $H^2$ dependence for $\sum_{t=1}^H \mathrm{Var}[X_{t+1}|\mathcal{D}_t]$ instead of a naive bound with $H^3$ and helps us to save a $H$ factor.

Next we show how to bound $\widetilde{d}^\pi_t $.

\subsection{Bounding $\widetilde{d}^\pi_t(s_t,a_t)-d^\pi_t(s_t,a_t)$}

Our analysis is based on using martingale structure to derive bound on $\widetilde{d}^\pi_t(s_t,a_t)-d^\pi_t(s_t,a_t)$ for fixed $t,s_t,a_t$ with probability $1-\delta/HSA$, then use a union bound to get a bound for all $\widetilde{d}^\pi_t(s_t,a_t)-d^\pi_t(s_t,a_t)$ with probability $1-\delta$. 

Concretely, in \eqref{difference_density} if we only extract the specific $(s_t,a_t)$, then we have 
\[
\widetilde{d}^\pi_t(s_t,a_t)-d^\pi_t(s_t,a_t)=\sum_{h=2}^t (\Gamma_{h+1:t}\pi_h(\widetilde{T}_h-T_h)\widetilde{d}^\pi_{h-1})(s_t,a_t)+(\Gamma_{1:t}(\widetilde{d}^\pi_1-d^\pi_1))(s_t,a_t),
\]  
here $\widetilde{d}^\pi_t(s_t,a_t)-d^\pi_t(s_t,a_t)$ already forms a martingale with filtration $\mathcal{F}_{t}=\sigma(\mathcal{D}_t)$ and $(\Gamma_{h+1:t}\pi_h(\widetilde{T}_h-T_h)\widetilde{d}^\pi_{h-1})(s_t,a_t)$ is the corresponding martingale difference since $$\mathbb{E}[(\Gamma_{h+1:t}\pi_h(\widetilde{T}_h-T_h)\widetilde{d}^\pi_{h-1})(s_t,a_t)|\mathcal{F}_{h-1}]=(\Gamma_{h+1:t}\pi_h\mathbb{E}[(\widetilde{T}_h-T_h)|\mathcal{F}_{h-1}]\widetilde{d}^\pi_{h-1})(s_t,a_t)=0.$$

Now we fix specific $(s_t,a_t)$. Then denote $(\Gamma_{h+1:t}\pi_h)(s_t,a_t):=\Gamma^\prime_{h:t}\in \R^{1\times S}$, then we have 
\[
|(\Gamma_{h+1:t}\pi_h(\widetilde{T}_h-T_h)\widetilde{d}^\pi_{h-1})(s_t,a_t)|=|\Gamma^\prime_{h:t}(\widetilde{T}_h-T_h)\widetilde{d}^\pi_{h-1}|=|\langle (\widetilde{T}_h-T_h)^T\Gamma^{\prime T}_{h:t} ,\widetilde{d}^\pi_{h-1}\rangle|\leq ||\Gamma^\prime_{h:t} (\widetilde{T}_h-T_h)||_\infty\cdot 1.
\]
Note here $\Gamma^\prime_{h:t} (\widetilde{T}_h-T_h)$ is a row vector with dimension $SA$.

\paragraph{Bounding $||\Gamma^\prime_{h:t} (\widetilde{T}_h-T_h)||_\infty$}.

In fact, for any given $(s_{h-1},a_{h-1})$, we have 
\begin{align*}
&\Gamma^\prime_{h:t}(\widetilde{T}_h-T_h)(s_{h-1},s_{h-1})=\mathds{1}(E_t)\cdot \Gamma^\prime_{h:t}(\widehat{T}_h-T_h)(s_{h-1},a_{h-1})\\
=&\mathds{1}(E_t)\cdot \Gamma^\prime_{h:t}\left(\frac{1}{n_{s_{t-1},a_{t-1}}}\sum_{i:s^{(i)}_{h-1}=s_{h-1},a^{(i)}_{h-1}=a_{h-1}}\textbf{e}_{s^{(i)}_h}-\E[\textbf{e}_{s^{(1)}_h}|s^{(1)}_{h-1}=s_{h-1},a^{(1)}_{h-1}=a_{h-1}]\right)\\
=&\mathds{1}(E_t)\left(\frac{1}{n_{s_{t-1},a_{t-1}}}\sum_{i:s^{(i)}_{h-1}=s_{h-1},a^{(i)}_{h-1}=a_{h-1}}\Gamma^\prime_{h:t}({s^{(i)}_h})-\E[\Gamma^\prime_{h:t}({s^{(1)}_h})|s^{(1)}_{h-1}=s_{h-1},a^{(1)}_{h-1}=a_{h-1}]\right)
\end{align*}
Note by definition $\Gamma^\prime_{h:t}({s^{(i)}_h})\leq 1$, since $(\Gamma_{h+1:t}\pi_h)(s_t,a_t):=\Gamma^\prime_{h:t}\in \R^{1\times S}$ and $\Gamma_{h+1:t},\pi_h$ are just probability transitions. Therefore by Hoeffding's inequality and law of total expectation, we have 

\begin{align*}
&\P\left(|\Gamma^\prime_{h:t}(\widetilde{T}_h-T_h)(s_{h-1},a_{h-1})|>\epsilon \right)=\P\left(|\Gamma^\prime_{h:t}(\widehat{T}_h-T_h)(s_{h-1},a_{h-1})|>\epsilon\middle| E_t \right)\\
&\leq \E\left[\exp({-\frac{2n_{s_{h-1},a_{h-1}}\cdot\epsilon^2}{1}})\middle| E_t\right]\leq \exp({-\frac{nd^\mu_{h-1}(s_{h-1},a_{h-1})\cdot\epsilon^2}{1}})\\
\end{align*}

and apply a union bound to get

\begin{equation}\label{eq:bound_weights_hoeffding}
\begin{aligned}
&P(\sup_h||\Gamma^\prime_{h:t} (\widetilde{T}_h-T_h)||_\infty>\epsilon)\leq H\cdot\sup_h P(||\Gamma^\prime_{h:t} (\widetilde{T}_h-T_h)||_\infty>\epsilon)\\
\leq& HSA\cdot\sup_{h,s_{h-1},a_{h-1}}\P\left(|\Gamma^\prime_{h:t}(\widetilde{T}_h-T_h)(s_{h-1},a_{h-1})|>\epsilon \right)\\
\leq& HSA\cdot\exp({-\frac{n\min d^\mu_{h-1}(s_{h-1},a_{h-1})\cdot\epsilon^2}{1}}):=\frac{\delta}{HSA}
\end{aligned}
\end{equation}

Let the right hand side of \eqref{eq:bound_weights_hoeffding} to be $\delta/HSA$, then we have w.p. $1-\delta/HSA$,
\begin{equation}\label{eqn:small_bound}
\sup_h||\Gamma^\prime_{h:t} (\widetilde{T}_h-T_h)||_\infty\leq O(\sqrt{\frac{1}{n\cdot d_m}\log\frac{H^2S^2A^2}{\delta}}).
\end{equation}

\paragraph{Go back to bounding $\widetilde{d}^\pi_t(s_t,a_t)-d^\pi_t(s_t,a_t)$.}

By Azuma-Hoeffding's inequality (Lemma~\ref{lem:azuma_hoeff}), we have\footnote{To be more precise here we actually use a weaker version of Azuma-Hoeffding's inequality, see Remark~\ref{remark:weaker_ineq}.} 
\[
\P(|\widetilde{d}^\pi_t(s_t,a_t)-d^\pi_t(s_t,a_t)|>\epsilon)\leq \exp(-\frac{\epsilon^2}{\sum_{i=1}^t (\sup_h||\Gamma^\prime_{h:t} (\widetilde{T}_h-T_h)||_\infty)^2}):=\delta/HSA,
\] 
where $\sum_{i=1}^t (\sup_h||\Gamma^\prime_{h:t} (\widetilde{T}_h-T_h)||_\infty)^2$ is the sum of bounded square differences in Azuma-Hoeffding's inequality. Therefore we have w.p. $1-\delta/HSA$,
\begin{equation}\label{eqn:weak_azuma}
|\widetilde{d}^\pi_t(s_t,a_t)-d^\pi_t(s_t,a_t)|\leq
O(\sqrt{t\cdot (\sup_h||\Gamma^\prime_{h:t} (\widetilde{T}_h-T_h)||_\infty)^2\log\frac{HSA}{\delta}}),
\end{equation}
combining \eqref{eqn:small_bound} with above we further have that w.p. $1-2\delta/HSA$, 
\[
|\widetilde{d}^\pi_t(s_t,a_t)-d^\pi_t(s_t,a_t)|\leq
O(\sqrt{\frac{t}{nd_m}\log\frac{H^2S^2A^2}{\delta}\log\frac{HSA}{\delta}})
\]

\textbf{Lastly}, by a union bound and simple scaling (from $2\delta$ to $\delta$) we have w.p. $1-\delta$ 
\[
\sup_t||\widetilde{d}^\pi_t-d^\pi_t||_\infty \leq O(\sqrt{\frac{H}{nd_m}\log\frac{H^2S^2A^2}{\delta}\log\frac{HSA}{\delta}}).
\]
This implies that w.p. $1-\delta$, $\forall t,s_t,a_t$, 
\begin{equation}\label{eqn:bound_d}
\widetilde{d}^\pi_t(s_t,a_t)^2\leq 2d^\pi_t(s_t,a_t)^2+O({\frac{H}{nd_m}\log\frac{H^2S^2A^2}{\delta}\log\frac{HSA}{\delta}}).
\end{equation}

Combining \eqref{eqn:bound_d} with Lemma~\ref{lem:H3_to_H2} and Lemma~\ref{lem:cond_var}, we obtain:

\begin{lemma}\label{lem:var_order} With probability $1-\delta$,
\begin{equation}\label{eqn:19}
\sum_{t=1}^H\mathrm{Var}[X_{t+1}|\mathcal{D}_t]\leq O(\frac{H^2}{nd_m})+O(\frac{H^4SA}{n^2d_m^2}\cdot\log\frac{H^2S^2A^2}{\delta}\log\frac{HSA}{\delta})
\end{equation}
\end{lemma}

\begin{proof}[Proof of Lemma~\ref{lem:var_order}]
By \eqref{eqn:bound_d} and Lemma~\ref{lem:cond_var}, we have $\forall t$, with probability ay least $1-\delta$, 
\begin{align*}
&\Var[X_{t+1}|\mathcal{D}_t]\leq \sum_{s_t,a_t}O(\frac{\widetilde{d}^\pi_t(s_t,a_t)^2\cdot }{nd_m})\cdot\mathrm{Var}[V^\pi_{t+1}(s^{(1)}_{t+1})|s^{(1)}_t=s_t,a^{(1)}_t=a_t]\\
&\leq \sum_{s_t,a_t}O(\frac{1}{nd_m})\left(2d^\pi_t(s_t,a_t)^2+O({\frac{H}{nd_m}\log\frac{H^2S^2A^2}{\delta}\log\frac{HSA}{\delta}})\right)\cdot\mathrm{Var}[V^\pi_{t+1}(s^{(1)}_{t+1})|s^{(1)}_t=s_t,a^{(1)}_t=a_t]\\
&\leq \sum_{s_t,a_t}O(\frac{1}{nd_m})\left(2d^\pi_t(s_t,a_t)+O({\frac{H}{nd_m}\log\frac{H^2S^2A^2}{\delta}\log\frac{HSA}{\delta}})\right)\cdot\mathrm{Var}[V^\pi_{t+1}(s^{(1)}_{t+1})|s^{(1)}_t=s_t,a^{(1)}_t=a_t]\\
&\leq O(\frac{1}{nd_m})\E\left[\Var[V^\pi_{t+1}(s^{(1)}_{t+1})|s^{(1)}_t,a^{(1)}_t]\right]+O(\frac{1}{nd_m}\cdot{\frac{H}{nd_m}\log\frac{H^2S^2A^2}{\delta}\log\frac{HSA}{\delta}}\cdot H^2SA)\\
&=O(\frac{1}{nd_m})\E\left[\Var[V^\pi_{t+1}(s^{(1)}_{t+1})|s^{(1)}_t,a^{(1)}_t]\right]+O(\frac{H^3SA}{n^2d_m^2}\cdot\log\frac{H^2S^2A^2}{\delta}\log\frac{HSA}{\delta})
\end{align*}
then sum over $t$ and apply Lemma~\ref{lem:H3_to_H2} gives the stated result.
\end{proof}
Combining all the results, we are able to prove:

\begin{theorem}\label{thm:finer_difference_bound}
	With probability $1-\delta$, we have 
	\[
	\left|\sum_{t=1}^H\langle \widetilde{d}^\pi_t-d^\pi_t,{r}_t\rangle\right|\leq {O}(\sqrt{\frac{H^2\log(HSA/\delta)}{nd_m}}+\sqrt{\frac{H^4SA\cdot\log(H^2S^2A^2/\delta)\log(HSA/\delta)}{n^2d_m^2}})
	\]
	where ${O}(\cdot)$ absorbs only the absolute constants.
\end{theorem}

\begin{proof}[Proof of Theorem~\ref{thm:finer_difference_bound}]
	Recall $X=\sum_{t=1}^H\langle \widetilde{d}^\pi_t-d^\pi_t,{r}_t\rangle$ and by law of total expectation it is easy to show $E[X]=0$. Next denote $\sigma^2=O(\frac{H^2}{nd_m})+O(\frac{H^4SA}{n^2d_m^2}\cdot\log\frac{H^2S^2A^2}{\delta}\log\frac{HSA}{\delta})$ as in Lemma~\ref{lem:var_order} and also let $M=\sup_t|X_t-X_{t-1}|$. Then by Freedman inequality (Lemma~\ref{lem:freedman}), we have with probability $1-\delta/3$, 
	\[
	|X-\E[X]|\leq \sqrt{{8\sigma^2\cdot\log(3/\delta)}}+\frac{2M}{3}\cdot\log(3/\delta), \quad\text{Or} \quad W\geq \sigma^2.
	\]
	where $W=\sum_{t=1}^H\mathrm{Var}[X_{t+1}|\mathcal{D}_t]$. Next by Lemma~\ref{lem:var_order}, we have $\P(W\geq\sigma^2)\leq 1/3\delta$, this implies with probability $1-2\delta/3$, 
	\[
		|X-\E[X]|\leq \sqrt{{8\sigma^2\cdot\log(3/\delta)}}+\frac{2M}{3}\cdot\log(3/\delta).
	\]
	Finally, by Lemma~\ref{lem:mart_diff_bound}, we have $\P(M\geq O(\sqrt{\frac{H^2\log(HSA/\delta)}{nd_m}}))\leq \delta/3$. Also use $\E[X]=0$, we have with probability $1-\delta$, 
	\[
	|X|\leq \sqrt{{8\sigma^2\cdot\log(3/\delta)}}+{O}(\sqrt{\frac{H^2\cdot \log(HSA/\delta)}{nd_m}}\log(3/\delta)).
	\]
	
	Plugging back the expression of $\sigma^2=O(\frac{H^2}{nd_m})+O(\frac{H^4SA}{n^2d_m^2}\cdot\log\frac{H^2S^2A^2}{\delta}\log\frac{HSA}{\delta})$ and assimilating the same order terms give the desired result.
\end{proof}

\begin{remark}\label{remark:weaker_ineq}
	Rigorously, standard Azuma-Hoeffding's inequality Lemma~\ref{lem:azuma_hoeff} does not apply to \eqref{eqn:weak_azuma} since $\sup_h||\Gamma^\prime_{h:t} (\widetilde{T}_h-T_h)||_\infty$ is not a deterministic upper bound, we only have the difference bound with high probability sense, see \eqref{eqn:small_bound}. Therefore, strictly speaking, we need to apply Theorem~32 in \cite{chung2006concentration} which is a weaker Azuma-Hoeffding's inequality allowing bounded difference with high probability. The same logic applies for a weaker freedman's inequality consisting of Theorem~34 and Theorem~37 in \cite{chung2006concentration} since our martingale difference $M=\sup_t|X_t-X_{t-1}|$ in the proof of Theorem~\ref{thm:finer_difference_bound} is bounded with high probability. We avoid explicitly using them in order to make our proofs more readable for our readers.
\end{remark}

We end this section by giving the proofs of Theorem~\ref{thm:single_finer_bound} and Theorem~\ref{thm:finer_bound}.

\begin{proof}[Proof of Lemma~\ref{thm:single_finer_bound} and Theorem~\ref{thm:finer_bound}]
	The proof of Lemma~\ref{thm:single_finer_bound} comes from Lemma~\ref{lem:zero_diff}, Lemma~\ref{lem:app_1} and Theorem~\ref{thm:finer_difference_bound}. The proof of Theorem~\ref{thm:finer_bound} relies on applying a union bound over $\Pi$ in Theorem~\ref{thm:single_finer_bound} (recall all non-stationary deterministic policies have $|\Pi|=A^{HS}$), then extra dependence of $\sqrt{\log(|\Pi|)}=\sqrt{HS\log(A)}$ pops out. Note that the higher order term has two trailing $\log$ terms (see the right hand side of \eqref{eqn:19}), so when replacing $\delta$ by $\delta/|\Pi|$ with a union bound, both terms will give extra $\sqrt{HS}$ dependence so in higher order term we have extra $HS$ dependence but not just $\sqrt{HS}$. 
	
\end{proof}

\section{Proof of uniform convergence problem with local policy class.}\label{sec:E}

In this section, we consider using OPEMA estimator with data $\mathcal{D}'$. Also, WLOG we only consider deterministic reward (as implied by Lemma~\ref{lem:app_1} random reward only causes lower order dependence). Also, we {fix $N>0$} for the moment. First recall for all $t=1,...,H$
\begin{align*}
V^\pi_t(s_t)&=\E_\pi\left[\sum_{t'=t}^H r_{t'}(s_{t'}^{(1)},a_t^{(1)})\middle|s_t^{(1)}=s_t\right] \\
Q^\pi_t(s_t,a_t)&=\E_\pi\left[\sum_{t'=t}^H r_{t'}(s_{t'}^{(1)},a_t^{(1)})\middle|s_t^{(1)}=s_t,a_t^{(1)}=a_t\right] 
\end{align*}
where $r_t(s,a)$ are deterministic rewards and $s_t^{(1)},a_t^{(1)}$ are random variables. Consider $V^\pi_t,Q^\pi_t$ as vectors, then by standard Bellman equations we have for all $t=1,...,H$ (define $V_{H+1}=Q_{H+1}=0$)
\begin{equation}\label{eqn:bellman}
Q^\pi_t=r_t+P_{t+1}^\pi Q^\pi_{t+1}=r_t+P_{t+1} V^\pi_{t+1},
\end{equation}
where $P_{t}^\pi\in\R^{(SA)\times(SA)}$ is the state-action transition and $P_t(\cdot|\cdot,\cdot)\in\R^{(SA)\times S}$ is the transition probabilities defined in Section~\ref{sec:formulation}. Also, we have bellman optimality equations:
\begin{equation}\label{eqn:bell_opt}
Q^\star_t=r_t+P_{t+1} V^\star_{t+1}, \qquad V^\star_{t}(s_{t}):=\max_{a_t}Q^\star_t(s_t,a_t),\;\;\pi^\star_t(s_t):=\argmax_{a_t}Q^\star_t(s_t,a_t)\;\;\forall s_t 
\end{equation}
where $\pi^\star$ is one optimal deterministic policy. The corresponding Bellman equations and Bellman optimality equations for empirical MDP $\widehat{M}$ are defined similarly. Since we consider deterministic rewards, by Bellman equations we have 
\[
\widehat{Q}^\pi_t-Q^\pi_t=\widehat{P}_{t+1}^\pi \widehat{Q}^\pi_{t+1}-P_{t+1}^\pi Q^\pi_{t+1}=(\widehat{P}_{t+1}^\pi-{P}_{t+1}^\pi) \widehat{Q}^\pi_{t+1}+P_{t+1}^\pi (\widehat{Q}^\pi_{t+1}-Q^\pi_{t+1})
\]
for $t=1,...,H$. By writing it recursively, we have $\forall t=1,...,H-1$
\begin{align*}
\widehat{Q}^\pi_t-Q^\pi_t&=\sum_{h=t+1}^H \Gamma_{t+1:h-1}^\pi(\widehat{P}_{h}^\pi-{P}_{h}^\pi) \widehat{Q}^\pi_{h}\\
&=\sum_{h=t+1}^H \Gamma_{t+1:h-1}^\pi(\widehat{P}_{h}-{P}_{h}) \widehat{V}^\pi_{h}
\end{align*}
where $\Gamma_{t:h}^\pi=\prod_{i=t}^hP^\pi_i$ is the multi-step state-action transition and $\Gamma_{t+1:t}^\pi:=I$.

Note $\widehat{\pi}^*$ to be the empirical optimal policy over $\widehat{M}$, we are interested in how to obtain uniform convergence for any policy $\pi$ that is close to $\widehat{\pi}^*$. More precisely, in this section we consider the policy class $\Pi_1$ to be:
\[
\Pi_1:=\lbrace \pi:s.t.\; ||\widehat{V}^\pi_t-\widehat{V}^{\widehat{\pi}^\star}_t ||_\infty\leq\epsilon_\mathrm{opt},\;\forall t=1,...,H\rbrace
\]
where $\epsilon_\mathrm{opt}\geq0$ is a parameter decides how large the policy class is. We now assume $\widehat{\pi}$ to be any policy within $\Pi_1$ throughout this section. \textbf{Also, $\widehat{\pi}$ may be a policy learned from a learning algorithm using the data $\mathcal{D}$. In this case, $\widehat{\pi}$ may not be independent of $\widehat{P}$.}

We start with the following simple calculation:\footnote{Since all quantities in the calculation are vectors, so the absolute value $|\cdot|$ used is point-wise operator.}
\begin{equation}
\begin{aligned}
\left|\widehat{Q}^{\widehat{\pi}}_t-Q^{\widehat{\pi}}_t\right|&\leq \sum_{h=t+1}^H \Gamma_{t+1:h-1}^\pi\left|(\widehat{P}_{h}-{P}_{h}) \widehat{V}^{\widehat{\pi}}_{h}\right|\\
&\leq \underbrace{\sum_{h=t+1}^H \Gamma_{t+1:h-1}^\pi\left|(\widehat{P}_{h}-{P}_{h}) \widehat{V}^{\widehat{\pi}^\star}_{h}\right|}_{(***)}+\underbrace{\sum_{h=t+1}^H \Gamma_{t+1:h-1}^\pi\left|(\widehat{P}_{h}-{P}_{h}) (\widehat{V}^{\widehat{\pi}^\star}_{h}-\widehat{V}^{{\widehat{\pi}}}_{h})\right|}_{(****)}
\end{aligned}
\end{equation}

We now analyze $(***)$ and $(****)$.

\subsection{Analyzing $\sum_{h=t+1}^H \Gamma_{t+1:h-1}^\pi\left|(\widehat{P}_{h}-{P}_{h}) (\widehat{V}^{\widehat{\pi}^\star}_{h}-\widehat{V}^{{\widehat{\pi}}}_{h})\right|$}

First, by vector induced matrix norm\footnote{For $A$ a matrix and $x$ a vector we have $\norm{Ax}_\infty\leq\norm{A}_\infty\norm{x}_\infty$.} we have 
\begin{align*}
\norm{\sum_{h=t+1}^H \Gamma_{t+1:h-1}^{\widehat{\pi}}\cdot\left|(\widehat{P}_{h}-{P}_{h}) (\widehat{V}^{\widehat{\pi}^\star}_{h}-\widehat{V}^{\widehat{\pi}}_{h})\right|}_\infty&\leq H\cdot \sup_h\norm{\Gamma_{t+1:h-1}^{\widehat{\pi}}}_\infty\norm{|(\widehat{P}_{h}-{P}_{h}) (\widehat{V}^{\widehat{\pi}^\star}_{h}-\widehat{V}^{{\widehat{\pi}}}_{h})|}_\infty\\
&\leq H\cdot \sup_h\norm{|(\widehat{P}_{h}-{P}_{h}) (\widehat{V}^{\widehat{\pi}^\star}_{h}-\widehat{V}^{{\widehat{\pi}}}_{h})|}_\infty
\end{align*}
where the last equal sign uses multi-step transition $\Gamma_{t+1:h-1}^\pi$ is row-stochastic. Note given $N$, $\widehat{P}_t(\cdot|\cdot,\cdot)$ all have $N$ in the denominator. Therefore, by Hoeffding inequality and a union bound we have with probability $1-\delta$, 

\[
\sup_{t,s_t,s_{t-1},a_{t-1}}|\widehat{P}_t(s_t|s_{t-1},a_{t-1})-{P}_t(s_t|s_{t-1},a_{t-1})|\leq O(\sqrt{\frac{\log(HSA/\delta)}{N}}),
\]
this indicates
\[
\sup_h\norm{|(\widehat{P}_{h}-{P}_{h}) (\widehat{V}^{\widehat{\pi}^\star}_{h}-\widehat{V}^{{\widehat{\pi}}}_{h})|}_\infty\leq \epsilon_\mathrm{opt}\cdot\sup_h\norm{|\widehat{P}_{h}-{P}_{h}|\cdot \mathbf{1}}_\infty\leq \epsilon_\mathrm{opt}\cdot O(S\sqrt{\frac{\log(HSA/\delta)}{N}}),
\]
where $\mathbf{1}\in\R^{S}$ is all-one vector. To sum up, we have
\begin{lemma}\label{lem:epi_opt_bound}
	Fix $N>0$, we have with probability $1-\delta$, for all $t=1,...,H-1$
	\[
	\sum_{h=t+1}^H \Gamma_{t+1:h-1}^{\hpi}\left|(\widehat{P}_{h}-{P}_{h}) (\widehat{V}^{\widehat{\pi}^\star}_{h}-\widehat{V}^{{\widehat{\pi}}}_{h})\right|\leq\epsilon_\mathrm{opt}\cdot O\left(\sqrt{\frac{H^2S^2\log(HSA/\delta)}{N}}\cdot\mathbf{1}\right)
	\]
\end{lemma}

Now we consider $(***)$.

\subsection{Analyzing $\sum_{h=t+1}^H \Gamma_{t+1:h-1}^{\hpi}\left|(\widehat{P}_{h}-{P}_{h}) \widehat{V}^{\widehat{\pi}^\star}_{h}\right|$.}

\begin{lemma}\label{lem:bern_bound}
	Given $N$, we have with probability $1-\delta$, $\forall t= 1,...,H-1$
	\[
	\sum_{h=t+1}^H \Gamma_{t+1:h-1}^{\hpi}\left|(\widehat{P}_{h}-{P}_{h}) \widehat{V}^{\widehat{\pi}^\star}_{h}\right|\leq \sum_{h=t+1}^H \Gamma_{t+1:h-1}^{\hpi}\left(4\sqrt{\frac{\log(HSA/\delta)}{N}}\sqrt{\Var(\widehat{V}^{\widehat{\pi}^\star}_h)}+\frac{4(H-t)}{3N}\log(\frac{HSA}{\delta})\cdot\mathbf{1}\right)
	\]
	where $\Var({v}_t^\pi)\in\R^{SA}$ and $\Var({V}_t^\pi)(s_{t-1},a_{t-1})=\Var_{s_t}[{V}_t^\pi(\cdot)|s_{t-1},a_{t-1}]$ and $|\cdot|,\sqrt{\cdot}$ are point-wise operator.
\end{lemma}

\begin{proof}[Proof of Lemma~\ref{lem:bern_bound}]
	The key point is to guarantee $\widehat{P}_h$ is independent of $\widehat{V}^{\widehat{\pi}^\star}_{h}$ so that we can apply Bernstein inequality w.r.t the randomness in $\widehat{P}_h$. In fact, note given $N$ all data pairs in $\mathcal{D}'$ are independent of each other,  and $\widehat{P}_h$ only uses data from $h-1$ to $h$. Moreover, $\widehat{V}^{\widehat{\pi}^\star}_{h}$ only uses data from time $h$ to $H$ since $\widehat{V}^\pi_h$ uses data from $h$ to $H$ by bellman equation \eqref{eqn:bellman} for any $\pi$ and optimal policy $\hpi^\star_{h:H}$ also only uses data from $h$ to $H$ by bellman optimality equation \eqref{eqn:bell_opt}.
	
	Then by Bernstein inequality (Lemma~\ref{lem:bernstein_ineq}), with probability $1-\delta$
	\[
	\left|(\widehat{P}_{h}-{P}_{h}) \widehat{V}^{\widehat{\pi}^\star}_{h}\right|(s_{t-1},a_{t-1})\leq 4\sqrt{\frac{\log(1/\delta)}{N}}\sqrt{\Var(\widehat{V}^{\widehat{\pi}^\star}_h)}(s_{t-1},a_{t-1})+\frac{4(H-t)}{3N}\log(\frac{1}{\delta})
	\]
	apply a union bound and take the sum we get the stated result.
\end{proof}

Now combine Lemma~\ref{lem:epi_opt_bound} and Lemma~\ref{lem:bern_bound} we obtain with probability $1-\delta$, for all $t=1,...,H-1$
\begin{equation}\label{eqn:diff_q}
\begin{aligned}
\left|\widehat{Q}^{\widehat{\pi}}_t-Q^{\widehat{\pi}}_t\right|&\leq \sum_{h=t+1}^H \Gamma_{t+1:h-1}^{\hpi}\left(4\sqrt{\frac{\log(HSA/\delta)}{N}}\sqrt{\Var(\widehat{V}^{\widehat{\pi}^\star}_h)}+\frac{4(H-t)}{3N}\log(\frac{HSA}{\delta})\cdot\mathbf{1}\right)\\
&+c_1\epsilon_{\text{opt}}\cdot\sqrt{\frac{H^2S^2\log(HSA/\delta)}{N}}\cdot\mathbf{1}\\
&\leq 4\sqrt{\frac{\log(HSA/\delta)}{N}}\sum_{h=t+1}^H\Gamma_{t+1:h-1}^{\hpi}\sqrt{\Var(\widehat{V}^{\widehat{\pi}^\star}_h)}+\frac{4H^2}{3N}\log(\frac{HSA}{\delta})\cdot\mathbf{1}\\
&+c_1\epsilon_{\text{opt}}\cdot\sqrt{\frac{H^2S^2\log(HSA/\delta)}{N}}\cdot\mathbf{1},\\
\end{aligned}
\end{equation}

Next note $\sqrt{\Var(\cdot)}$ is a norm, therefore by norm triangle inequality we have
\begin{align*}
\sqrt{\Var(\widehat{V}^{\widehat{\pi}^\star}_h)}&\leq \sqrt{\Var(\widehat{V}^{\widehat{\pi}^\star}_h-\widehat{V}^{\hpi}_h)} + \sqrt{\Var(\widehat{V}^{\hpi}_h-{V}^{\hpi}_h)}+\sqrt{\Var({V}^{\hpi}_h)}\\
&\leq \norm{\widehat{V}^{\widehat{\pi}^\star}_h-\widehat{V}^{\hpi}_h}_\infty\cdot\mathbf{1}+\norm{\widehat{V}^{\hpi}_h-{V}^{\hpi}_h}_\infty\cdot\mathbf{1}+\sqrt{\Var({V}^{\hpi}_h)}\\
&\leq \epsilon_{\text{opt}}\cdot\mathbf{1}+\norm{\widehat{Q}^{\hpi}_h-{Q}^{\hpi}_h}_\infty\cdot\mathbf{1}+\sqrt{\Var({V}^{\hpi}_h)}\\
\end{align*}

Plug this into \eqref{eqn:diff_q} to obtain
\begin{equation}\label{eqn:diff_q_1}
\begin{aligned}
	\left|\widehat{Q}^{\widehat{\pi}}_t-Q^{\widehat{\pi}}_t\right|
	&\leq 4\sqrt{\frac{\log(HSA/\delta)}{N}}\sum_{h=t+1}^H\left(\Gamma_{t+1:h-1}^{\hpi}\sqrt{\Var({V}^{\widehat{\pi}}_h)}+\norm{\widehat{Q}^{\hpi}_h-{Q}^{\hpi}_h}_\infty\cdot\mathbf{1}\right)+\frac{4H^2}{3N}\log(\frac{HSA}{\delta})\cdot\mathbf{1}\\
	&+c_2\epsilon_{\text{opt}}\cdot\sqrt{\frac{H^2S^2\log(HSA/\delta)}{N}}\cdot\mathbf{1}.\\
\end{aligned}
\end{equation}

Next lemma helps us to bound $\sum_{h=t+1}^H\Gamma_{t+1:h-1}^{\hpi}\sqrt{\Var({V}^{\widehat{\pi}}_h)}$.

\begin{lemma}\label{lem:H3}
A conditional version of Lemma~\ref{lem:H3_to_H2} holds:
			\begin{equation}\label{eqn:11}
			\begin{aligned}
			&\mathrm{Var}_\pi\left[\sum_{t=h}^H r^{(1)}_t\middle|s^{(1)}_h=s_h,a^{(1)}_h=a_h\right] = \sum_{t=h}^H \Big(\E_\pi\left[ \mathrm{Var}\left[r^{(1)}_t+V^\pi_{t+1}(s_{t+1}^{(1)}) \middle|s^{(1)}_t,a^{(1)}_t\right] \middle|s^{(1)}_h=s_h,a^{(1)}_h=a_h\right]\\
			&\quad +  \E_\pi\left[ \mathrm{Var}\left[  \E[r^{(1)}_t+V^\pi_{t+1}(s_{t+1}^{(1)}) | s^{(1)}_t, a^{(1)}_t]  \middle|s^{(1)}_t\right]\middle|s^{(1)}_h=s_h,a^{(1)}_h=a_h \right]\Big).
			\end{aligned}
			\end{equation}
			and by using \eqref{eqn:11} we can show
			\[
			\sum_{h=t+1}^H\Gamma_{t+1:h-1}^{\hpi}\sqrt{\Var({V}^{\widehat{\pi}}_h)}\leq \sqrt{(H-t)^3}\cdot\mathbf{1}.
			\]
\end{lemma}

\begin{proof}
	The proof of \eqref{eqn:11} uses the identical trick as Lemma~\ref{lem:H3_to_H2} except the total law of variance is replaced by the total law of conditional variance.
	
	Moreover, recall $\Gamma_{t+1:h-1}^{\hpi}=\prod_{i=t+1}^{h-1}P^{\hpi}_i$ is the multi-step transition, so for any pair $(s_t,a_t)$,
	
	\begin{align*}
	&\sum_{h=t+1}^H\left(\Gamma_{t+1:h-1}^{\hpi}\sqrt{\Var({V}^{\widehat{\pi}}_h)}\right)(s_t,a_t)\\
	=&\sum_{h=t+1}^H\sum_{s_{h-1},a_{h-1}}\sqrt{\Var[{V}^{\widehat{\pi}}_h|s_{h-1},a_{h-1}]}d^{\hpi}_t(s_{h-1},a_{h-1}|s_t,a_t)\\
	=&\sum_{h=t+1}^H\sum_{s_{h-1},a_{h-1}}\sqrt{\Var[{V}^{\widehat{\pi}}_h|s_{h-1},a_{h-1}]d^{\hpi}_t(s_{h-1},a_{h-1}|s_t,a_t)}\cdot\sqrt{d^{\hpi}_t(s_{h-1},a_{h-1}|s_t,a_t)}\\
	\leq&\sum_{h=t+1}^H\sqrt{\sum_{s_{h-1},a_{h-1}}\Var[{V}^{\widehat{\pi}}_h|s_{h-1},a_{h-1}]d^{\hpi}_t(s_{h-1},a_{h-1}|s_t,a_t)\cdot\sum_{s_{h-1},a_{h-1}}d^{\hpi}_t(s_{h-1},a_{h-1}|s_t,a_t)}\\
	=&\sum_{h=t+1}^H\sqrt{\sum_{s_{h-1},a_{h-1}}\Var[{V}^{\widehat{\pi}}_h|s_{h-1},a_{h-1}]d^{\hpi}_t(s_{h-1},a_{h-1}|s_t,a_t)}\\
	=&\sum_{h=t+1}^H\sqrt{\E_{\widehat{\pi}}\bigg[\Var[{V}^{\widehat{\pi}}_h|s_{h-1}^{(1)},a_{h-1}^{(1)}]\bigg|s_t,a_t\bigg]}\\
	=&\sum_{h=t+1}^H\sqrt{1}\cdot\sqrt{\E_{\widehat{\pi}}\bigg[\Var[{V}^{\widehat{\pi}}_h|s_{h-1}^{(1)},a_{h-1}^{(1)}]\bigg|s_t,a_t\bigg]}\\
	\leq&\sqrt{(H-t)\sum_{h=t+1}^H\E_{\widehat{\pi}}\bigg[\Var[{V}^{\widehat{\pi}}_h|s_{h-1}^{(1)},a_{h-1}^{(1)}]\bigg|s_t,a_t\bigg]}\\
	\leq&\sqrt{(H-t)\cdot \mathrm{Var}_{\widehat{\pi}}\left[\sum_{h=t+1}^H r^{(1)}_h\middle|s^{(1)}_t=s_t,a^{(1)}_t=a_t\right] }\leq \sqrt{(H-t)^3}
	\end{align*}
	where all the inequalities are Cauchy-Schwarz inequalities.
\end{proof}

Apply Lemma~\ref{lem:H3} to bound \eqref{eqn:diff_q_1}, and use $\infty$ norm on both sides, we obtain

\begin{theorem}\label{thm:recursive_diff_q}
	Conditional on $N>0$, then with probability $1-\delta$, we have for all $t=1,...,H-1$
	\begin{align*}
	\norm{\widehat{Q}^{\widehat{\pi}}_t-Q^{\widehat{\pi}}_t}_\infty
	&\leq 4\sqrt{\frac{H^3\log(HSA/\delta)}{N}}+4\sqrt{\frac{\log(HSA/\delta)}{N}}\sum_{h=t+1}^H\norm{\widehat{Q}^{\hpi}_h-{Q}^{\hpi}_h}_\infty+\frac{4H^2}{3N}\log(\frac{HSA}{\delta})\\
	&+c_2\epsilon_{\text{opt}}\cdot\sqrt{\frac{H^2S^2\log(HSA/\delta)}{N}}.\\
	\end{align*}
\end{theorem}

 Then by using backward induction and Theorem~\ref{thm:recursive_diff_q}, we have  the following:
 
 \begin{theorem}\label{thm:local_uni}
 	
Suppose $N\geq 64H^2\cdot\log(HSA/\delta)$ and $\epsilon_{\text{opt}}\leq \sqrt{H}/S$, then we have with probability $1-\delta$, 
\[
\norm{\widehat{Q}^{\widehat{\pi}}_1-Q^{\widehat{\pi}}_1}_\infty
\leq 2(9+c_2)\sqrt{\frac{H^3\log(HSA/\delta)}{N}}
\]
where $c_2$ is the same constant in Theorem~\ref{thm:recursive_diff_q}.
 \end{theorem}

\begin{proof}
	Under the condition, by Theorem~\ref{thm:recursive_diff_q} it is easy to check for all $t=1,...,H-1$ with probability $1-\delta$,
		\begin{align*}
	\norm{\widehat{Q}^{\widehat{\pi}}_t-Q^{\widehat{\pi}}_t}_\infty
	&\leq (5+c_2)\sqrt{\frac{H^3\log(HSA/\delta)}{N}}+4\sqrt{\frac{\log(HSA/\delta)}{N}}\sum_{h=t+1}^H\norm{\widehat{Q}^{\hpi}_h-{Q}^{\hpi}_h}_\infty,
	\end{align*}
	which we conditional on.
	
	For $t=H-1$, we have 
		\begin{align*}
	\norm{\widehat{Q}^{\widehat{\pi}}_{H-1}-Q^{\widehat{\pi}}_{H-1}}_\infty
	\leq& (5+c_2)\sqrt{\frac{H^3\log(HSA/\delta)}{N}}+4\sqrt{\frac{\log(HSA/\delta)}{N}}\norm{\widehat{Q}^{\hpi}_H-{Q}^{\hpi}_H}_\infty\\
	\leq&(5+c_2)\sqrt{\frac{H^3\log(HSA/\delta)}{N}}+4\sqrt{\frac{H^2\log(HSA/\delta)}{N}}\\
	\leq&(9+c_2)\sqrt{\frac{H^3\log(HSA/\delta)}{N}}
	\end{align*}
	Suppose $\norm{\widehat{Q}^{\widehat{\pi}}_h-Q^{\widehat{\pi}}_h}_\infty
	\leq 2(9+c_2)\sqrt{\frac{H^3\log(HSA/\delta)}{N}}$ holds for all $h=t+1,...,H$, then for $h=t$, we have 
	\begin{align*}
	\norm{\widehat{Q}^{\widehat{\pi}}_t-Q^{\widehat{\pi}}_t}_\infty
	&\leq (5+c_2)\sqrt{\frac{H^3\log(HSA/\delta)}{N}}+4\sqrt{\frac{\log(HSA/\delta)}{N}}\sum_{h=t+1}^H\norm{\widehat{Q}^{\hpi}_h-{Q}^{\hpi}_h}_\infty\\
	&\leq (9+c_2)\sqrt{\frac{H^3\log(HSA/\delta)}{N}}+4\sqrt{\frac{(H-1)^2\log(HSA/\delta)}{N}}\cdot 2(9+c_2)\sqrt{\frac{H^3\log(HSA/\delta)}{N}}\\
	&\leq 2 (9+c_2)\sqrt{\frac{H^3\log(HSA/\delta)}{N}}
	\end{align*}
	where the last line uses the condition $N\geq 64H^2\cdot\log(HSA/\delta)$. By induction, we have the result.
\end{proof}

\begin{proof}[Proof of Theorem~\ref{thm:local_uni_opt}]
By Theorem~\ref{thm:local_uni} we have for $N\geq c\cdot H^2\cdot\log(HSA/\delta)$, 
\[
\P\left(\norm{\widehat{Q}^{\widehat{\pi}}_1-Q^{\widehat{\pi}}_1}_\infty
\geq 2(9+c_2)\sqrt{\frac{H^3\log(HSA/\delta)}{N}}\middle|N\right)\leq \delta
\]	
The only thing left is to use Lemma~\ref{lem:sufficient_sample} to bound the event that $\{N<nd_m/2\}$ has small probability.

Last but not least, the condition $n>c_1H^2\log(HSA/\delta)/d_m$ is sufficient for applying Lemma~\ref{lem:sufficient_sample} and it also implies $N\geq c\cdot H^2\cdot\log(HSA/\delta)$ (the condition of Theorem~\ref{thm:local_uni}) when $N\geq nd_m/2$ since:
\[
n>c_1H^2\log(HSA/\delta)/d_m \Rightarrow nd_m/2\geq c_2 H^2\log(HSA/\delta)
\]
which implies $N\geq c_2\cdot H^2\cdot\log(HSA/\delta)$ when $N\geq nd_m/2$.
\end{proof}

\section{Proof of uniform convergence lower bound.}\label{sec:lower}

In this section we prove a uniform convergence OPE lower bound of $\Omega(H^3/d_m\epsilon^2)$. Conceptually, uniform convergence lower bound can be derived by a reduction to the lower bound of identifying the $\epsilon$-optimal policy. There are quite a few literature that provide information theoretical lower bounds in different setting, \textit{e.g.} \cite{dann2015sample,jiang2017contextual,krishnamurthy2016pac,jin2018q,sidford2018near}. However, to the best of our knowledge, there is no result proven for the non-stationary transition finite horizon episodic setting with bounded rewards. For example, \cite{sidford2018near} prove the result sample complexity lower bound of $\Omega(H^3SA/\epsilon^2)$ with stationary MDP and their proof cannot be directly applied to non-stationary setting as they reduce the problem to infinite horizon discounted setting which always has stationary transitions. \cite{dann2015sample} prove the episode complexity of $\widetilde{\Omega}(H^2SA/\epsilon^2)$ for the stationary transition setting. \cite{jin2018q} prove the $\Omega(\sqrt{H^2SAT})$ regret lower bound for non-stationary finite horizon online setting but it is not clear how to translate the regret to PAC-learning setting by keeping the same sample complexity optimality. \cite{jiang2017contextual} prove the $\Omega(HSA/\epsilon^2)$ lower bound for the non-stationary finite horizon offline episodic setting where they assume $\sum_{i=1}^Hr_i\leq 1$ and this is also different from our setting since we have $0\leq r_t\leq 1$ for each time step.

Our proof consists of three steps. \textbf{1.} We will first show a minimax lower bound (\textbf{over all MDP instances}) for learning $\epsilon$-optimal policy is $\Omega(H^3SA/\epsilon^2)$; \textbf{2.} Based on 1, we can further show a minimax lower bound (\textbf{over problem class $\mathcal{M}_{d_m}$}) for learning $\epsilon$-optimal policy is $\Omega(H^3/d_m\epsilon^2)$; \textbf{3.} prove the uniform convergence OPE lower bound of the same rate.

\subsection{Information theoretical lower sample complexity bound over all MDP instances for identifying $\epsilon$-optimal policy.}  

In fact, a modified construction of Theorem~5 in \cite{jiang2017contextual} is our tool for obtaining $\Omega(H^3SA/\epsilon^2)$ lower bound. We can get the additional $H^2$ factor by using $\sum_{i=1}^Hr_i$ can be of order $O(H)$. 

 \begin{theorem}\label{thm:learn_low_bound_all}
 	Given $H\geq 2$, $A\geq2$, $0<\epsilon<\frac{1}{48\sqrt{8}}$ and $S\geq c_1$ where $c_1$ is a universal constant. Then there exists another universal constant $c$ such that for any algorithm and any $n\leq cH^3SA/\epsilon^2$, there exists a non-stationary $H$ horizon MDP with probability at least $1/12$, the algorithm outputs a policy $\hpi$ 
	with $v^\star- v^{\hpi}\geq\epsilon $. 
\end{theorem}

Like in \cite{jiang2017contextual}, the proof relies on embedding $\Theta(HS)$ independent multi-arm bandit problems into a hard-to-learn MDP so that any algorithm that wants to output a near-optimal policy needs to identify the best action in $\Omega(HS)$ problems. However, in our construction we make a further modification of \cite{jiang2017contextual} so that there is \textbf{no} waiting states, which is crucial for the reduction from offline family. We also double the length of the hard-to-learn MDP instance so that the latter half uses a ``naive'' copy construction which is uninformative. The uninformative extension will help to produce the additional $H^2$ factor. 

\begin{proof}[Proof of Theorem~\ref{thm:learn_low_bound}]
	We construct a non-stationary MDP with $S$ states per level, $A$ actions per state and has horizon $2H$. At each time step, states are categorized into four types with two special states $g_h$, $b_h$ and the remaining $S-2$ ``bandit'' states denoted by $s_{h,i}$, $i\in[S-2]$. Each bandit state has an unknown best action $a^\star_{h,i}$ that provides the highest expected reward comparing to other actions.
	
	\begin{figure}[H]
		\centering     
		\includegraphics[width=0.6\linewidth]{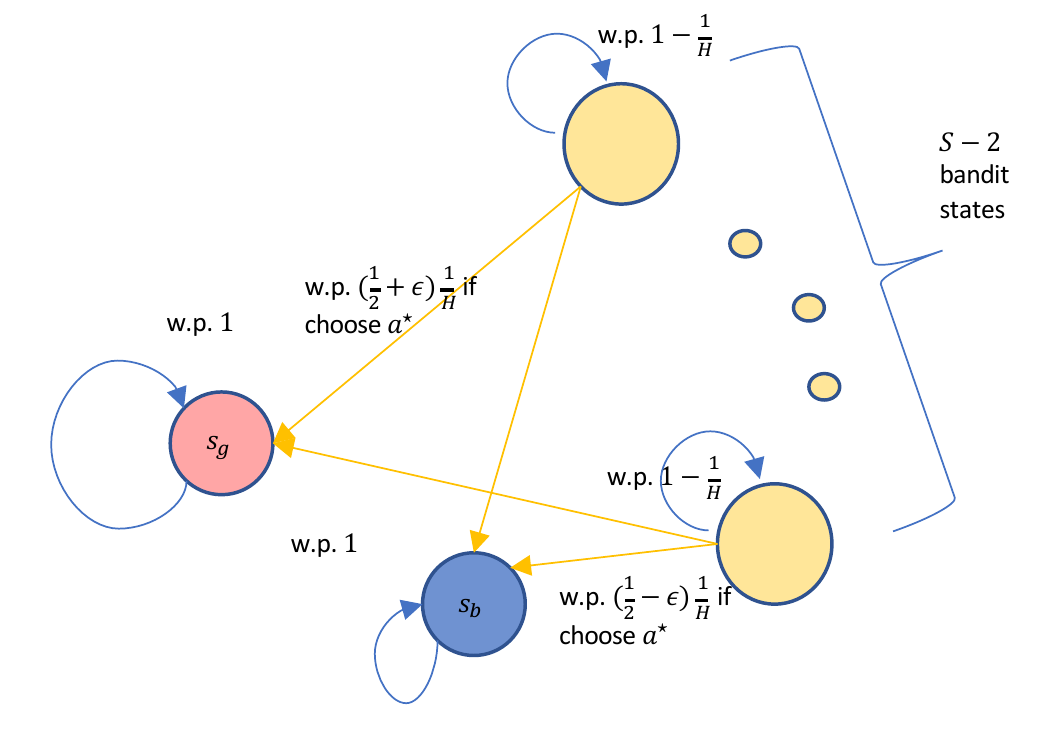}
		\caption{An illustration of State-space transition diagram}
		\label{fig:ill}
	\end{figure}
	
	The transition dynamics are defined as follows:
	\begin{itemize}
		
		\item for $h=1,...,H-1$, 
			\begin{itemize}
				\item For bandit states $b_{h,i}$, there is probability $1-\frac{1}{H}$ to transition to $b_{h+1,i}$ regardless of the action chosen. For the rest of $\frac{1}{H}$ probability, optimal action $a^\star_{h,i}$ will have probability $\frac{1}{2}+\tau$ or $\frac{1}{2}-\tau$ transition to $g_{h+1}$ or $b_{h+1}$ and all other actions $a$ will have equal probability $\frac{1}{2}$ for either $g_{h+1}$ or $b_{h+1}$, where $\tau$ is a parameter will be decided later. Or equivalently,
				
				\[
				\P(\cdot|s_{h,i},a^\star_{h,i})=\begin{cases} 1-\frac{1}{H}\quad &\text{if}\;\cdot=s_{h+1,i}\\
				(\frac{1}{2}+\tau)\cdot\frac{1}{H}\quad &\text{if}\;\cdot=g_{h+1}\\
				(\frac{1}{2}-\tau)\cdot\frac{1}{H} \quad &\text{if}\;\cdot=b_{h+1}\end{cases}\quad 	\P(\cdot|s_{h,i},a)=\begin{cases} 1-\frac{1}{H}\quad &\text{if}\;\cdot=s_{h+1,i}\\
				\frac{1}{2}\cdot\frac{1}{H}\quad &\text{if}\;\cdot=g_{h+1}\\
				\frac{1}{2}\cdot\frac{1}{H} \quad &\text{if}\;\cdot=b_{h+1}\end{cases}
				\] 
				\item $g_h$ always transitions to $g_{h+1}$ and $b_h$ always transitions to $b_{h+1}$, \textit{i.e.} for all $a\in\mathcal{A}$, we have
				\[
				\P(g_{h+1}|g_h,a)=1,\quad \P(b_{h+1}|b_h,a)=1.
				\]
				We will determine parameter $\tau$ at the end of the proof.

			\end{itemize}

		\item for $h=H,...,2H-1$, all states will always transition to the same type of states for the next step, \textit{i.e.} $\forall a\in\mathcal{A}$,
		\begin{equation}\label{eqn:MDP_latter}
		\P(g_{h+1}|g_h,a)=\P(b_{h+1}|b_h,a)=\P(s_{h+1,i}|s_{h,i},a)=1,\;\forall i\in[S-2].
		\end{equation}

	\item The initial distribution is decided by:
	\begin{equation}\label{eqn:initial_c_MDP}
	\P(s_{1,i})=\frac{1}{S},\; \forall i\in[S-2],\;	\P(g_{1})=\frac{1}{S},\;\;\P(b_1)=\frac{1}{S}
	\end{equation}
		
	\item State $s$ will receives reward $1$ if and only if $s=g_h$ and $h\geq H$. The reward at all other states is zero.
	\end{itemize}

 By this construction the optimal policy must take $a^\star_{h,i}$ for each bandit state $s_{h,i}$ for at least the first half of the MDP, \textit{i.e.} need to take $a^\star_{h,i}$ for $h\leq H$. In other words, this construction embeds at least $H(S-2)$ independent best arm identification problems that are identical to the stochastic multi-arm bandit problem in Lemma~\ref{lem:bandit} into the MDP. \textbf{Note the key innovation here is that we can remove the waiting states used in \citet{jiang2017contextual} but still keep the multi-arm bandit problem independent!}\footnote{Here independence means solving one bandit problem provides no information on other bandit problems.} 

Notice in our construction, for any bandit state $s_{h,i}$ with $h\leq H$, the difference of the expected reward between optimal action $a_{h,i}^\star$ and other actions is:
\begin{equation}\label{eqn:diff_reward}
\begin{aligned}
&(\frac{1}{2}+\tau)\cdot\frac{1}{H}\cdot \E[r_{{(h+1)}:2H}|g_{h+1}]+(\frac{1}{2}-\tau)\cdot\frac{1}{H}\cdot \E[r_{{(h+1)}:2H}|b_{h+1}]+(1-\frac{1}{H})\cdot \E[r_{{(h+1)}:2H}|s_{h+1,i}]\\
&-\frac{1}{2H}\cdot \E[r_{{(h+1)}:2H}|g_{h+1}]-\frac{1}{2H}\cdot \E[r_{{(h+1)}:2H}|b_{h+1}]-(1-\frac{1}{H})\cdot \E[r_{{(h+1)}:2H}|s_{h+1,i}]\\
=&(\frac{1}{2}+\tau)\cdot\frac{1}{H}\cdot \E[r_{{(h+1)}:2H}|g_{h+1}]+(\frac{1}{2}-\tau)\cdot\frac{1}{H}\cdot \E[r_{{(h+1)}:2H}|b_{h+1}]\\
&-\frac{1}{2H}\cdot \E[r_{{(h+1)}:2H}|g_{h+1}]-\frac{1}{2H}\cdot \E[r_{{(h+1)}:2H}|b_{h+1}]\\
=&(\frac{1}{2}+\tau)\frac{1}{H}\cdot H+(\frac{1}{2}-\tau)\frac{1}{H}\cdot 0-\frac{1}{2H}\cdot H+\frac{1}{2H}\cdot 0=\tau 
\end{aligned}
\end{equation}
so it seems by Lemma~\ref{lem:bandit} one suffices to use the least possible $\frac{A}{72(\tau)^2}$ samples to identify the best action $a_{h,i}^\star$. However, note the construction of the latter half of the MDP \eqref{eqn:MDP_latter} uses mindless reproduction of previous steps and therefore provides no additional information about the best action once the state at time $H$ is known. In other words, observing $\sum_{t=1}^{2H}r_t=H$ is equivalent as observing $\sum_{t=1}^{H}r_t=1$. Therefore, for the bandit states in the first half the samples that provide information for identifying the best arm is up to time $H$. As a result, the difference of the expected reward between optimal action $a_{h,i}^\star$ and other action for identifying the best arm should be corrected as: 
\begin{align*}
	&(\frac{1}{2}+\tau)\cdot\frac{1}{H}\cdot \E[r_{{(h+1)}:H}|g_{h+1}]+(\frac{1}{2}-\tau)\cdot\frac{1}{H}\cdot \E[r_{{(h+1)}:H}|b_{h+1}]+(1-\frac{1}{H})\cdot \E[r_{{(h+1)}:H}|s_{h+1,i}]\\
	&-\frac{1}{2H}\cdot \E[r_{{(h+1)}:H}|g_{h+1}]-\frac{1}{2H}\cdot \E[r_{{(h+1)}:H}|b_{h+1}]-(1-\frac{1}{H})\cdot \E[r_{{(h+1)}:H}|s_{h+1,i}]\\
	=&(\frac{1}{2}+\tau)\cdot\frac{1}{H}\cdot \E[r_{{(h+1)}:H}|g_{h+1}]+(\frac{1}{2}-\tau)\cdot\frac{1}{H}\cdot \E[r_{{(h+1)}:H}|b_{h+1}]\\
	&-\frac{1}{2H}\cdot \E[r_{{(h+1)}:H}|g_{h+1}]-\frac{1}{2H}\cdot \E[r_{{(h+1)}:H}|b_{h+1}]\\
	=&(\frac{1}{2}+\tau)\frac{1}{H}\cdot 1+(\frac{1}{2}-\tau)\frac{1}{H}\cdot 0-\frac{1}{2H}\cdot 1+\frac{1}{2H}\cdot 0= \frac{\tau}{H}
\end{align*}
Now by Lemma~\ref{lem:bandit}, for each bandit state $s_{h,i}$ satisfying $h\leq H$,
unless $\frac{A}{72(\tau/H)^2}$ samples are collected from that state, the learning algorithm fails to identify the optimal action $a^\star_{h,i}$ with probability at least $1/3$.

After running any algorithm, let $C$ be the set of $(h,s)$ pairs for which the algorithm identifies
the correct action. Let $D$ be the set of $(h, s)$ pairs for which the algorithm collects fewer than $\frac{A}{72(\tau/H)^2}$ samples. Then by Lemma~\ref{lem:bandit} we have
\begin{align*}
\E[|C|]&=\E\left[\sum_{(h,s)}\mathds{1}[a_{h,s}=a_{h,s}^\star]\right]\leq ((S-2)H-|D|)+\E\left[\sum_{(h,s)\in D}\mathds{1}[a_{h,s}=a_{h,s}^\star]\right]\\
&\leq ((S-2)H-|D|)+\frac{2}{3}|D|=(S-2)H-\frac{1}{3}|D|.
\end{align*}

If we have $n\leq \frac{H(S-2)}{2}\times\frac{A}{72(\tau/H)^2}$, by pigeonhole principle the algorithm can collect $\frac{A}{72(\tau/H)^2}$ samples for at most half of the bandit problems, \textit{i.e.} $|D|\geq H(S-2)/2$. Therefore we have
\[
\E[|C|]\leq (S-2)H-\frac{1}{3}|D|\leq \frac{5}{6}(S-2)H.
\]
Then by Markov inequality 
\[
\P\left[|C|\geq \frac{11}{12}H(S-2)\right]\leq \frac{5/6}{11/12}=\frac{10}{11}
\]
so the algorithm failed to identify the optimal action on 1/12 fraction of the bandit problems with probability at least $1/11$. Note for each failure in identification, the reward is differ by $\tau $ (see \eqref{eqn:diff_reward}), therefore under the event $\{|C'|\geq \frac{1}{12}H(S-2)\}$, following the similar calculation of \cite{jiang2017contextual} the suboptimality of the policy produced by the algorithm is

\begin{align*}
\epsilon:=v^\star-v^{\widehat{\pi}}&=\P[\text{visit} \;C']\times \tau +\P[\text{visit} \;C]\times 0=\P[\bigcup_{(h,i)\in C'}\text{visit}(h,i)]\times\tau \\
&=\sum_{(h,i)\in C'} \P[\text{visit}(h,i)]\times\tau =\sum_{(h,i)\in C'} \frac{1}{HS}(1-1/H)^{h-1}\tau \\
&\geq \sum_{(h,i)\in C'} \frac{1}{HS}(1-1/H)^{H}\tau \geq \sum_{(h,i)\in C'} \frac{1}{HS}\frac{1}{4}\tau \\
&\geq \frac{H(S-2)}{12}\frac{1}{HS}\frac{1}{4}\tau =c_1\frac{\tau }{48}.
\end{align*}
where the third equal sign uses all best arm identification problems are independent. Now we set $\tau=\min(\sqrt{1/8},48\epsilon/c_1)$ and under condition $n\leq cH^3SA/\epsilon^2$, we have
\[
n\leq cH^3SA/\epsilon^2\leq c48^2H^3SA/\tau^2=c{48^2\cdot 72} HS\cdot \frac{A}{72(\tau/H)^2}:=c'HS\cdot \frac{A}{72\tau^2}\leq \frac{H(S-2)}{2}\cdot \frac{A}{72\tau^2},
\]
the last inequality holds as long as $S\geq 2/(1-2c')$. Therefore in this situation, with probability at least $1/11$, $v^\star-v^{\widehat{\pi}}\geq \epsilon$. Finally, we can use scaling to reduce the horizon from $2H$ to $H$.
	
\end{proof}

\subsection{Information theoretical lower sample complexity bound over problems in $\mathcal{M}_{d_m}$ for identifying $\epsilon$-optimal policy.}  

	For all $0<d_m\leq\frac{1}{SA}$, let the class of problems be

	$$\mathcal{M}_{d_m} :=\big\{(\mu,M) \; \big| \;\min_{t,s_t,a_t} d_t^\mu(s_t,a_t) \newline \geq d_m\big\},$$
now we consider deriving minimax lower bound over this class.

\begin{theorem}\label{thm:learn_low_bound}
	Under the same condition of Theorem~\ref{thm:learn_low_bound_all}. In addition assume $0<d_m\leq\frac{1}{SA}$. There exists another universal constant $c$ such that when $n\leq cH^3/d_m\epsilon^2$, we always have 
	\[
\inf_{{v}^{\pi_{alg}}}\sup_{(\mu,M)\in\mathcal{M}_{d_m}}\P_{\mu,M}\left(v^*-v^{\pi_{alg}}\geq \epsilon\right)\geq p.
\]
\end{theorem}

\begin{proof} 
	
The hard instance $(\mu,M)$ we used is based on Theorem~\ref{thm:learn_low_bound_all}, which is described as follows. 	

	\begin{itemize}
	
	\item for the MDP $M=(\mathcal{S},\mathcal{A}, r,P,d_1,2H+2)$, 
	\begin{itemize}
		
		\item Initial distribution $d_1$ will always enter state $s_0$, and there are two actions with action $a_1$ always transitions to $s_\text{yes}$ and action $a_2$ always transitions to $s_\text{no}$. The reward at the first time $r_1(s,a)=0$ for any $s,a$. 
		
		\item For state $s_\text{no}$, it will always transition back to itself regardless of the action and receive reward $0$, \emph{i.e.} 
		\[
		P_t(s_\text{no}|s_\text{no},a)=1,\;r_t(s_\text{no},a)=0,\;\forall t,\;\forall a.
		\] 
		\item For state $s_\text{yes}$, it will transition to the MDP construction in Theorem~\ref{thm:learn_low_bound_all} with horizon $2H$ and $s_\text{yes}$ always receives reward zero. 
		\item For $t=1$, choose $\mu(a_1|s_0)=\frac{1}{2}d_mSA$ and $\mu(a_2|s_0)=1-\frac{1}{2}d_m SA$. For $t\geq 2$, choose $\mu$ to be uniform policy, \emph{i.e.} $\mu(a_t|s_t)=1/A$.

	\end{itemize}
	
\end{itemize}

Based on this construction, the optimal policy has the form $\pi^\star=(a_1,\ldots)$ and therefore the MDP branch that enters $s_\text{no}$ is uninformative. Hence, data collected by that part is uninformed about the optimal policy and there is only $\frac{1}{2}d_m SA$ proportion of data from $s_\text{yes}$ are useful. Moreover, by Theorem~\ref{thm:learn_low_bound_all} the rest of Markov chain succeeded from $s_\text{yes}$ requires $\Omega(H^3SA/\epsilon^2)$ episodes (regardless of the exploration strategy/logging policy), so the actual data complexity needed for the whole construction $(\mu,M)$ is $\frac{\Omega(H^3SA/\epsilon^2)}{d_mSA}=\Omega(H^3/d_m\epsilon^2)$. 

It remains to check this construction $\mu,M$ stays within $\mathcal{M}_{d_m}$.
\begin{itemize}
	\item For $t=1$, we have $d_1(s_0,a_1)=\frac{1}{2}d_m SA\geq d_m$ (since $S\geq 2$) and $d_1(s_0,a_2)=1-\frac{1}{2}d_mSA\geq d_m$ (this is since $d_m\leq \frac{1}{SA}\leq \frac{2}{2+SA}$);
	\item For $t=2$, $d_2(s_\text{yes},a)=\frac{1}{2}d_m SA\cdot\frac{1}{A}=\frac{1}{2}d_m S\geq d_m$ (since $S\geq 2$) and similar for $s_\text{no}$;
	\item For $t\geq 3$, for $g_h$ and $b_h$ in the sub-chain inherited from $s_\text{yes}$, note $d_h(g_h)\leq d_{h+1}(g_{h+1})$ (since $g_h$ and $b_h$ are absorbing states regardless of actions), therefore $d_h(g_h)\geq d_1(g_1)=d_1(s_\text{yes})\cdot \P(g_1|s_\text{yes})=\frac{1}{2}d_m SA\cdot \frac{1}{S}=\frac{1}{2}d_m A$, since $\mu$ is uniform so $d_h(g_h,a)\geq\Omega(d_m A)\cdot \frac{1}{A}=\Omega(d_m )$ forall $a$. Similar result can be derived for $b_h$ in identical way.
	
	For bandit state, we have for all $i\in[S-2]$,
	\begin{align*}
	d^\mu_{t+1}(s_{t+1,i})&\geq \P^\mu(s_{t+1,i},s_{t,i},s_{t-1,i},\ldots,s_{2,i},s_{1,i},s_\text{yes},s_0)\\
	&= \prod_{u=1}^{t}\P^\mu(s_{u+1,i}|s_u)\P^\mu(s_{1,i}|s_\text{yes})\P^\mu(s_\text{yes}|s_0)\\
	&=(1-\frac{1}{H})^t \left(\frac{1}{S}\right)\left(\frac{1}{2}d_m SA\right)\geq c d_m A,\\
	\end{align*}
	now by $\mu$ is uniform we have $d^\mu_{t+1}(s_{t+1,i},a)\geq \Omega(d_mA)\cdot\frac{1}{A}=\Omega(d_m)$ for all $a$. This concludes the proof.
	
\end{itemize}

\end{proof}

\begin{remark}
A directly corollary is that the sample complexity in Theorem~\ref{thm:offlinelearning} part 3. is optimal. Indeed, for the case $\epsilon_{\text{opt}}=0$, Theorem~\ref{thm:offlinelearning} implies $\widehat{\pi}$ is the $\epsilon$-optimal policy learned with sample complexity $O(H^3\log(HSA/\delta)/d_m\epsilon^2)$. Theorem~\ref{thm:learn_low_bound} implies this sample complexity cannot be further reduced up to the logarithmic factor.
\end{remark}

\subsection{Information theoretical lower sample complexity bound for uniform convergence in OPE.}  

By applying Theorem~\ref{thm:learn_low_bound}, we can now prove Theorem~\ref{thm:uni_lower}.

\begin{proof}[Proof of Theorem~\ref{thm:uni_lower}]
We prove it by contradiction. Suppose there is one off-policy evaluation method $\widehat{v}^\pi$ such that 
\[
\sup_{\pi\in\Pi}|\widehat{v}^{{\pi}}-v^{{\pi}}|\leq o\left(\sqrt{\frac{H^3}{d_m n}}\right),
\]
where $o(\cdot)$ represents the standard small $o$-notation. Then by \begin{align*}
0&\leq v^{\pi^\star}-v^{\widehat{\pi}^\star}= v^{\pi^\star}-\widehat{v}^{\widehat{\pi}^\star}+\widehat{v}^{\widehat{\pi}^\star}-v^{\widehat{\pi}^\star}\\
&\leq |v^{\pi^\star}-\widehat{v}^{{\pi}^\star}|+|\widehat{v}^{\widehat{\pi}^\star}-v^{\widehat{\pi}^\star}|\leq 2\sup_\pi|v^\pi-\widehat{v}^\pi|.
\end{align*} this OPE method implies a $\epsilon$-optimal policy learning algorithm with sample complexity
$
o(H^3/d_m\epsilon^2)
$ which is smaller than the information theoretical lower bound obtained in Theorem~\ref{thm:learn_low_bound}. Contradiction!
\end{proof}

\section{Proofs of Theorem~\ref{thm:offlinelearning}}

\begin{proof}[Proof of Theorem~\ref{thm:offlinelearning}]
	Part 1. and Part 2. are just direct corollaries. We only prove Part 3. here. Indeed, by definition of empirical optimal policy we have $\widehat{Q}^{\pi^\star}\leq\widehat{Q}^{\widehat{\pi}^\star} $, so we have the following:
	\begin{align*}
	Q^{\pi^\star}_1-Q_1^{\widehat{\pi}}&=Q^{\pi^\star}_1-\widehat{Q}_1^{\widehat{\pi}^\star}+\widehat{Q}_1^{\widehat{\pi}^\star}-\widehat{Q}_1^{\widehat{\pi}}+\widehat{Q}_1^{\widehat{\pi}}-Q_1^{\widehat{\pi}}\\
	&\leq Q^{\pi^\star}_1-\widehat{Q}_1^{{\pi}^\star}+\widehat{Q}_1^{\widehat{\pi}^\star}-\widehat{Q}_1^{\widehat{\pi}}+\widehat{Q}_1^{\widehat{\pi}}-Q_1^{\widehat{\pi}}\\
	&\leq Q^{\pi^\star}_1-\widehat{Q}_1^{{\pi}^\star}+\epsilon_{\text{opt}}\cdot \mathbf{1}+\widehat{Q}_1^{\widehat{\pi}}-Q_1^{\widehat{\pi}}\\
	\end{align*}
	and $\widehat{Q}_1^{\widehat{\pi}}-Q_1^{\widehat{\pi}}$ can be bounded by Theorem~\ref{thm:local_uni_opt} using local uniform convergence. $Q^{\pi^\star}_1-\widehat{Q}_1^{{\pi}^\star}$ can be bounded by $O(\sqrt{\frac{H^3\log(HSA/\delta)}{nd_m}})$ using the similar technique in Section~\ref{sec:E} even without introducing $\epsilon_{\text{opt}}$ since $\pi^\star$ is a fixed policy. All these implies:
	\[
	Q^{\pi^\star}_1-Q_1^{\widehat{\pi}}\leq \left(O(\sqrt{\frac{H^3\log(HSA/\delta)}{nd_m}}) + \epsilon_{\text{opt}}\right)\cdot \mathbf{1}.
	\]
	Especially when $\epsilon_\text{opt}=0$ then this is slightly stronger than the stated result since:
	{\small
	\[
	v^{\pi^\star}_1-v_1^{\widehat{\pi}^\star} =Q^{\pi^\star}_1(\cdot,\pi^\star(\cdot))-Q_1^{\widehat{\pi}^\star}(\cdot,\widehat{\pi}^\star(\cdot))\leq Q^{\pi^\star}_1(\cdot,\pi^\star(\cdot))-Q_1^{\widehat{\pi}^\star}(\cdot,{\pi}^\star(\cdot))\leq \norm{Q^{\pi^\star}_1-Q_1^{\widehat{\pi}^\star}}_\infty\leq O(\sqrt{\frac{H^3\log(HSA/\delta)}{nd_m}}) \cdot \mathbf{1}
	\]
}

\end{proof}

\section{Simulation details}\label{sec:simluation_detail}

The non-stationary MDP with used for the experiments have $2$ states $s_0,s_1$ and $2$ actions $a_1,a_2$ where action $a_1$ has probability $1$ always going back the current state and for action $a_2$, there is one state s.t. after choosing $a_2$ the dynamic transitions to both states with equal probability $\frac{1}{2}$ and the other one has asymmetric probability assignment ($\frac{1}{4}$ and $\frac{3}{4}$). The transition after choosing $a_2$ is changing over different time steps therefore the MDP is non-stationary and the change is decided by a sequence of pseudo-random numbers. More formally, $P_t$ can be either

\begin{align*}
\P(s_0|s_0,a_1)=1;
\P(s_1|s_1,a_1)=1;
\P(\cdot|s_0,a_2)=\begin{cases}
\frac{1}{2},\;\text{if} \;\cdot=s_1 \\
\frac{1}{2},\;\text{if}\;\cdot = s_0 \\
\end{cases};\;
\P(\cdot|s_1,a_2)=\begin{cases}
\frac{3}{4},\;\text{if} \;\cdot=s_1 \\
\frac{1}{4},\;\text{if}\;\cdot = s_0 \\
\end{cases}
\end{align*}
or 
\begin{align*}
\P(s_0|s_0,a_1)=1;
\P(s_1|s_1,a_1)=1;
\P(\cdot|s_0,a_2)=\begin{cases}
\frac{1}{4},\;\text{if} \;\cdot=s_1 \\
\frac{3}{4},\;\text{if}\;\cdot = s_0 \\
\end{cases};\;
\P(\cdot|s_1,a_2)=\begin{cases}
\frac{1}{2},\;\text{if} \;\cdot=s_1 \\
\frac{1}{2},\;\text{if}\;\cdot = s_0 \\
\end{cases}
\end{align*}

 Moreover, to make the learning problem non-trivial we use non-stationary rewards with $4$ categories, \emph{i.e.} $r_t(s,a)\in\{\frac{1}{4},\frac{2}{4},\frac{3}{4},1\}$ and assignment of $r_t(s,a)$ for each value is changing over time. That means, one possible assignment can be
 \[
 r_t(s_0,a_1)= 1/4, \; r_t(s_0,a_2)=2/4,\;r_t(s_1,a_1)=3/4,\;r_t(s_1,a_2)=1/4.
 \]

Moreover, the logging policy in Figure~\ref{fig:different_H} is uniform with $\mu_t(a_1|s)=\mu_t(a_2|s)=\frac{1}{2}$ for both states. We implement the non-stationary MDP  in the {\fontfamily{cmtt}\selectfont Python} environment and pseudo-random numbers $p_t,r_t$'s are generated by keeping {\fontfamily{cmtt}\selectfont numpy.random.seed(100)}. 

We fix episodes $n=2048$ and run each algorithm under $K = 100$ macro-replications with data $\mathcal{D}_{(k)}=\left\lbrace (s_t^{(i)},a_t^{(i)},r_t^{(i)})\right\rbrace^{i\in[n],t\in[H]}_{(k)} $, and use each $\mathcal{D}_{(k)}$ $(k=1,...,K)$ to construct a estimator $\widehat{v}^\pi_{[k]}$, then the (empirical) RMSE for fixed policy is computed as: 
\[
\mathrm{RMSE\_FIX} = \sqrt{\frac{\sum_{k=1}^K (\widehat{v}^\pi_{[k]}-v^\pi_{\mathrm{true}})^2}{K}},
\]
and RMSE for suboptimality gap is computed as 
\[
\mathrm{RMSE\_SUB} = \sqrt{\frac{\sum_{k=1}^K ({v}^{\widehat{\pi}^\star_{[k]}}-v^{\pi^\star}_{\mathrm{true}})^2}{K}},
\]
and RMSE for empirical optimal policy gap is computed as 
\[
\mathrm{RMSE\_EMPIRICAL} = \sqrt{\frac{\sum_{k=1}^K (\widehat{v}^{\widehat{\pi}^\star}_{[k]}-v^{\widehat{\pi}^\star}_{\mathrm{true}})^2}{K}},
\]
where $v^\pi_{\mathrm{true}}$ is obtained by calculating  $P^\pi_{t+1,t}(s^\prime|s)=\sum_a P_{t+1,t}(s^\prime|s,a)\pi_t(a|s)$, the marginal state distribution ${d}_t^\pi  = {P}^{\pi}_{t,t-1} {d}_{t-1}^\pi$, ${r}^\pi_t(s_t)=\sum_{a_t}{r}_t(s_t,a_t)\pi_t(a_t|s_t)$ and ${v}^\pi_{\mathrm{true}}=\sum_{t=1}^H\sum_{s_t}{d}^\pi_t(s_t){r}^\pi_t(s_t)$. $v^{\pi^\star}_{\mathrm{true}}$ is obtained by running Value Iteration exhaustively until the error converges to $0$. The average relative error for suboptimality (average of $|{v}^{\widehat{\pi}^\star_{[k]}}-v^{\pi^\star}_{\mathrm{true}}|/v^{\pi^\star}_{\mathrm{true}}$) at $H=1000$ is $0.0011$. Lastly, we also show the scaling of $|\widehat{v}^{\widehat{\pi}^\star}-{v}^{\widehat{\pi}^\star}|$ in Figure~\ref{fig:appendix}, which shares a similar pattern as the suboptimality plot as a whole. \footnote{Here we do point out the empirical dependence on $H$ for $|\widehat{v}^{\widehat{\pi}^\star}-{v}^{\widehat{\pi}^\star}|$ in the Figure~\ref{fig:appendix} is actually less than $H^{1.5}$, this comes from that the MDP example we choose is not the ``hardest'' example for quantity $|\widehat{v}^{\widehat{\pi}^\star}-{v}^{\widehat{\pi}^\star}|$, as opposed to quantity $|v^\star-v^{\widehat{\pi}^
\star}|$ in Figure~\ref{fig:main}.}       

\begin{figure}[H]
	\centering     
	\includegraphics[width=0.6\linewidth]{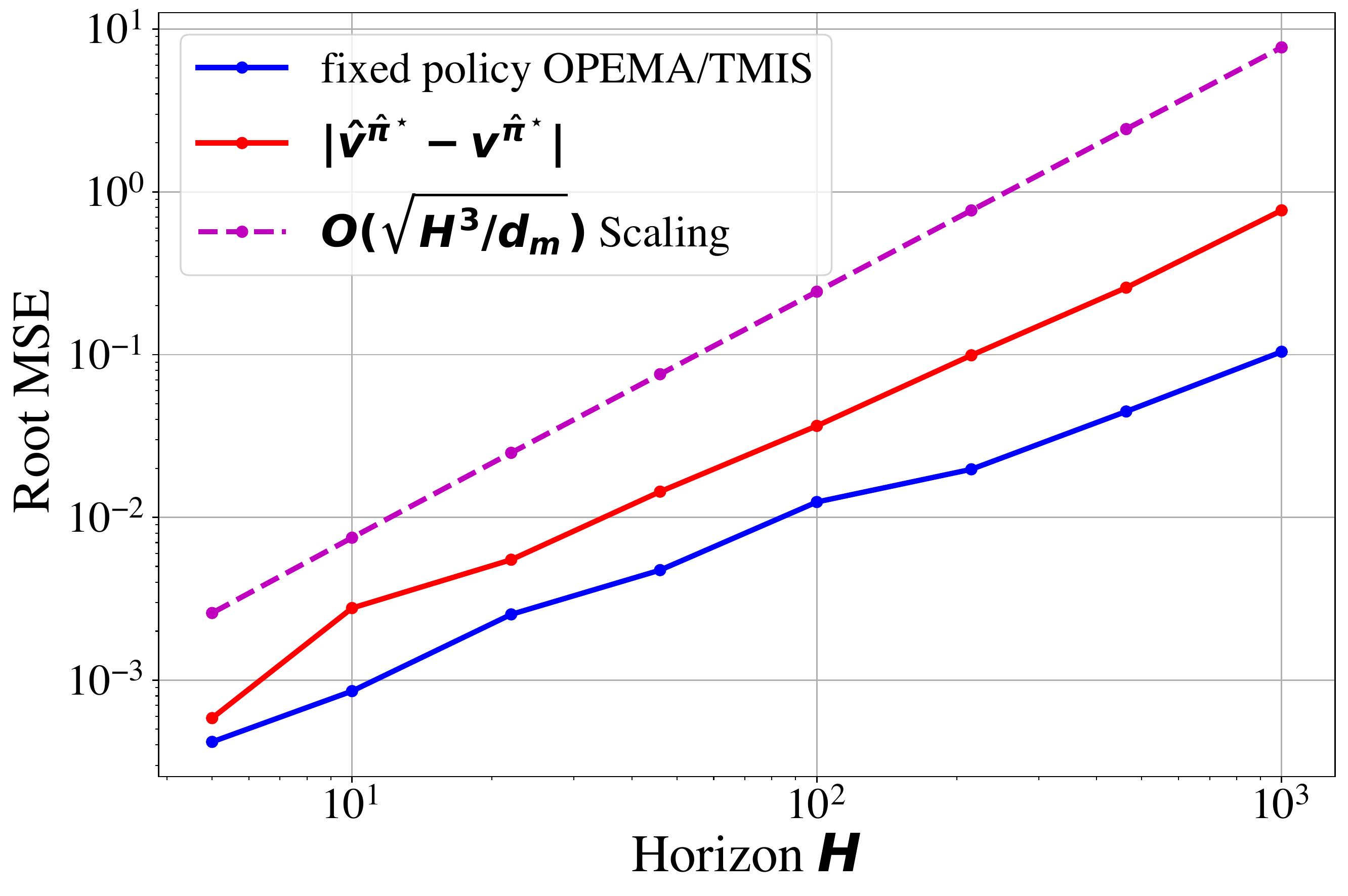}
	\caption{Log-log plot  showing the dependence on horizon of uniform OPE and pointwise OPE via learning ($|\hat{v}^{\widehat{\pi}^\star}-v^{\widehat{\pi}^\star}|$) over a non-stationary MDP example.}
	\label{fig:appendix}
\end{figure}

\section{On improvement over vanilla simulation lemma for fixed policy evaluation}\label{sec:app_sl}

\paragraph{Vanilla simulation lemma, Lemma~1 of \cite{jiang2018notes}.}Without loss of generality, assuming reward is determinsitic function over state-action. By definition of Bellman equation, we have the following:
\[
\widehat{V}^\pi_t=r+\widehat{P}_{t+1}^\pi\widehat{V}^\pi_{t+1},\quad {V}^\pi_t=r+{P}_{t+1}^\pi{V}^\pi_{t+1},
\]
define $\epsilon_P=\sup_{t,s_t,a_t}||\widehat{P}_{t}(\cdot|s_t,a_t)-P_t(\cdot|s_t,a_t)||_1$, then by Hoeffding's inequality and union bound, with probability $1-\delta$,
\[
\epsilon_P\leq S\cdot \sup_{t,s_t,a_t}||\widehat{P}_{t}(\cdot|s_t,a_t)-P_t(\cdot|s_t,a_t)||_\infty\leq S \cdot\sup_{t,s_t,a_t} O\left(\sqrt{\frac{\log(HSA/\delta)}{n_{s_t,a_t}}}\mathbf{1}(E_t)\right)=O\left(\sqrt{\frac{S^2\log(HSA/\delta)}{n\cdot d_m}}\right)
\]
then 
\begin{align*}
\widehat{V}^\pi_t-{V}^\pi_t=&\widehat{P}_{t+1}^\pi\widehat{V}^\pi_{t+1}-{P}_{t+1}^\pi{V}^\pi_{t+1}\\
\leq&\left(\norm{\widehat{P}_{t+1}^\pi-{P}_{t+1}^\pi}_1\norm{\widehat{V}^\pi_{t+1}}_\infty+\norm{{P}_{t+1}^\pi}_1\norm{\widehat{V}^\pi_{t+1}-{V}^\pi_{t+1}}_\infty\right)\cdot \mathbf{1}\\
\leq& \left(H\epsilon_P +\norm{\widehat{V}^\pi_{t+1}-{V}^\pi_{t+1}}_\infty\right)\cdot\mathbf{1},\\
\end{align*}
solving recursively, we have
\[
\norm{\widehat{V}^\pi_1-{V}^\pi_1}_\infty\leq H^2\epsilon_P\leq O\left(\sqrt{\frac{H^4 S^2\log(HSA/\delta)}{n\cdot d_m}}\right).
\]

This verifies SL has complexity $\widetilde{O}(H^4S^2/d_m\epsilon^2)$. We do point out above standard analysis can be improved (\emph{e.g.} \cite{jiang2018notes} Section~2.2) to $ \tilde{O}({{H^4 S}/{ d_m\epsilon^2}})$, then in this case our analysis (Lemma~\ref{thm:single_finer_bound}) has an improvement of $H^2S$ with respect to the modified result.

\section{Algorithms}\label{sec:alg}

\begin{algorithm*}[thb]
	\caption{OPEMA}
	\label{alg:mainalgo}
	{\bfseries Input:} Logging data $\mathcal D = \{\{s_t^{(i)},a_t^{(i)},r_t^{(i)}\}_{t = 1}^{H }\}_{i = 1}^{n}$ from the behavior policy $\mu$. A target policy $\pi$ which we want to evaluate its cumulative reward.
	\begin{algorithmic}[1]
		\STATE Calculate the on-policy estimation of initial distribution $d_1(\cdot)$ by
		$
		\widehat d_1(s) := \frac{1}{n}\sum_{i = 1}^{n} \mathbf{1}(s_1^{(i)} = s),
		$
		and set $\widehat d_1^{\mu}(\cdot):=\widehat d_1(\cdot)$, $\widehat d_1^\pi(s):=\widehat d_1(\cdot)$.
		\FOR{$t = 2,3,\dotsc,H$}
		\STATE Choose all transition data at time step $t$, $\{s_t^{(i)},a_t^{(i)},r_t^{(i)}\}_{i = 1}^{n}$.
		\STATE Calculate the on-policy estimation of $d_{t}^{\mu}(\cdot)$ by
		$
		\widehat d_{t}^{\mu}(s) := \frac{1}{n}\sum_{i = 1}^{n} \mathbf{1}(s_{t }^{(i)} = s).
		$
		\STATE Set the off-policy estimation of $\widehat{P}_{t}(s_{t}|s_{t-1},a_{t-1})$:
		\begin{equation*}
		\hspace{-6.5mm}
		\begin{aligned}
		\widehat{P}_{t}(s_{t}|s_{t-1},a_{t-1})
		:=\frac{\sum_{i=1}^n\mathbf{1}[(s^{(i)}_{t},a^{(i)}_{t-1},s^{(i)}_{t-1})=(s_{t},s_{t-1},a_{t-1})]}{n_{s_{t-1},a_{t-1}}}
		\end{aligned}
		\end{equation*}
		when $n_{s_{t-1},a_{t-1}}>0$. Otherwise set it to be zero.
		\STATE Estimate the reward function
		\begin{align*}
		\widehat{r}_t(s_t,a_t) := \frac{\sum_{i=1}^n r_t^{(i)}\mathbf{1}(s_t^{(i)}=s_t, a_t^{(i)} =a_t)}{\sum_{i=1}^n \mathbf{1}(s_t^{(i)}=s_t, a_t^{(i)} =a_t) }.
		\end{align*}
		when $n_{s_t,a_t}>0$. Otherwise set it to be zero.
		\STATE Set $\widehat{d}^\pi_t(\cdot,\cdot)$ according to $\widehat{d}^\pi_t=\widehat{P}^\pi_t\widehat{d}^\pi_{t-1}$, where $\widehat{d}^\pi_t(\cdot,\cdot)$ is the estimated state-action distribution. 
		\ENDFOR
		\STATE Substitute the all estimated values above into $\widehat v^\pi=\sum_{t=1}^H\langle \widehat{d}^\pi_t,\widehat{r}_t\rangle$ to obtain $\widehat v^\pi$, the estimated value of $\pi$.
	\end{algorithmic}
\end{algorithm*}

\begin{algorithm*}[thb]
	\caption{Data Splitting TMIS in \cite{yin2020asymptotically}}
	\label{alg:secondalgo}
	{\bfseries Input:} Logging data $\mathcal D = \{\{s_t^{(i)},a_t^{(i)},r_t^{(i)}\}_{t = 1}^{H }\}_{i = 1}^{n}$ from the behavior policy $\mu$. A target policy $\pi$ which we want to evaluate its cumulative reward. Requiring splitting data size $M$. 
	\begin{algorithmic}[1]
		\STATE Randomly splitting the data $\mathcal{D}$ evenly into $N$ folds, with each fold $|\mathcal{D}^{(i)}|=M$, \emph{i.e.} $n=M\cdot N$.
		\FOR{$i = 1,2,\dotsc,N$}
		\STATE Use Algorithm~\ref{alg:mainalgo} to estimate $\widehat{v}_{(i)}^\pi$ with data $\mathcal{D}^{(i)}$.
		\ENDFOR
		\STATE Use the mean of $\widehat{v}_{(1)}^\pi,\widehat{v}_{(2)}^\pi,...,\widehat{v}_{(N)}^\pi$ as the final estimation of $v^\pi$.
	\end{algorithmic}
\end{algorithm*}

\begin{remark}
	In short, we can see Algorithm~\ref{alg:secondalgo} requires the splitting data size $M$ which is undecided by \cite{yin2020asymptotically} and that makes the hyper-parameter requiring additional concrete specifications to make the data splitting estimator sample efficient. In contrast, OPEMA in Algorithm~\ref{alg:mainalgo} is defined without ambiguity and can be implemented without extra work.
	
	Their results require number of episodes in each splitted data $M$ to satisfy $\widetilde{O}(\sqrt{nSA})>M>O(HSA)$. To achieve data efficiency, they need $n\approx \Theta(H^2SA/\epsilon^2)$ and by that condition $M$ has to satisfy $M\approx C\cdot HSA$. In this case, data-splitting version needs to create $N=n/M$ empirical transition dynamics and each dynamics use $H^3/N\approx C\cdot H^2SA/\epsilon^2$ episodes which is less than the lower bound ($O(H^3)$) required for learning. Most critically, due to data-splitting it has $N$ empirical transitions hence it is not clear which transition to plan over. Therefore in this sense their result does not enables efficient offline learning. Our Analysis for unsplitted version (OPEMA) addresses all these issues.

\end{remark}

\end{document}